\documentclass[10pt]{article} % For LaTeX2e

\usepackage[preprint]{tmlr}
\usepackage{booktabs}
\usepackage{algorithm}
\usepackage{enumitem}
\usepackage{multirow}
\let\tmlrAND\AND
\let\AND\undefined
\usepackage{algorithmic}

\let\AND\tmlrAND

\usepackage{amsmath,amsfonts,bm}

\def\eqref#1{equation~\ref{#1}}
\def\1{\bm{1}}

\DeclareMathAlphabet{\mathsfit}{\encodingdefault}{\sfdefault}{m}{sl}
\SetMathAlphabet{\mathsfit}{bold}{\encodingdefault}{\sfdefault}{bx}{n}

\newcommand{\E}{\mathbb{E}}

\usepackage{amsmath,amssymb,amsthm}
\usepackage{graphicx}

\usepackage{hyperref}
\hypersetup{colorlinks=true,linkcolor=blue,citecolor=blue,urlcolor=blue}

\newcommand{\Pa}{\mathrm{Pa}}
\theoremstyle{plain}
\newtheorem{theorem}{Theorem}[section]
\newtheorem{proposition}[theorem]{Proposition}
\newtheorem{lemma}[theorem]{Lemma}
\newtheorem{corollary}[theorem]{Corollary}
\theoremstyle{definition}
\newtheorem{definition}[theorem]{Definition}
\newtheorem{assumption}[theorem]{Assumption}
\theoremstyle{remark}
\newtheorem{remark}[theorem]{Remark}

\newtheorem{fact}[theorem]{Fact}

\title{Intervention-Based Time Series Causal Discovery\\via Simulator-Generated Interventional Distributions}
        
\author{\name Tsuyoshi Okita \\
  \addr Kyushu Institute of Technology}

\def\month{MM}  % Insert correct month for camera-ready version
\def\year{YYYY} % Insert correct year for camera-ready version
\def\openreview{\url{https://openreview.net/forum?id=XXXX}} % Insert correct link to OpenReview for camera-ready version

\begin{document}

\maketitle
\begin{abstract}
In many scientific domains, physics-based simulators---programs that
compute system behaviour from governing equations, such as density
functional theory for materials or fluid dynamics solvers---encode
causal mechanisms and can predict system behaviour under hypothetical
interventions. Machine learning extracts patterns
from observational time series at scale, but those patterns reflect
statistical associations ($P(Y \mid X)$), not causal
effects ($P(Y \mid \mathrm{do}(X))$): in the presence of latent
confounders, the structural VAR is provably non-identifiable from
observational data alone
(Fact~\ref{thm:svar_nonidentifiable}), and no amount of statistical
sophistication can substitute for genuine interventional data. Bridging
these two traditions has so far been limited to using simulators for
prediction; no existing framework uses them for \emph{causal structure
discovery} in time series.

We propose SVAR-FM (Structural VAR with Flow Matching), a framework
that treats a physical simulator as a mechanical realization of
Pearl's $\mathrm{do}(\cdot)$ operator. Clamping a variable in the
simulator physically severs confounding paths, producing
interventional data by construction rather than by statistical
argument. Conditional Flow Matching then parameterizes the
interventional conditionals, enabling nonlinear mechanism learning.

This yields four results.
(1)~The full structural VAR---contemporaneous and lagged edges
jointly---becomes identifiable under a \emph{coverage condition} on
the simulator-clampable variables, verifiable \emph{a priori} from
domain knowledge alone
(Theorem~\ref{thm:svarfm_identifiability}). The argument is intrinsic
to the time series setting and has no i.i.d.\ counterpart.
(2)~An end-to-end error bound
$|\hat{e}_{i\to j} - e^{*}_{i\to j}| \le
O(M^{-1/2}) + O(\delta_{\mathcal{S}}) + O(\varepsilon_{\mathrm{FM}})$
(Theorem~\ref{thm:simulator_error_extended}) cleanly separates
Monte Carlo sampling, simulator fidelity~$\delta_{\mathcal{S}}$, and
Flow Matching approximation. A sharp consequence is a \emph{sign-flip
regime} (Corollary~\ref{cor:robustness}): when
$\delta_{\mathcal{S}}$ exceeds a threshold set by the signal
magnitude, the estimated causal effect reverses sign---a prediction
that the prevailing forward-prediction view of simulators cannot
produce.
(3)~The CausalSim benchmark confirms that SVAR-FM recovers the
correct causal sign across four scientific domains (macroeconomics,
diabetes, cosmic ray physics, and battery degradation) where observational methods produce
sign-reversed estimates due to confounding.
(4)~A case study in ultrafast laser physics tests the
sign-flip prediction by physically varying $\delta_{\mathcal{S}}$
through the accuracy level of a first-principles quantum solver:
the low-accuracy setting produces a sign-reversed estimate, while
the high-accuracy setting recovers the correct positive slope
($R^2 = 0.983$, zero bias relative to ground truth), providing the
first experimental demonstration of a simulator-fidelity-dominated
failure mode in causal discovery.
\end{abstract}

\section{Introduction}
\label{sec:introduction}

Two powerful but historically separate traditions exist for
understanding complex dynamical systems. On the one hand, mechanistic
simulators encode causal mechanisms at the microscopic level and can
predict system behaviour under hypothetical interventions. These range
from \emph{first-principles solvers}---programs that compute material
properties directly from quantum mechanics, a cornerstone of recent
AI-for-Science efforts in new materials
discovery~\citep{merchant2023gnome,zhang2024mattersim} and
drug design~\citep{jumper2021alphafold}---to Navier--Stokes solvers
for fluid dynamics, kinetic models for chemical reactions, and
agent-based models for financial markets. On the other hand, machine
learning methods extract statistical patterns from observational time
series at scale, but the patterns they find are associational: they
reflect correlations shaped by confounding, feedback, and latent
common causes, not the underlying causal structure.

Bridging these two traditions has long been recognised as a key
challenge. Physics-informed machine
learning~\citep{raissi2019pinn,karniadakis2021physics} incorporates
physical laws as regularization of predictive models, and
AI-for-Science surrogate
models~\citep{jumper2021alphafold,lam2023graphcast,bodnar2024aurora}
learn to replace simulators for faster forward prediction. In both
cases, however, the simulator's role is to improve \emph{prediction},
not to identify \emph{causal structure}. The causal graph---which
variable drives which, and through what mechanism---remains
inaccessible from observational data alone whenever latent confounders
are present.

Causal inference~\citep{pearl2009causality} provides the missing
language: the $\mathrm{do}(\cdot)$ operator formally distinguishes
intervention from observation, and under appropriate conditions makes
causal structure identifiable. Yet existing causal discovery methods
for time series either assume causal sufficiency (no latent
confounders)~\citep{granger1969investigating,runge2019pcmci,hyvarinen2010varlingam}
or operate on i.i.d.\ data~\citep{wang2017igsp,mooij2020jci,brouillard2020dcdi}.
In the presence of latent confounders, the structural VAR is provably
non-identifiable from observational data alone
(Fact~\ref{thm:svar_nonidentifiable}). The fundamental difficulty is
that no amount of statistical sophistication---non-Gaussianity,
heteroscedasticity, score regularity---can substitute for the
interventional data that the $\mathrm{do}(\cdot)$ operator requires.

The key observation of this paper is that \emph{the simulators already
in routine use in the physical sciences can mechanically realize
Pearl's $\mathrm{do}(\cdot)$ operator}. When a simulator is run with a
variable clamped to a fixed value (e.g., fixing temperature in an
Arrhenius kinetics model, or fixing laser intensity in a quantum
dynamics solver), confounding paths through that variable are
physically severed; the resulting output is interventional by
construction, not by statistical argument. This is categorically
different from physics-informed ML (which uses physical laws for
prediction, not causal identification) and from i.i.d.\ interventional
methods (whose interventions are experimental labels, not
simulator-realized manipulations of a time-varying process).
Incorporating simulators as causal operators into machine learning
therefore requires a framework that integrates causal inference with
time series analysis and explicitly accounts for the simulator's
imperfection---the gap between the simulator's physics and reality.

The framework we develop, \emph{SVAR-FM} (Structural
VAR with Flow Matching), uses a structural VAR \citep{sims1980macroeconomics} as the
causal language for time series and conditional Flow Matching
\citep{lipman2022flow,tong2023conditional} as the distributional
learner. The central move is that the simulator is not one more
ingredient of the learning objective but the operator that defines what
is being identified in the first place. Once the simulator is taken
seriously in this role, a new object enters the analysis---its fidelity
$\delta_{\mathcal{S}}$, which classical causal theory does not track,
and which we will show controls both the sample complexity and the sign
of the recovered effects.

We evaluate SVAR-FM on a case study in ultrafast laser
physics~\citep{corkum1993hhg} where a single first-principles
quantum-mechanical solver generates both observational and
interventional data, and the solver's physical accuracy can be tuned
continuously---providing a direct test of the sign-flip prediction.
We also evaluate on the CausalSim benchmark spanning four scientific
domains (macroeconomics, diabetes, cosmic ray physics, and battery
degradation).

Against this background, the paper addresses four interrelated
problems.

\emph{1. How should a simulator enter causal inference?}
Existing uses of simulators in ML---as prediction
targets~\citep{jumper2021alphafold,lam2023graphcast}, as supervised
data sources~\citep{cranmer2020frontier}, or as tools inside
autonomous agents~\citep{boiko2023coscientist}---treat the
simulator's forward output as the quantity of interest. None uses the
simulator as a \emph{causal operator}. We formalise the
simulator-as-$\mathrm{do}$-operator view: when the simulator clamps a
variable, confounding paths are physically severed, and the resulting
output is interventional by construction.

\emph{2. Under what conditions is the causal graph identifiable?}
In the i.i.d.\ setting, Eberhardt's
theory~\citep{eberhardt2005number,eberhardt2007interventions} shows
that $O(\log d)$ interventional experiments suffice; practical
algorithms (IGSP~\citep{wang2017igsp},
JCI~\citep{mooij2020jci}) implement this for i.i.d.\ data. For time
series, however, contemporaneous and lagged edges must be handled
jointly, and spurious high-frequency effects from unobserved lagged
confounders have no i.i.d.\ counterpart. We prove that the full SVAR
is identifiable under a \emph{coverage condition} on the
simulator-clampable variables
(Theorem~\ref{thm:svarfm_identifiability}), verifiable \emph{a priori}
from domain knowledge alone.

\emph{3. What happens when the simulator is imperfect?}
Every real simulator has finite fidelity
$\delta_{\mathcal{S}}$. Classical causal theory does not track this
quantity; the AI-for-Science literature treats it as a prediction
loss to minimise. We derive an end-to-end error bound
(Theorem~\ref{thm:simulator_error_extended}) that decomposes the
estimation error into Monte Carlo, simulator fidelity, and Flow
Matching terms. A sharp consequence is a \emph{sign-flip regime}
(Corollary~\ref{cor:robustness}): when $\delta_{\mathcal{S}}$ exceeds
a threshold set by the signal magnitude, the estimated causal effect
reverses sign---a prediction that the forward-prediction view of
simulators cannot produce.

\emph{4. Can the sign-flip prediction be tested?}
A case study in ultrafast laser physics (\S\ref{sec:extrinsic})
provides a setting where $\delta_{\mathcal{S}}$ is physically
controllable: switching between two levels of physical accuracy in the
same first-principles solver changes $\delta_{\mathcal{S}}$ from large
(sign-reversed estimate) to small (correct sign recovered,
$R^2 = 0.983$). The CausalSim
benchmark confirms the same sign-reversal pattern across four
additional domains (macroeconomics, diabetes, cosmic ray physics,
battery degradation).

In summary, the main contributions of this paper are:
\begin{itemize}
\item A framework (SVAR-FM) that treats physics-based simulators as
  mechanical realizations of Pearl's $\mathrm{do}(\cdot)$ operator for
  time series causal discovery, bridging mechanistic simulation and
  causal inference.
\item An identifiability theorem
  (Theorem~\ref{thm:svarfm_identifiability}) showing that the full
  structural VAR is recoverable under a coverage condition verifiable
  \emph{a priori} from domain knowledge, with an argument intrinsic to
  the time series setting.
\item An end-to-end error bound
  (Theorem~\ref{thm:simulator_error_extended}) that decomposes
  estimation error into Monte Carlo, simulator fidelity
  $\delta_{\mathcal{S}}$, and Flow Matching terms, and predicts a
  sign-flip regime (Corollary~\ref{cor:robustness}).
\item The CausalSim benchmark: four cross-domain experiments
  (macroeconomics, diabetes, cosmic ray physics, battery degradation)
  demonstrating that SVAR-FM recovers correct causal signs where
  observational methods fail.
\item A case study in ultrafast laser physics providing the first
  experimental test of the sign-flip prediction by physically varying
  the simulator's accuracy level.
\end{itemize}

Sections~\ref{sec:relatedworks}--\ref{sec:problem} review related work
and formalise the non-identifiability of SVARs from observational data.
Section~\ref{sec:svarfm} introduces the
simulator-as-$\mathrm{do}$-operator framework and states the
identifiability theorem; Section~\ref{sec:theory} derives the
end-to-end error bound (Theorem~\ref{thm:simulator_error_extended})
and the sign-flip corollary, connecting simulator fidelity
$\delta_{\mathcal{S}}$ to the sign of the recovered effects.
Section~\ref{sec:algorithm} describes the algorithm.
Section~\ref{sec:intrinsic_eval} reports the CausalSim benchmark
across four domains, and Section~\ref{sec:extrinsic} presents the
ultrafast laser physics case study. Section~\ref{sec:discussion} discusses positioning and
limitations, and concludes.

\section{Related Works}
\label{sec:relatedworks}

SVAR-FM proposes a use of physical simulators that has no direct
counterpart in the existing causal-discovery or AI-for-Science
literature---the simulator as a mechanical realization of Pearl's
$\mathrm{do}(\cdot)$ operator, with its fidelity $\delta_{\mathcal{S}}$
treated as a first-class object of the error analysis. To make this
positioning precise, we contrast the proposal against four research
streams whose concerns it touches: (i)~time series causal discovery,
(ii)~intervention-based causal discovery,
(iii)~causal discovery under latent confounders, and
(iv)~simulation-based and mechanistic approaches to causal inference.
For each stream we identify, in turn, what is and is not shared with
SVAR-FM in the choice of data source, the role assigned to the
simulator, the form of the identifiability argument, and the object
that controls the error.
Table~\ref{tab:method_comparison} summarises the positioning at a glance.
Table~\ref{tab:ts_detail} (Appendix~\ref{app:iid_detail}) gives the
full time-series listing; Table~\ref{tab:iid_detail}
(Appendix~\ref{app:iid_detail}) gives the i.i.d.\ listing;
the rest of this section develops the argument in prose.

\begin{table*}[t]
\centering
\caption{\textbf{Overview:} positioning of representative methods across the
assumption space.
``Setting'' is i.i.d.\ or time series (TS);
``Conf.'' denotes latent confounders;
``Int.'' denotes use of interventional data;
``NL'' denotes support for nonlinear mechanisms;
``Graph'' indicates whether the causal graph is an input (\emph{known}) or an
output (\emph{discovered}).
Full listings of i.i.d.\ and TS methods appear in
Appendix~\ref{app:iid_detail} (Tables~\ref{tab:iid_detail} and~\ref{tab:ts_detail}).}
\label{tab:method_comparison}
\small
\begin{tabular}{lcccccl}
\toprule
Method & Setting & Conf. & Int. & NL & Graph \\
\midrule
\multicolumn{6}{l}{\emph{i.i.d.\ observational}} \\
ANM \citep{hoyer2009anm}                        & i.i.d. & $\times$ & $\times$ & $\bigcirc$ & discovered \\
Nonlinear CD w/ latent \citep{kaltenpoth2023nonlinear} & i.i.d. & $\bigcirc$ & $\times$ & $\bigcirc$ & discovered \\
\midrule
\multicolumn{6}{l}{\emph{i.i.d.\ interventional}} \\
DCDI \citep{brouillard2020dcdi}                 & i.i.d. & $\times$     & $\bigcirc$ & $\bigcirc$ & discovered \\
JCI \citep{mooij2020jci}                        & i.i.d. & $\bigcirc$ & $\bigcirc$ & $\bigcirc$ & discovered \\
DeCaFlow \citep{almodovar2025decaflow}          & i.i.d. & $\bigcirc$ & $\bigcirc$ & $\bigcirc$ & known \\
\midrule
\multicolumn{6}{l}{\emph{Time series observational}} \\
PCMCI \citep{runge2019pcmci}                    & TS     & $\times$     & $\times$     & $\bigcirc$ & discovered \\
LPCMCI \citep{gerhardus2020high}                & TS     & $\bigcirc$ & $\times$     & $\bigcirc$ & discovered (PAG) \\
Rahmani--Frossard \citep{rahmani2025flow}       & TS     & $\times$     & $\times$     & $\bigcirc$ & discovered \\
\midrule
\multicolumn{6}{l}{\emph{Time series flow-based, graph known}} \\
DoFlow \citep{wu2025doflow}                     & TS     & $\times$     & $\bigcirc$ & $\bigcirc$ & known \\
\midrule
\multicolumn{6}{l}{\emph{Time series, graph discovered, latent confounders, interventional (ours)}} \\
\textbf{SVAR-FM (ours)}                         & \textbf{TS} & $\bigcirc$ & $\bigcirc$ & $\bigcirc$ & \textbf{discovered} \\
\bottomrule
\end{tabular}
\end{table*}

\subsection{Time series causal discovery}
\label{sec:rw_tscd}

Granger causality~\citep{granger1969investigating} and its
measure-theoretic extensions~\citep{dahlhaus2003causality,eichler2012graphical}
recover directed relationships from predictive improvement under causal
sufficiency and linearity.
Constraint-based methods---PCMCI~\citep{runge2019pcmci},
PCMCI$^+$~\citep{runge2020discovering}, and the decadal synthesis
of~\citet{runge2023causal}---lift linearity via nonparametric tests but
retain causal sufficiency. VARLiNGAM~\citep{hyvarinen2010varlingam} and
SpinSVAR~\citep{misiakos2025spinsvar} exploit non-Gaussianity for
linear-SVAR identification.

Score-based and deep-learning approaches---DYNOTEARS~\citep{pamfil2020dynotears},
Neural Granger~\citep{tank2021neural},
Rhino~\citep{gong2023rhino},
CUTS/CUTS$^+$~\citep{cheng2023cuts},
Amortized CD~\citep{lowe2022amortized},
TS-CausalNN~\citep{assaad2022tscausal},
\citet{sun2023dcd},
temporal score matching~\citep{chen2024score},
and score-informed neural operators~\citep{kang2025score}---broaden the
mechanism class but still operate on observational data under causal
sufficiency. Information-theoretic sample-complexity
bounds~\citep{yin2023information,veedu2023information,zhu2024sample}
place the observational baseline.

A parallel flow-based line---CAREFL~\citep{khemakhem2021carefl},
Causal NF~\citep{javaloy2023causalnf},
OCDaf~\citep{kamkari2023ocdaf},
\citet{hoang2024enabling},
CASTOR~\citep{rahmani2023castor},
and~\citet{rahmani2025flow,rahmani2025nonstationary}---parameterizes
causal generative processes with flows and extracts identifiability
from statistical noise properties (non-Gaussianity,
heteroscedasticity), requiring causal sufficiency.

Every method above recovers causal structure by \emph{adding} an
assumption to observational data (sparsity, acyclicity, non-Gaussianity,
regime structure). SVAR-FM takes a different route: it \emph{drops}
these assumptions and takes as input a different kind of
data---simulator-generated $\mathrm{do}(\cdot)$ realizations.
Identification is handled by the simulator's $\mathrm{do}$-operator
(Theorem~\ref{thm:svarfm_identifiability}); Flow Matching then
parameterizes the interventional conditionals, enabling nonlinear
mechanism learning while preserving tolerance to latent confounders.

\subsection{Intervention-based causal discovery}
\label{sec:rw_interv}

The closest methodological relatives of SVAR-FM are methods that
discover causal structure from interventional data. In the i.i.d.\
setting, Eberhardt's
theory~\citep{eberhardt2005number,eberhardt2007interventions,eberhardt2012number}
shows that $O(\log d)$ experiments suffice for complete identification;
practical algorithms include
GIES~\citep{hauser2012characterization},
IGSP~\citep{wang2017igsp}, UT-IGSP~\citep{squires2020utigsp},
JCI~\citep{mooij2020jci}, DCDI~\citep{brouillard2020dcdi},
ENCO~\citep{lippe2022enco}, and Bicycle~\citep{rohbeck2024bicycle}.

All of these assume i.i.d.\ data; none addresses structural VARs with
contemporaneous and lagged edges. SVAR-FM shares the general principle
(use interventions for identifiability) but differs in three respects:
(i)~it operates on time series and handles the VAR structure natively;
(ii)~it admits latent confounders (Theorem~\ref{thm:svarfm_identifiability}
imposes no causal sufficiency); (iii)~its interventions come from a
physical simulator whose fidelity $\delta_{\mathcal{S}}$ enters the
error bound as a first-class term.

\subsection{Causal discovery under latent confounders}
\label{sec:rw_latent}

Latent confounders require moving beyond Markov equivalence classes.
The FCI family~\citep{spirtes2000fci} identifies partial ancestral
graphs (PAGs); tsFCI~\citep{entner2010tsfci},
LPCMCI~\citep{gerhardus2020high}, and
SVAR-GFCI~\citep{malinsky2018svargfci} extend this to time series.
In the i.i.d.\ setting, JCI~\citep{mooij2020jci} and
DeCaFlow~\citep{almodovar2025decaflow} handle latent confounders
with interventional data.

SVAR-FM also admits latent confounders, but its identification route is
fundamentally different: the FCI methods use conditional-independence
constraints to output a PAG (an equivalence class), while SVAR-FM uses
physically realized interventions to identify the full SVAR (a single
graph, not an equivalence class).
In terms of Pearl's do-calculus~\citep{pearl2009causality}, the
identification strategy of SVAR-FM relies primarily on Rule~2 (the
action/observation exchange): under an intervention that severs all
backdoor paths, $P(Y \mid \mathrm{do}(X{=}x), Z) = P(Y \mid X{=}x, Z)$,
which justifies using the simulator output as if it were experimental
data. Settings requiring Rule~1 (insertion/deletion of observations)
or Rule~3 (insertion/deletion of actions)---such as front-door
identification when no variable can be directly intervened upon---fall
outside the current scope of SVAR-FM. Extending the framework to
exploit such indirect identification strategies via the simulator is a
direction for future work.

\subsection{Simulation-based and mechanistic approaches}
\label{sec:rw_sim}

Several lines of work bring simulators and mechanistic models into
contact with causal inference, but with different goals.
Simulation-based inference
(SBI)~\citep{cranmer2020frontier,brehmer2020mining,lueckmann2021benchmarking,radev2023jana}
estimates posterior distributions over the parameters of a mechanistic
model whose causal structure is \emph{fixed}; SVAR-FM \emph{discovers}
the structure.
Park et al.~\citep{park2023gobi} (GOBI) use ODE
data-reproducibility for causal inference; the causal criterion is the
ability to reproduce an observed time series, not a do-operator
intervention.
Deep SCMs~\citep{pawlowski2020deepscm,sanchez2022diffscm,chao2024diffusion}
and DoFlow~\citep{wu2025doflow} learn causal mechanisms under a
\emph{known} graph; SVAR-FM discovers the graph.
DeCaFlow~\citep{almodovar2025decaflow} and
PO-Flow~\citep{wu2025poflow} extend this to confounders and potential
outcomes, respectively, but still require the graph as input.
CaTSG~\citep{xia2025catsg} generates causal time series under a known
graph. Identifiable flow models
(Le et al.~\citep{le2025identifiable}) learn mechanisms from a known
ordering using conditional Flow Matching---the same technical primitive
as SVAR-FM's Phase~4, but solving a different problem (mechanism
learning vs.\ structure discovery).

\subsection{AI for Science and the role of simulators}
\label{sec:rw_ai4science}

The past five years have seen the emergence of \emph{AI for Science}
(AI4S) as a recognised research agenda at major ML
venues~\citep{wang2023scientific,karniadakis2021physics}. The broader
ecosystem now includes dedicated workshops at NeurIPS, ICML and ICLR
(AI4Science, AI4Mat, ML4Science), a growing set of scientific foundation
models, and a parallel push towards autonomous research agents.
Much of this work can be organised around three recurring uses of
simulators.

\textbf{(a) Simulators as prediction targets.}
A first line trains ML models \emph{to replace} a first-principles
simulator whose output is itself the quantity of scientific interest.
Representative examples include AlphaFold and its successors for protein
structure~\citep{jumper2021alphafold,abramson2024alphafold3},
GraphCast~\citep{lam2023graphcast} and GenCast~\citep{price2025gencast}
for global weather, the Aurora atmospheric foundation
model~\citep{bodnar2024aurora}, GNoME~\citep{merchant2023gnome} and
MatterSim~\citep{zhang2024mattersim} for materials discovery, and MACE
for interatomic potentials~\citep{batatia2022mace}.
In these works the simulator (or the experimental ground truth it
approximates) is the target, and the ML model's job is to reproduce it
faster or more accurately.

\textbf{(b) Simulators as training-data generators.}
A second line uses simulators to supply supervised
data for tasks where real observations are scarce or expensive: PDE
surrogates \citep{brunton2024promising,li2021fno}, physics-informed
learning \citep{karniadakis2021physics,raissi2019pinn}, and
simulation-based inference with neural density
estimators~\citep{cranmer2020frontier,brehmer2020mining,lueckmann2021benchmarking,radev2023jana}.
Here the simulator is a teacher; the learned model inherits its
inductive biases and is evaluated on held-out simulator states or on
downstream tasks.

\textbf{(c) Simulators inside autonomous research pipelines.}
A third, newer line treats the simulator as one tool among many inside
an LLM-driven research agent: Coscientist~\citep{boiko2023coscientist},
the AI Scientist line~\citep{lu2025aiscientist}, and proposals for
AI co-scientists~\citep{gottweis2025coscientist} all chain simulator
calls with literature search, code execution, and (optionally) robotic
experimentation. The simulator there is one substitutable component of
an open-ended search procedure.

SVAR-FM proposes a use of simulators that is not captured by (a)--(c)
above. It is not a fourth item in the same list: the simulator here is
neither a prediction target, nor a supervised-data source, nor a
swappable tool inside an agent. It is an operator---an executable
realization of Pearl's $\mathrm{do}(\cdot)$ on the real-world data
generating process---and the thing being learned is not the simulator's
forward map but the causal structure whose effects the operator makes
identifiable.
This shift in role changes what the analysis must say about the
simulator. Its imperfections cannot be rolled into a prediction loss,
because the quantity of interest is the causal effect
$e^{*}_{i \to j}$, which depends nonlinearly on the interventional
distribution and can be sign-flipped by a bounded simulator bias
$\delta_{\mathcal{S}}$
(Theorem~\ref{thm:simulator_error_extended}, Corollary~\ref{cor:robustness}).
The bias becomes a first-class object of the analysis, with its own
error term and its own threshold for sign reversal---features that
none of the three uses (a)--(c) need to produce, because in none of
them is the object of interest a causal structure.
We therefore regard SVAR-FM not as an application of AI-for-Science
methodology to causal discovery, but as a proposal that simulators in
AI for Science have a second, distinct role beyond forward
prediction.

\section{Problem Setting}
\label{sec:problem}

\subsection{Notation}

Consider $d$ observed time series variables $\mathbf{X}_t = (X_{1,t}, \ldots, X_{d,t})^\top \in \mathbb{R}^d$ ($t = 1, \ldots, T$).
The true causal structure is represented by a contemporaneous causal matrix $B_0 \in \mathbb{R}^{d \times d}$ and lagged causal matrices $\{B_l\}_{l=1}^p$.

\subsection{Structural VAR Model}

The structural VAR (SVAR) was introduced by Sims \citep{sims1980macroeconomics} and has become a standard tool in macroeconomics \citep{kilian2017structural}.

\begin{definition}[Structural VAR (SVAR) \citep{sims1980macroeconomics}]
\label{def:svar}
The structural VAR model is defined as:
\begin{equation}
B_0 \mathbf{X}_t = \sum_{l=1}^{p} B_l \mathbf{X}_{t-l} + \boldsymbol{\epsilon}_t
\label{eq:svar}
\end{equation}
where $B_0$ is the contemporaneous causal structure (with unit diagonal entries), $B_l$ represents the causal effects at lag $l$, and $\boldsymbol{\epsilon}_t$ denotes the structural shocks (mutually independent, $\E[\boldsymbol{\epsilon}_t] = 0$, $\mathrm{Cov}(\boldsymbol{\epsilon}_t) = \Sigma_\epsilon$ is diagonal).
\end{definition}

Rearranging Eq.~(\ref{eq:svar}) yields the reduced-form VAR:
\begin{equation}
\mathbf{X}_t = \sum_{l=1}^{p} \Phi_l \mathbf{X}_{t-l} + \mathbf{u}_t
\label{eq:var_reduced}
\end{equation}
where $\Phi_l = B_0^{-1}B_l$ and $\mathbf{u}_t = B_0^{-1}\boldsymbol{\epsilon}_t$.

\subsection{The Identification Problem}

The identification problem of structural VARs has been extensively studied in econometrics \citep{sims1980macroeconomics, kilian2017structural}.

\begin{fact}[Non-identifiability of Structural VAR \citep{kilian2017structural}]
\label{thm:svar_nonidentifiable}
For the error covariance matrix $\Sigma_u = \mathrm{Cov}(\mathbf{u}_t)$, there exist infinitely many pairs $(B_0, \Sigma_\epsilon)$ satisfying
\begin{equation}
\Sigma_u = B_0^{-1} \Sigma_\epsilon (B_0^{-1})^\top
\label{eq:sigma_decomposition}
\end{equation}
For any orthogonal matrix $Q$, the pair $(B_0 Q, Q^\top \Sigma_\epsilon Q)$ also satisfies Eq.~(\ref{eq:sigma_decomposition}).
\end{fact}

\textit{(Proof: see Appendix~\ref{app:proofs}.)}

\begin{remark}[Conventional identification strategies]
Kilian and L\"utkepohl \citep{kilian2017structural} systematize the following identification strategies:
(1) Short-run restrictions: imposing zeros in $B_0$ \citep{sims1980macroeconomics};
(2) Long-run restrictions: constraining cumulative effects \citep{blanchard1989dynamic};
(3) Sign restrictions: imposing sign constraints on effects \citep{uhlig2005effects};
(4) Non-Gaussianity: SVAR-LiNGAM \citep{hyvarinen2010varlingam}.
All four strategies recover identifiability by imposing additional
statistical or sign restrictions on the SVAR; none of them admit latent
confounders. The framework developed in the remainder of this paper
takes a different route, using physically realized interventions rather
than additional statistical restrictions, and consequently is not a
fifth item in the above list.
\end{remark}

\subsection{Latent Confounders}

In Pearl's \citep{pearl2009causality} framework, the limitations of causal discovery under the presence of latent confounders are clearly established.

\begin{definition}[Confounder \citep{pearl2009causality}]
A variable $Z$ is a \textbf{confounder} of $X$ and $Y$ if it has the structure $X \leftarrow Z \to Y$.
\end{definition}

\begin{proposition}[Non-identifiability under confounding \citep{pearl2009causality}]
\label{prop:confounding_nonidentifiable}
When a latent confounder $Z$ common to $X$ and $Y$ exists, observational data $P(X, Y)$ cannot distinguish between $X \to Y$ (direct causation) and $X \leftarrow Z \to Y$ (confounding).
\end{proposition}

\subsection{Summary of Identifiability Conditions}

Table~\ref{tab:identifiability_summary} summarizes the identifiability conditions according to model class, noise assumptions, and the presence of confounding.

\begin{table*}[h]
\centering
\caption{Identifiability conditions for causal discovery methods}
\label{tab:identifiability_summary}
%\footnotesize
%\setlength{\tabcolsep}{3pt}
%\begin{tabular}{@{}cccl@{}}
\begin{tabular}{cccl}
\toprule
Model & Noise & Confounding & Identifiability \\
\midrule
Linear & Gaussian & $\times$ & Not identifiable \\
Linear & Non-Gaussian & $\times$ & Identifiable (VARLiNGAM) \\
Linear & Non-Gaussian & $\bigcirc$ & \textbf{Not identifiable} (fails under conf.) \\
Linear & Any & $\bigcirc$ & \textbf{Identifiable} (SVAR-FM$^\dagger$) \\
\midrule
Nonlinear & Any & $\times$ & Conditionally identifiable (ANM \citep{hoyer2009anm}, etc.) \\
Nonlinear & Any & $\bigcirc$ & \textbf{Identifiable} (SVAR-FM$^\dagger$) \\
\bottomrule
\multicolumn{4}{l}{\footnotesize $^\dagger$ Requires a simulator satisfying Assumption~\ref{ass:simulator} ($\delta_{\mathcal{S}} \approx 0$) and} \\
\multicolumn{4}{l}{\footnotesize the structural conditions of Remark~\ref{rem:structural_conditions}.} \\
\end{tabular}
\end{table*}

In the presence of confounding (marked with $\bigcirc$), identification is impossible from observational data alone, and intervention-based methods such as SVAR-FM are essential.

\section{Proposed Method: SVAR-FM}
\label{sec:svarfm}

\subsection{Nonlinear Structural VAR}
\label{sec:svarfm_nonlinear}

We work with a nonlinear generalization of the SVAR in
Def.~\ref{def:svar}. This is not the main novel ingredient of the paper:
the causal mechanism is nonlinear in virtually all scientific
applications of interest (HHG\footnote{High harmonic generation:
a nonlinear optical process in which molecules irradiated by intense
femtosecond laser pulses emit photons at integer multiples of the
laser frequency; see \S\ref{sec:extrinsic} for details.},
kinetics, physiology, etc.), so the
linear SVAR of Def.~\ref{def:svar} would not be applicable in the
settings we care about. The essential new ingredient---the use of a
physical simulator as a realization of Pearl's $\mathrm{do}(\cdot)$
operator---is introduced in \S\ref{sec:svarfm_do}.

\begin{definition}[Nonlinear Structural VAR]
\label{def:nonlinear_svar}
The nonlinear structural VAR is defined as:
\begin{equation}
X_{j,t} = f_j\left(\Pa^{(0)}_j(t), \Pa^{(1:p)}_j(t)\right) + \epsilon_{j,t}
\label{eq:nonlinear_svar}
\end{equation}
where $\Pa^{(0)}_j(t)$ denotes the contemporaneous parent variables, $\Pa^{(1:p)}_j(t)$ denotes the lagged parent variables, and $f_j$ is a nonlinear causal mechanism.
\end{definition}

\subsection{Flow Matching as the Distributional Learner}
\label{sec:svarfm_fm}

Conditional Flow Matching \citep{lipman2022flow, albergo2023stochastic,
chen2018neural, tong2023conditional} parameterizes the interventional
conditionals $P(\cdot \mid \mathrm{do}(X_i = x))$ produced by the
simulator. It has two roles in SVAR-FM: (1)~modeling nonlinear causal
mechanisms $f_j$ in Eq.~(\ref{eq:nonlinear_svar}), and
(2)~incorporating physical constraints from the simulator into the
conditioning vector, enabling sensitivity analysis of causal effects
with respect to simulator assumptions.
The technical details (optimal-transport conditional flow matching
(OT-CFM) loss, conditioning design, universal
approximation) are given in Appendix~\ref{app:fm_details}.

\subsection{The simulator as a realization of Pearl's $\mathrm{do}$-operator}
\label{sec:svarfm_do}

We now introduce the central object of the paper. Let $\mathcal{S}$ be a
physical simulator (e.g.,
TDDFT\footnote{Time-dependent density functional theory: a
quantum-mechanical method that computes the time evolution of the
electron density under external fields (e.g., laser pulses), used here
via the Octopus code~\citep{tancogne2020octopus}.},
Arrhenius kinetics\footnote{A rate-equation framework in which the
rate constant of a chemical or degradation process depends
exponentially on temperature: $k = A\exp(-E_a / k_B T)$. Widely used
in battery degradation modelling and chemical engineering.},
an agent-based model). The simulator is \emph{not} used as a surrogate of the forward
dynamics and \emph{not} used as a source of supervised training data.
Instead, we use it as an operator that, on demand, produces samples
from the interventional distribution
$P(\cdot \mid \mathrm{do}(X_i = x))$ of the nonlinear SVAR of
Def.~\ref{def:nonlinear_svar}. That is, $\mathcal{S}$ is a
\emph{mechanical realization} of Pearl's $\mathrm{do}(\cdot)$ operator:
querying $\mathcal{S}(\mathrm{do}(X_i = x))$ returns a sample of the
downstream variables that corresponds, by construction, to having
physically severed all incoming causal arrows into $X_i$ and fixed
$X_i$ at the value $x$.
Three consequences follow immediately. First, the observational and
interventional distributions are distinguished \emph{by construction},
not by a statistical proxy. Second, the identifiability argument does
not need to assume non-Gaussianity, sparsity, or any other
distributional restriction on the noise; it rests instead on which
variables the simulator can clamp. Third, the fidelity of $\mathcal{S}$
enters the error analysis as an explicit, quantified object
$\delta_{\mathcal{S}}$ (Assumption~\ref{ass:simulator} below), rather
than as unmodelled noise.

\begin{assumption}[Simulator validity]
\label{ass:simulator}
The physical simulator $\mathcal{S}$ approximates the true interventional distribution:
$\|P_{\mathcal{S}}(\cdot | \mathrm{do}(X_i = x)) - P(\cdot | \mathrm{do}(X_i = x))\|_{TV}\footnote{$\|\cdot\|_{TV}$ denotes the total variation distance between two probability distributions: $\|P - Q\|_{TV} = \sup_A |P(A) - Q(A)|$. It ranges from 0 (identical distributions) to 1 (distributions with disjoint support).} \leq \delta_{\mathcal{S}}$
\end{assumption}

\begin{remark}[Structural conditions for the simulator-as-do equivalence]
\label{rem:structural_conditions}
Assumption~\ref{ass:simulator} is a distributional condition (TV
distance). For the simulator to faithfully realize Pearl's
$\mathrm{do}(\cdot)$ operator, three structural conditions must also
hold, which we state explicitly:
(a)~\emph{Structural fidelity}: the internal causal structure of
$\mathcal{S}$ is consistent with the causal DAG of the target system,
in the sense that the parent--child relationships encoded in the
simulator's equations correspond to those of the true SCM;
(b)~\emph{Modularity} (the independent causal mechanism principle
\citep{peters2017elements,scholkopf2012causal}): the operation
$\mathrm{do}(X_i = x)$ within $\mathcal{S}$ replaces exactly the
structural equation for $X_i$ while leaving all other equations
invariant;
(c)~\emph{Variable correspondence}: the simulator and the target
system are defined over the same set of endogenous variables (or a
known superset thereof).
In the applications of this paper, conditions (a)--(c) are satisfied
by construction: each simulator implements a well-validated physical
model whose causal structure is known from domain science (e.g.,
the Navier--Stokes equations, Arrhenius kinetics, the UVA/Padova
ODE system). When these structural conditions hold,
Assumption~\ref{ass:simulator} reduces the gap between $\mathcal{S}$
and the true system to a purely quantitative discrepancy
$\delta_{\mathcal{S}}$ (numerical accuracy, parameter uncertainty),
which is the object analysed in
Theorem~\ref{thm:simulator_error_extended}.
When condition (a) is violated---e.g., the simulator omits a relevant
variable---the TV bound may still hold but the causal interpretation
breaks down; in such cases, $\delta_{\mathcal{S}}$ absorbs structural
model error and may be large.
\end{remark}

\begin{lemma}[Identification of contemporaneous causation via intervention \citep{pearl2009causality, eberhardt2007interventions}]
\label{lem:contemporaneous_identification}
Under Assumption~\ref{ass:simulator}, the interventional effect
$e_{i \to j} = \E[X_j | \mathrm{do}(X_i = x')] - \E[X_j | \mathrm{do}(X_i = x)]$
identifies the contemporaneous causal direction.
\end{lemma}

\textit{(Proof: see Appendix~\ref{app:proofs}.)}

\begin{definition}[Intervention target set and coverage]
\label{def:intervention_coverage}
Let $\mathcal{I} \subseteq \{1, \ldots, d\}$ denote the set of variables for
which simulator-generated interventional distributions
$P_{\mathcal{S}}(\cdot \mid \mathrm{do}(X_{i,t} = x))$ are available for a
sufficiently rich set of intervention values $x \in \mathcal{X}_i$ with
$|\mathcal{X}_i| \ge 2$, for all $t$ in the estimation window.
We say that $\mathcal{I}$ \emph{covers} the graph $\mathcal{G}$ if every
contemporaneous edge $(i,j)$ with $i \ne j$ in $\mathcal{G}$ has $i \in \mathcal{I}$
or $j \in \mathcal{I}$; equivalently, no two endogenous variables are simultaneously
un-intervened.
\end{definition}

\begin{assumption}[Regularity of the SVAR process]
\label{ass:regularity}
The SVAR in Def.~\ref{def:nonlinear_svar} is
(a)~\emph{stable}, in the sense that the companion matrix of the lagged
coefficients has spectral radius strictly less than one;
(b)~\emph{acyclic in contemporaneous time}, i.e., the contemporaneous edge set
$\{(i,j) : B_0^{ij} \ne 0,\, i \ne j\}$ forms a DAG;
(c)~the structural shocks $\epsilon_{j,t}$ are mutually independent across
$j$ and independent across $t$ with finite second moments.
\end{assumption}

\begin{theorem}[Interventional identifiability of SVAR-FM]
\label{thm:svarfm_identifiability}
Let Assumptions~\ref{ass:simulator} (with $\delta_{\mathcal{S}} = 0$) and
\ref{ass:regularity} hold, and let $\mathcal{I}$ be an intervention target set
that covers the contemporaneous graph of $\mathcal{G}$.
Then both the contemporaneous causal structure $B_0$ and all lagged causal
structures $\{B_l\}_{l=1}^{p}$ of the SVAR in Def.~\ref{def:nonlinear_svar}
are uniquely identifiable from the joint of the observational distribution
$P(\mathbf{X}_{1:T})$ and the family of interventional distributions
$\{P_{\mathcal{S}}(\mathbf{X}_{1:T} \mid \mathrm{do}(X_{i,t} = x)) :
i \in \mathcal{I},\, x \in \mathcal{X}_i,\, t\}$.
Moreover, $\mathcal{I} = \{1, \ldots, d\}$ is sufficient but not necessary;
for chain graphs, $|\mathcal{I}| = d - 1$ suffices.
\end{theorem}

\textit{(Proof: see Appendix~\ref{app:proofs}.)}

\begin{remark}[From exact to approximate identifiability]
\label{rem:exact_to_approximate}
Theorem~\ref{thm:svarfm_identifiability} establishes identifiability
under the idealisation $\delta_{\mathcal{S}} = 0$. When
$\delta_{\mathcal{S}} > 0$ (as in every real simulator), the causal
structure is no longer exactly identifiable, but the end-to-end error
bound of Theorem~\ref{thm:simulator_error_extended}
(\S\ref{sec:theory}) quantifies the resulting estimation error and
shows that it decomposes cleanly into Monte Carlo, simulator fidelity,
and Flow Matching components. In particular,
Corollary~\ref{cor:robustness} gives a threshold on
$\delta_{\mathcal{S}}$ below which the estimated sign of each causal
effect is preserved---the practical condition under which the
identifiability of Theorem~\ref{thm:svarfm_identifiability} degrades
gracefully rather than catastrophically. For instance, in
CausalSim-Macro the true effect is $|e^*| = 0.006$; the sign is
preserved whenever $\delta_{\mathcal{S}} < |e^*|/2 = 0.003$, and
the observed bias of $0.001$ (Table~\ref{tab:macro_results}) confirms
that this condition holds.
\end{remark}

\begin{remark}[Operational meaning of the coverage condition]
\label{rem:coverage_operational}
Def.~\ref{def:intervention_coverage} is what we verify in each application.
For instance, in the laser physics case study (\S\ref{sec:extrinsic}),
there are two correlated variables ($R$ and $E_0$) and a single
confounding edge $R \leftrightarrow E_0$; fixing $R$ at a single value
places $R$ in $\mathcal{I}$, which covers this edge by
Def.~\ref{def:intervention_coverage}.
More generally, the coverage condition can be verified \emph{a priori}
from domain knowledge of which variable pairs the simulator is capable of
holding independently fixed.
\end{remark}

\begin{fact}[Separation of confounding via intervention \citep{pearl2009causality}]
\label{thm:confounding_separation}
When $Z$ is a confounder of $X_i$ and $X_j$, $\mathrm{do}(Z = z)$ distinguishes confounding from direct causation.
This is a standard result from Chapter~3 of Pearl \citep{pearl2009causality}.
\end{fact}

\textit{(Proof: see Appendix~\ref{app:proofs}.)}

\begin{remark}[Addressing unobserved confounding via the simulator]
\label{rem:unobserved_confounding}
Theorem~\ref{thm:svarfm_identifiability} and the known result (Fact~\ref{thm:confounding_separation}) primarily address \textbf{contemporaneous causation among observed variables} ($X \leftrightarrow Y$).
However, SVAR-FM can also address \textbf{unobserved confounders} ($U \to X, U \to Y$ where $U$ is unobservable).

In the unobserved confounding case, the causal effect $X \to Y$ cannot be identified from observational data alone:
\begin{equation}
P(Y | X = x) \neq P(Y | \mathrm{do}(X = x))
\end{equation}
This arises because the backdoor path $X \leftarrow U \to Y$ remains open.
The discrepancy between $P(Y \mid X)$ and $P(Y \mid \mathrm{do}(X))$
is the mechanism behind Simpson's paradox~\citep{simpson1951interpretation,pearl2009causality}:
a statistical association can reverse sign when a confounding variable
is conditioned on.

In SVAR-FM, we exploit the fact that the simulator $\mathcal{S}$ can \textbf{specify and control} the values of the unobserved variable $U$:
\begin{equation}
P_{\mathcal{S}}(Y | \mathrm{do}(X = x), U = u) = P(Y | X = x, U = u)
\end{equation}
Since the simulator can fix $U$ and generate intervention data, the backdoor path through $U$ can be blocked.

This property enables SVAR-FM to address the following two types of confounding in a unified manner:
\begin{itemize}%[nosep]
    \item \textbf{Simultaneous causation} (e.g., in macroeconomics,
          a central bank adjusts interest rates $i$ in response to
          inflation $\pi$, creating a feedback loop
          $i \leftrightarrow \pi$ that masks the true causal effect of
          $i$ on $\pi$; see \S\ref{subsec:macro}):
          intervention \textit{severs} the reverse causal direction
    \item \textbf{Unobserved confounding} (e.g., in the laser physics
          case study of \S\ref{sec:extrinsic}, a molecular parameter
          affects both the laser field and the spectral output, but
          cannot be independently measured in a real experiment):
          the simulator \textit{fixes} the confounding variable
\end{itemize}
Although the mechanisms differ, both are resolved within the unified framework of ``intervention via a simulator.''
\end{remark}

\section{Theoretical Analysis}
\label{sec:theory}

The theory of this section is organized around the object that the
simulator-as-$\mathrm{do}$-operator viewpoint introduces, namely the
simulator fidelity $\delta_{\mathcal{S}}$, and its consequences for the
recovery of causal effects. Two results are where the new content is
concentrated:
(i)~interventional identifiability of the time series causal structure
under an explicit coverage condition on the intervention target set
(Theorem~\ref{thm:svarfm_identifiability}), which is not an instance of
a previously available theorem---it has to handle the simultaneous
presence of contemporaneous and lagged edges and to distinguish
genuine contemporaneous causation from spurious high-frequency effects
due to unobserved lagged confounders, neither of which arises in
i.i.d.\ formulations;
(ii)~an end-to-end decomposition of the total estimation error into
Monte Carlo, simulator fidelity, and Flow Matching approximation
components
(Theorem~\ref{thm:simulator_error_extended} and
Corollary~\ref{cor:sample_complexity_fm}), together with its
sign-flip consequence
(Corollary~\ref{cor:robustness}): whenever $\delta_{\mathcal{S}}$
exceeds a threshold determined by the signal magnitude, the estimated
causal effect is reversed in sign relative to the ground
truth---a prediction we verify in the ultrafast laser physics case
study (\S\ref{subsec:femtosecond}) by switching between a low-accuracy
and a high-accuracy first-principles solver.

The remaining results in the section---the convergence rate of the
interventional expectation estimator
(Proposition~\ref{thm:convergence_rate}), the sample complexity for
joint identification of all pairwise effects
(Proposition~\ref{thm:sample_complexity}), the lower and upper bounds on
the number of interventions
(Corollaries~\ref{thm:intervention_lower_bound}--\ref{thm:intervention_bounds}),
and consistency under vanishing estimation error
(Proposition~\ref{thm:consistency})---are obtained by specializing
classical statistical and PAC-learning tools
\citep{van2000asymptotic,wasserman2006all,valiant1984theory,
hoeffding1963probability}
and the interventional combinatorics of \citet{eberhardt2005number,
eberhardt2007interventions}
to the SVAR-FM setting, so that the paper is self-contained and the
role of each component of the error bound is visible.

\subsection{Convergence Rate of Intervention Effect Estimation}

The following result records the convergence rate of the Monte Carlo
estimator of the interventional expectation, for use later in
Theorem~\ref{thm:simulator_error_extended}.

\begin{proposition}[Convergence rate \citep{van2000asymptotic, hoeffding1963probability}]
\label{thm:convergence_rate}
For the estimation error between the intervention effect $\hat{e}_{i \to j}$ estimated from $M$ intervention samples and the true value $e_{i \to j}^*$, when the variance of $X_j | \mathrm{do}(X_i = x)$ is bounded by $\sigma^2$:
\begin{equation}
|\hat{e}_{i \to j} - e_{i \to j}^*| = O_p(M^{-1/2})
\label{eq:convergence_rate}
\end{equation}
More precisely, by Hoeffding's inequality \citep{hoeffding1963probability}:
\begin{equation}
P\left(|\hat{e}_{i \to j} - e_{i \to j}^*| > \epsilon\right) \leq 2\exp\left(-\frac{M\epsilon^2}{2\sigma^2}\right)
\label{eq:hoeffding}
\end{equation}
\end{proposition}

\textit{(Proof: see Appendix~\ref{app:proofs}.)}

\begin{corollary}[Confidence interval \citep{wasserman2006all}]
\label{cor:confidence_interval}
The confidence interval for the intervention effect at confidence level $1-\alpha$ is given by $\hat{e}_{i \to j} \pm z_{\alpha/2} \cdot \hat{\sigma}/\sqrt{M}$.
\end{corollary}

\subsection{Sample Complexity}

We record the sample complexity for recovering all pairwise causal
effects from simulator-generated interventions, specializing the PAC
learning framework \citep{valiant1984theory, kearns1994introduction} to
the SVAR-FM setting.

\begin{proposition}[Sample complexity \citep{valiant1984theory, wasserman2006all}]
\label{thm:sample_complexity}
The total number of intervention samples required to identify the causal structure of $d$ variables with error probability at most $\delta$ and estimation accuracy $\epsilon$ is:
\begin{equation}
M_{\text{total}} = O\left(\frac{d^2 \sigma^2 \log(d^2/\delta)}{\epsilon^2}\right)
\label{eq:sample_complexity}
\end{equation}
\end{proposition}

\textit{(Proof: see Appendix~\ref{app:proofs}.)}

\begin{remark}[Comparison with existing methods]
The sample complexity $O(d^2)$ is of the same order as the number of conditional independence tests $O(d^2 p)$ in PCMCI \citep{runge2019pcmci}.
However, SVAR-FM enjoys the advantage of being robust to confounding.
\end{remark}

\subsection{Lower and Upper Bounds on the Required Number of Interventions}

The interventional combinatorics of \citet{eberhardt2005number,
eberhardt2007interventions}, originally formulated for i.i.d.\ DAGs,
continues to govern the SVAR-FM setting once the coverage condition
of Def.~\ref{def:intervention_coverage} is in force; we record the
resulting bounds for use in the error analysis.

\begin{corollary}[Lower bound on the number of interventions]
\label{thm:intervention_lower_bound}
To identify the causal structure of $d$ variables, at least $d-1$ single-variable interventions are required in the worst case.
\end{corollary}

\textit{(Proof: see Appendix~\ref{app:proofs}.)}

\begin{corollary}[Upper bound on the number of interventions]
\label{thm:intervention_upper_bound}
With $d$ single-variable interventions (one per variable), any causal structure over $d$ variables can be identified.
\end{corollary}

\textit{(Proof: see Appendix~\ref{app:proofs}.)}

\begin{corollary}[Exact number of interventions \citep{eberhardt2005number, eberhardt2007interventions}]
\label{thm:intervention_bounds}
The required number of interventions $I^*$ satisfies $d - 1 \leq I^* \leq d$.
The lower bound $I^* = d-1$ is attained for chain structures $X_1 \to X_2 \to \cdots \to X_d$.
\end{corollary}

\textit{(Proof: see Appendix~\ref{app:proofs}.)}

\subsection{Computational Complexity}

\begin{proposition}[Computational complexity]
\label{thm:computational_complexity}
The computational complexity of SVAR-FM is $O(d^2 p T + d \cdot M \cdot C_{\mathcal{S}} + d \cdot T \cdot C_{\text{FM}})$,
where $C_{\mathcal{S}}$ is the simulator cost and $C_{\text{FM}}$ is the Flow Matching training cost.
\end{proposition}

\subsection{Consistency}

\begin{proposition}[Consistency]
\label{thm:consistency}
Under Assumption~\ref{ass:simulator}, inclusion of the true causal mechanism in the model class, and $T, M \to \infty$, the estimated causal structure $\hat{\mathcal{G}}$ of SVAR-FM converges in probability to the true structure $\mathcal{G}^*$.
\end{proposition}

\textit{(Proof: see Appendix~\ref{app:proofs}.)}

\subsection{Impact of Simulator Error and Evaluation of Physical Constraints}

The following theorem quantifies the degree to which physical constraints inherent in the simulator (approximation accuracy of governing equations, numerical solver errors, etc.)\ affect causal estimation.
Combined with the conditional generative model of Flow Matching (Remark~\ref{rem:fm_physics}), the sensitivity of estimates to changes in physical constraints (e.g., changes in physical parameters or approximation levels of the simulator) can be theoretically assessed.

\begin{theorem}[Propagation of simulator error and impact assessment of physical constraints]
\label{thm:simulator_error}
When the simulator error is bounded by $\delta_{\mathcal{S}}$ (Assumption~\ref{ass:simulator}), the estimation error of the intervention effect satisfies:
\begin{equation}
|\hat{e}_{i \to j} - e_{i \to j}^*| \leq O(M^{-1/2}) + O(\delta_{\mathcal{S}})
\label{eq:simulator_error}
\end{equation}
The first term is the statistical estimation error (dependent on the sample size $M$), and the second term is the systematic bias arising from the physical assumptions of the simulator.
Through Flow Matching, the variation of $\delta_{\mathcal{S}}(\mathbf{c})$ under different physical parameters $\mathbf{c}$ can be continuously modeled, enabling quantitative assessment of the impact of physical constraints on causal estimation.
\end{theorem}

\textit{(Proof: see Appendix~\ref{app:proofs}.)}

\begin{corollary}[Robustness condition]
\label{cor:robustness}
For the identification of the causal effect $e^*$ to be robust, it is necessary that $|e^*| > 2\delta_{\mathcal{S}} + O(M^{-1/2})$.
\end{corollary}

\subsection{End-to-end error propagation through Flow Matching}
\label{subsec:end_to_end_error}

Theorem~\ref{thm:simulator_error} bounds the error of a \emph{direct}
Monte Carlo estimator of the interventional expectation.
In SVAR-FM, however, the interventional distribution is not used directly:
it is first approximated by a conditional Flow Matching model, and the
causal effect is then read off from this model.
The final estimation error therefore combines three sources: the Monte Carlo
error, the simulator fidelity bias, and the Flow Matching approximation error.
We now quantify this composition.

Let $v^*$ denote the true (simulator-induced) conditional vector field that
generates $P_{\mathcal{S}}(\cdot \mid \mathrm{do}(X_i = x))$, and let
$\hat v_{\theta}$ denote the estimate returned by minimizing the empirical
conditional Flow Matching loss
$\widehat{\mathcal{L}}_{\mathrm{CFM}}(\theta)$ over the simulator-generated
intervention dataset of size $M$.
Denote the Flow Matching approximation error by
\[
\varepsilon_{\mathrm{FM}}
\;\;:=\;\;
\big( \mathcal{L}_{\mathrm{CFM}}(\hat v_\theta)
- \mathcal{L}_{\mathrm{CFM}}(v^{*}) \big)^{1/2}.
\]
This is the excess risk of the learned vector field with respect to the
population CFM objective, and is upper-bounded by standard statistical-learning
bounds in the form
$\varepsilon_{\mathrm{FM}} \le \mathcal{O}\big(\mathrm{Rad}_M(\mathcal{V}_\theta)\big)$
whenever $\mathcal{V}_\theta$ has bounded Rademacher complexity
\citep{van2000asymptotic}.

\begin{assumption}[Lipschitz flow and bounded response]
\label{ass:fm_lip}
(a)~The target vector field $v^*(\cdot, t \mid \mathbf{c})$ is $L$-Lipschitz in
$\mathbf{x}$ uniformly in $(t, \mathbf{c})$;
(b)~the response functional $g_j(P) := \mathbb{E}_{X_j \sim P}[X_j]$ is
$1$-Lipschitz with respect to the Wasserstein-1 distance on the marginal of
$X_j$, which holds whenever $X_j$ has bounded conditional variance under the
interventional distribution.
\end{assumption}

Assumption~\ref{ass:fm_lip}(a) is standard for flow-matching-based density
estimation and implies, via a Gronwall-type argument
\citep{albergo2023stochastic,tong2023conditional}, that the
Wasserstein-1 distance between the distributions generated by $\hat v_\theta$
and $v^*$ is controlled by their $L^2$ difference:
$W_1\big(\hat P_{\hat\theta}, P_{\mathcal{S}}\big)
\le e^{L} \cdot \varepsilon_{\mathrm{FM}}.$

\begin{theorem}[End-to-end error bound]
\label{thm:simulator_error_extended}
Under Assumptions~\ref{ass:simulator}, \ref{ass:regularity}, and
\ref{ass:fm_lip}, the Flow-Matching-based estimator
$\hat e_{i \to j}^{\mathrm{FM}}
:= g_j(\hat P_{\hat\theta}(\cdot \mid \mathrm{do}(X_i = x')))
- g_j(\hat P_{\hat\theta}(\cdot \mid \mathrm{do}(X_i = x)))$
satisfies
\begin{equation}
|\hat e_{i \to j}^{\mathrm{FM}} - e_{i \to j}^{*}|
\;\le\;
\underbrace{C_1 \sigma / \sqrt{M}}_{\text{Monte Carlo}}
\;+\;
\underbrace{C_2 \cdot \delta_{\mathcal{S}}}_{\text{simulator fidelity}}
\;+\;
\underbrace{C_3 \cdot e^{L} \cdot \varepsilon_{\mathrm{FM}}}_{\text{FM approximation}},
\label{eq:error_bound_extended}
\end{equation}
with probability at least $1 - \eta$, where
$C_1 = \sqrt{2 \log(2/\eta)}$ and
$C_2, C_3 \le 2 \cdot \sup_{\mathbf{x}} |x_j|$
are constants that depend only on the conditional range of $X_j$.
\end{theorem}

\textit{(Proof: see Appendix~\ref{app:proofs}.)}

\begin{remark}[Which term dominates?]
\label{rem:dominance}
Which term of Eq.~\eqref{eq:error_bound_extended} dominates is a
\emph{diagnostic} quantity that SVAR-FM can estimate empirically.
(i)~When the simulator is a first-principles solver with well-controlled
discretization error, $\delta_{\mathcal{S}}$ is small and the error is
limited by $M$ or $\varepsilon_{\mathrm{FM}}$.
(ii)~When the simulator's physical model omits an important
effect, $\delta_{\mathcal{S}}$ becomes $O(1)$ relative to the signal,
and Corollary~\ref{cor:robustness} predicts a \emph{sign reversal}
rather than a smooth degradation.
This is exactly what we observe in the laser physics case study
(\S\ref{subsec:femtosecond}), where switching from a low-accuracy to a
high-accuracy solver collapses the dominant term from
$\delta_{\mathcal{S}}$ to $\varepsilon_{\mathrm{FM}}$ and restores the
correct causal sign.
\end{remark}

\begin{corollary}[Sample complexity revisited under FM approximation]
\label{cor:sample_complexity_fm}
To achieve an overall estimation accuracy $\epsilon$ on $d(d-1)$ causal
effects simultaneously with probability at least $1 - \delta$, it suffices
to have intervention sample size
$M \ge C \sigma^2 \log(d^2/\delta) / (\epsilon/3)^2$,
simulator fidelity $\delta_{\mathcal{S}} \le \epsilon/(3 C_2)$,
and Flow Matching approximation error
$\varepsilon_{\mathrm{FM}} \le \epsilon / (3 C_3 e^{L})$.
\end{corollary}

% 5. Algorithm and Methodology

\section{Algorithm and Methodology}
\label{sec:algorithm}

This section describes the SVAR-FM algorithm (Section~\ref{subsec:alg_detail}), the architectural design principles (Section~\ref{subsec:arch}), and methodological guidance for practical application (Section~\ref{subsec:methodology}).

% ---------------------------------------------------------------------------
\subsection{Algorithm}
\label{subsec:alg_detail}

SVAR-FM consists of five phases (Algorithm~\ref{alg:svarfm}).
Phases~1--3 are responsible for causal structure identification, while Phases~4--5 handle the learning of causal mechanisms via Flow Matching and the evaluation of physical constraints.

%\begin{algorithm}[t]
%\begin{algorithm}[H]
%\begin{algorithm}[tb]
\begin{algorithm}
\caption{SVAR-FM}
\label{alg:svarfm}
\begin{algorithmic}[1]
\REQUIRE Observed time series $\mathcal{D}_{obs} = \{\mathbf{y}_t\}_{t=1}^T$, simulator $\mathcal{S}$, maximum lag $p_{\max}$, significance level $\alpha$
\ENSURE Causal graph $\hat{\mathcal{G}}$, causal mechanisms $\{v_{\theta_j}\}$, sensitivities $\{s_{ij}\}$

\STATE \textbf{// Phase 1: Reduced-form VAR estimation} \citep{hamilton1994time}
\STATE $p^* \gets \arg\min_{p \leq p_{\max}} \mathrm{BIC}(p)$
\STATE $\{\hat{\Phi}_l\}_{l=1}^{p^*} \gets$ OLS estimation
\STATE $\hat{\mathbf{u}}_t \gets \mathbf{y}_t - \sum_{l=1}^{p^*} \hat{\Phi}_l \mathbf{y}_{t-l}$

\STATE \textbf{// Phase 2: Intervention data generation via simulator}
\FOR{$i = 1, \ldots, d$}
    \STATE Design intervention values $\{x_i^{(m)}\}_{m=1}^M$ within the physically valid range $[\underline{x}_i, \overline{x}_i]$
    \FOR{$m = 1, \ldots, M$}
        \STATE $\mathbf{y}^{(m)} \gets \mathcal{S}.\mathrm{do}(X_i = x_i^{(m)})$
    \ENDFOR
\ENDFOR

\STATE \textbf{// Phase 3: Causal structure identification} (Lemma~\ref{lem:contemporaneous_identification})
\STATE $\hat{\mathcal{G}} \gets \emptyset$
\FOR{$(i, j)$, $i \neq j$}
    \STATE $\hat{e}_{i \to j} \gets \frac{1}{M}\sum_{m=1}^M y_j^{(m)} - \bar{y}_j^{obs}$
    \STATE $\hat{\sigma}_{ij} \gets$ bootstrap standard error ($B = 1000$ replicates)
    \IF{$|\hat{e}_{i \to j}| / \hat{\sigma}_{ij} > z_{\alpha/2}$}
        \STATE $\hat{\mathcal{G}} \gets \hat{\mathcal{G}} \cup \{X_i \to X_j\}$ \quad (significant causal edge)
    \ENDIF
\ENDFOR

\STATE \textbf{// Phase 4: Causal mechanism learning via Flow Matching} (Remark~\ref{rem:fm_physics})
\FOR{$j = 1, \ldots, d$}
    \STATE Construct conditioning vector $\mathbf{c}_j \gets [\mathbf{x}_{\Pa(j)}, \boldsymbol{\phi}]$
    \STATE \quad ($\mathbf{x}_{\Pa(j)}$: parent variables of $j$ in $\hat{\mathcal{G}}$, $\boldsymbol{\phi}$: physical parameters)
    \STATE $\theta_j \gets \arg\min_\theta \mathcal{L}_{\text{CFM}}(\theta; \mathbf{c}_j)$ \quad (Eq.~\ref{eq:cfm_loss})
\ENDFOR

\STATE \textbf{// Phase 5: Impact assessment of physical constraints}
\FOR{$(i, j) \in \hat{\mathcal{G}}$}
    \FOR{$k = 1, \ldots, |\boldsymbol{\phi}|$}
        \STATE $\boldsymbol{\phi}'_k \gets$ perturb the $k$-th component of $\boldsymbol{\phi}$ by $\pm\delta$
        \STATE $s_{ij,k} \gets |\hat{e}_{i \to j}(\boldsymbol{\phi}'_k) - \hat{e}_{i \to j}(\boldsymbol{\phi})| / \delta$
    \ENDFOR
\ENDFOR

\RETURN $\hat{\mathcal{G}}, \{v_{\theta_j}\}, \{s_{ij,k}\}$
\end{algorithmic}
\end{algorithm}

Phase~1 is standard reduced-form VAR estimation, with the lag order $p^*$ selected by BIC \citep{hamilton1994time}; this is the same initial step used by VARLiNGAM \citep{hyvarinen2010varlingam} and other VAR-based causal discovery methods.
Default hyperparameters: $p_{\max} = 10$, bootstrap replicates $B = 1000$, significance level $\alpha = 0.05$ with Bonferroni correction $\alpha' = \alpha / d(d{-}1)$.
Phase~2 generates intervention data using the simulator $\mathcal{S}$. The design of intervention values is discussed in Section~\ref{subsubsec:intervention_design}.
Phase~3 identifies causal edges through bootstrap tests of the intervention effects $\hat{e}_{i \to j}$. By employing statistical significance testing rather than threshold-based effect size judgments alone, consistency with the theoretical guarantees of Proposition~\ref{thm:convergence_rate} is maintained.

The identification formula implemented by Phase~3 can be stated
explicitly: for each variable pair $(i, j)$ with $i \in \mathcal{I}$,
the causal effect is identified as
\begin{equation}
\label{eq:identification_formula}
\hat{e}_{i \to j}
  = \frac{1}{M} \sum_{m=1}^{M} X_{j}^{(m)}\big|_{\mathrm{do}(X_i = x')}
  - \bar{X}_{j}^{\mathrm{obs}}
\end{equation}
where $\{X_{j}^{(m)}\}_{m=1}^{M}$ are simulator-generated samples and
$\bar{X}_{j}^{\mathrm{obs}}$ is the observational mean. The edge
$i \to j$ is included in $\hat{\mathcal{G}}$ if and only if the
bootstrap $z$-statistic $|\hat{e}_{i \to j}| / \hat{\sigma}_{ij}$
exceeds $z_{\alpha'/2}$ under the Bonferroni-corrected threshold
$\alpha' = \alpha / d(d{-}1)$. This formula makes the
identification constructive: it specifies an explicit estimand
expressed in terms of observable (observational mean) and
experimentally accessible (simulator samples) quantities.
Phases~4--5 are optional; the causal graph $\hat{\mathcal{G}}$ is obtained from Phases~1--3 alone. The necessity of Phases~4--5 is discussed in Section~\ref{subsubsec:when_fm}.

\subsection{Architecture}
\label{subsec:arch}

Figure~\ref{fig:architecture} illustrates the SVAR-FM pipeline.
Phases~1--3 (causal structure identification) and Phases~4--5 (physical constraint modeling) are clearly separated, each serving distinct purposes.

Phases~1--3 are composed of the following data flow. First, a reduced-form VAR is estimated from the observed time series $\mathcal{D}_{obs}$ (Phase~1), and intervention data $\mathcal{D}_{int}$ are simultaneously generated using the simulator $\mathcal{S}$ (Phase~2). These two outputs are integrated to construct the causal graph $\hat{\mathcal{G}}$ via estimation and statistical testing of intervention effects (Phase~3).

Phases~4--5 provide additional value. In Phase~4, the causal mechanisms of each variable are learned via conditional Flow Matching based on the structure of $\hat{\mathcal{G}}$. By including physical parameters $\boldsymbol{\phi}$ in the conditioning vector $\mathbf{c}_j$, a generative model consistent with physical laws is obtained. In Phase~5, sensitivities $s_{ij,k}$ of the causal effects to perturbations in $\boldsymbol{\phi}$ are computed, quantifying the impact of physical constraints on the estimates.
Phase~5 is analogous in purpose to the sensitivity analysis frameworks
of \citet{vanderweele2017sensitivity} (E-value) and
\citet{robins2000sensitivity} (bounding factor), which assess how
strong an unmeasured confounder would need to be to explain away an
observed effect. In SVAR-FM, the ``unmeasured confounder'' is the
simulator's physical approximation: Phase~5 asks how large a change in
$\boldsymbol{\phi}$ is needed to reverse the sign of the estimated
causal effect.

Table~\ref{tab:simulators} lists the simulators used in this paper; the bottom rows show further applications in preparation (cf.\ \S\ref{sec:conclusion}).

\begin{figure*}[t]
\centering
% Import EPS generated from fig_architecture.tex
% Compilation: platex fig_architecture.tex -> dvipdfmx -> pdftops -eps
%\includegraphics[width=0.92\textwidth]{fig_architecture.eps}
\includegraphics[width=0.92\textwidth]{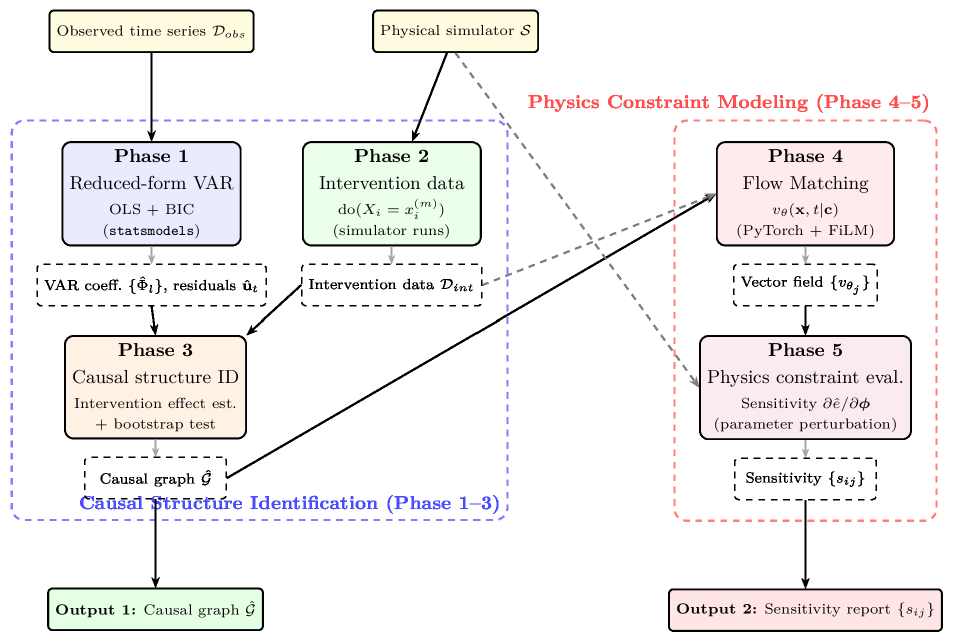}
\caption{Architecture of SVAR-FM. The inputs are the observed time series $\mathcal{D}_{obs}$ and the physical simulator $\mathcal{S}$. Phases~1--3 are responsible for identifying the causal structure $\hat{\mathcal{G}}$, and for linear causal mechanisms, the procedure is complete at this stage. Phases~4--5 learn nonlinear causal mechanisms via conditional Flow Matching and assess the impact of physical constraints.}
\label{fig:architecture}
\end{figure*}

\begin{table*}[h]
\centering
\caption{Simulators used in this paper. CausalSim instances
(\S\ref{sec:simulator_driven_eval} and Appendix~\ref{app:battery})
are listed in the middle block; the HHG case study (top) and further
applications in preparation (bottom) complete the picture. ``Type''
reflects simulator fidelity: \emph{ab initio} $>$ numerical $>$
analytical $>$ MC $>$ statistical $>$ agent-based.}
\label{tab:simulators}
\begin{tabular}{lllrl}
\toprule
Application & Simulator & Intervention variable & $M$ & Type \\
\midrule
HHG (\S\ref{sec:extrinsic}) & Octopus (TDDFT, SIC-ADSIC) & Laser amplitude $E_0$ & 10 & \textit{Ab initio} \\
\midrule
\multicolumn{5}{l}{\emph{CausalSim benchmark (\S\ref{sec:simulator_driven_eval}):}} \\
Macroeconomics & DSGE (Taylor Rule) & Policy rate $i$ & 100 & Analytical \\
Diabetes & UVA/Padova T1DMS & Insulin dose & 200 & Numerical \\
Cosmic rays & Heitler--Matthews & $\sigma_\text{inel}$ ($E$, $A$ fixed) & 500 & MC \\
Battery (App.~\ref{app:battery}) & Quantum ESPRESSO (DFT) + Arrhenius & Temperature $T$ & 50 & \textit{Ab initio} \\
\midrule
\multicolumn{5}{l}{\emph{Further applications (in preparation, cf.\ \S\ref{sec:conclusion}):}} \\
ECG & PhysioNet synthetic & Disease severity & 100 & Statistical \\
Finance & BS + VIX-linked & Sentiment & 100 & Analytical \\
Finance (rates) & Mesa ABM 3.0 & Interest rate $\Delta r$ & 371 & Agent-based \\
\bottomrule
\end{tabular}
\end{table*}

\subsubsection{SVAR-FM-DAG: NOTEARS Post-Processing Variant}
\label{sec:svarfm_dag}

SVAR-FM-DAG is the variant of SVAR-FM used when the data are expected
to obey a strict DAG constraint (e.g., Tigramite synthetic data). It
differs from the base algorithm in two components.

\paragraph{Phase 0: PINN ensemble prior knowledge}

Prior to Phase~1, a directional score is computed as prior knowledge using a linear two-stage ensemble (VAR coefficients + Ridge Granger).
When $T \geq 100$ and $N \leq 20$, a flow-matching-based nonlinear two-stage ensemble is also executed (when PyTorch is available) and integrated with the Phase~1 scores using weight $w = 0.2$.

\paragraph{NOTEARS post-processing}

The score matrix obtained in Phase~1 is subjected to the acyclicity constraint
\begin{equation}
  h(W) = \mathrm{tr}\!\left(e^{W \circ W}\right) - d = 0
  \label{eq:notears}
\end{equation}
enforced via the augmented Lagrangian method~\citep{zheng2018notears}.
This substantially reduces false discovery rate (FDR) for data where a DAG structure is expected (e.g., Tigramite synthetic data).

\subsubsection{SVAR-dyn1/SVAR-dyn2: Variants for ODE-Based Dynamical Systems}
\label{sec:svarfm_dynamics}

For dynamical systems with bidirectional edges (cycles), the NOTEARS
DAG constraint is inappropriate. SVAR-dyn1 and SVAR-dyn2 disable
NOTEARS and Phase~0, and add a differential Granger score based on
$\mathrm{d}X_j/\mathrm{d}t \approx f(X)$. SVAR-dyn2 further
integrates Phases~3--5 with adaptive weighting and an eps-guard for
deterministic chaotic systems. Full details are given in
Appendix~\ref{app:svarfm_dynamics}.

\subsection{Methodological Guidance}
\label{subsec:methodology}

Applying SVAR-FM requires several methodological decisions.
Below, we describe the conditions for applicability, the principles of intervention design, and the criteria for selecting among the framework components.

\paragraph{Variant selection.}
The choice among SVAR-FM, SVAR-FM-DAG, and SVAR-dyn1/dyn2 is guided
by domain knowledge of the system under study, not by post-hoc
performance comparison:
\begin{itemize}
\item If the causal graph is known to be \textbf{acyclic} (e.g.,
  regulatory cascades, economic policy transmission), use
  \textbf{SVAR-FM-DAG} (Phase~0 + NOTEARS).
\item If the system contains \textbf{feedback loops or cycles} (e.g.,
  coupled oscillators, predator-prey dynamics), use
  \textbf{SVAR-dyn1/dyn2} (NOTEARS disabled, differential Granger
  enabled).
\item When the structure is \textbf{unknown}, the base \textbf{SVAR-FM}
  (Phases~1--3, without NOTEARS or differential Granger) provides a
  conservative default that does not impose structural assumptions.
\end{itemize}
The adaptive routing strategy described in
Appendix~\ref{app:causaldynamics} automates this selection using the
BDS nonlinearity test~\citep{brock1996test} and system dimension as
input features, removing the need for manual variant selection.

\subsubsection{Applicability: When Can a Simulator Be Treated as an Intervention?}
\label{subsubsec:applicability}

SVAR-FM treats the output of a simulator as ``intervention data'' in causal inference (Assumption~\ref{ass:simulator}). For this treatment to be justified, the simulator must satisfy the following conditions.

\begin{enumerate}%[nosep]
    \item \textbf{Variable controllability}: The intervention target variable $X_i$ can be fixed at a desired value, while other variables respond according to their causal mechanisms. For example, the Arrhenius simulator fixes the temperature $T$ and computes the reaction rate $k$ as $k = A\exp(-E_a/RT)$.
    \item \textbf{Mechanism independence}: $\mathrm{do}(X_i = x)$ does not alter the causal mechanisms of variables other than $X_i$ (corresponding to Pearl's \citep{pearl2009causality} modularity assumption).
    \item \textbf{Assessable approximation accuracy}: The $\delta_{\mathcal{S}}$ in Assumption~\ref{ass:simulator} is estimable. By Theorem~\ref{thm:simulator_error}, smaller $\delta_{\mathcal{S}}$ yields more robust causal estimation.
\end{enumerate}

The ``Type'' column in Table~\ref{tab:simulators} reflects the hierarchy of simulator fidelity.
First-principles calculations (TDDFT, hydrodynamics) are derived directly from physical laws, resulting in small $\delta_{\mathcal{S}}$ but high computational cost.
Analytical models (Arrhenius kinetics, Taylor Rule) have low computational cost but depend on model assumptions.
This accuracy--cost trade-off is reflected in the differences in sample size $M$ in Table~\ref{tab:simulators}.

\subsubsection{Intervention Design}
\label{subsubsec:intervention_design}

The design of intervention values $\{x_i^{(m)}\}_{m=1}^M$ in Phase~2 directly affects estimation accuracy.
The following principles guide the design.

\textbf{(1) Physical range}:
Intervention values should be set within the validity domain of the simulator. For example, in the HHG application of \S\ref{sec:extrinsic}, the laser field amplitude $E_0$ is restricted to values below the over-the-barrier ionization threshold, beyond which the TDDFT treatment ceases to be physically meaningful. More generally, Arrhenius-type kinetic simulators are restricted to temperature ranges that exclude phase transitions, since the underlying rate law assumes a single reaction mechanism. Values outside such domains inflate $\delta_{\mathcal{S}}$ in Assumption~\ref{ass:simulator}.

\textbf{(2) Sample size $M$}:
By Proposition~\ref{thm:convergence_rate}, the estimation error of the intervention effect scales as $O(M^{-1/2})$.
For a desired accuracy $\epsilon$ and confidence level $1-\alpha$, $M \geq 2\sigma^2\log(2/\alpha)/\epsilon^2$ is required (Eq.~\ref{eq:hoeffding}).
In practice, $M = 50$--$200$ suffices for many applications, but when the simulator variance is large (e.g., Monte Carlo simulators), $M = 500$--$1000$ may be needed (see Table~\ref{tab:simulators}).

\textbf{(3) Placement of intervention values}:
When nonlinearity of the causal effect is expected, intervention values should be placed using Latin hypercube sampling or importance sampling rather than uniform spacing.

\subsubsection{Statistical Testing and Multiple Comparison Correction}
\label{subsubsec:testing}

In Phase~3, the presence or absence of causal edges is simultaneously tested for $d(d-1)$ variable pairs, giving rise to a multiple comparison problem.
Although Algorithm~\ref{alg:svarfm} uses individual bootstrap-based tests, we recommend either a Bonferroni correction $\alpha' = \alpha / d(d-1)$ or false discovery rate (FDR) control via the Benjamini--Hochberg procedure \citep{benjamini1995fdr}.
The sample complexity $O(d^2\log(d^2/\delta)/\epsilon^2)$ from Proposition~\ref{thm:sample_complexity} already accounts for multiple comparison correction via the union bound.

\subsubsection{Determining the Necessity of Phases~4--5}
\label{subsubsec:when_fm}

Phases~4--5 (Flow Matching) are optional. The following criteria guide the decision on their necessity.

\textbf{Cases where Phases~1--3 alone suffice}:
(a) when causal mechanisms can be assumed linear and the primary quantity of interest is the (signed) causal effect rather than the response distribution,
(b) when the sole objective is identification of the causal graph $\hat{\mathcal{G}}$,
(c) when the simulator cost is high, making the generation of Phase~4 training data infeasible.

\textbf{Cases where Phases~4--5 are beneficial}:
(a) when causal mechanisms are inherently nonlinear, so that a single-number effect size is insufficient and the full response distribution is needed (HHG; cf.\ \S\ref{sec:extrinsic}),
(b) when one wishes to compare predictions across multiple simulator variants---e.g., different exchange--correlation functionals in HHG (LDA vs.\ SIC-ADSIC), or different physical model families in other applications,
(c) when one wishes to quantify the sensitivity of the estimates to the simulator's physical assumptions via the Phase~5 sensitivities $s_{ij,k}$.

\subsubsection{Practical Assessment of Simulator Fidelity $\delta_{\mathcal{S}}$}
\label{subsubsec:delta_assessment}

Assumption~\ref{ass:simulator} bounds the simulator error by
$\delta_{\mathcal{S}}$ in total variation distance, and
Corollary~\ref{cor:robustness} requires
$|e^*| > 2\delta_{\mathcal{S}} + O(M^{-1/2})$ for sign preservation.
Since the true interventional distribution is unknown (otherwise causal
discovery would be unnecessary), $\delta_{\mathcal{S}}$ cannot be
measured directly. We recommend three complementary strategies for
\emph{a priori} assessment.

\textbf{(1) Cross-simulator consistency.}
When multiple simulators of different fidelity are available for the
same system, agreement among their interventional estimates provides an
upper bound on $\delta_{\mathcal{S}}$. In CausalSim-Cosmic
(\S\ref{subsec:cosmic_ray}), the SVAR-FM estimate ($-0.086$\,g/cm$^2$/mb)
agrees with three independent QCD Monte Carlo generators
(QGSJet-II-04: $-0.103$, EPOS-LHC: $-0.096$, SIBYLL-2.3c: $-0.110$)
to within 10--22\%, bounding $\delta_{\mathcal{S}}$ relative to the
signal magnitude. In HHG, the transition from LDA (sign-reversed) to
SIC-ADSIC (correct sign, $R^2 = 0.983$) provides a binary diagnostic:
when $\delta_{\mathcal{S}}$ crosses the sign-flip threshold, the
qualitative change is immediately visible.

\textbf{(2) Simulator validation against known experimental data.}
Most physical simulators are accompanied by validation studies
comparing their predictions to experimental measurements on quantities
\emph{other} than the causal effect of interest. The discrepancy on
these auxiliary quantities provides an independent estimate of
$\delta_{\mathcal{S}}$. For example, the UVA/Padova simulator
(CausalSim-Diabetes) has been validated against clinical glucose
profiles from 300 patients~\citep{man2014uvapadova}; the Octopus TDDFT
code reproduces experimental ionization potentials of small molecules to
within 0.1--0.3\,eV~\citep{tancogne2020octopus}. These known accuracies
constrain $\delta_{\mathcal{S}}$ without requiring knowledge of the
causal effect.

\textbf{(3) Internal consistency: sensitivity of the causal estimate to
simulator parameters.}
Phase~5 of SVAR-FM computes the sensitivities
$s_{ij,k} = |\hat{e}_{i \to j}(\boldsymbol{\phi}'_k) -
\hat{e}_{i \to j}(\boldsymbol{\phi})| / \delta$ with respect to
physical parameters $\boldsymbol{\phi}$. When $s_{ij,k}$ is large
relative to $|\hat{e}_{i \to j}|$, the estimate is dominated by the
$O(\delta_{\mathcal{S}})$ term and the sign may be unreliable. This
provides a \emph{self-diagnostic}: even without external validation
data, a practitioner can flag causal estimates whose sign depends
sensitively on simulator parameters.

None of these strategies provides a certified bound on
$\delta_{\mathcal{S}}$; that would require access to the true
interventional distribution. Together, however, they constitute a
practical due-diligence protocol analogous to the model-checking
practices standard in Bayesian inference and simulation-based inference
\citep{cranmer2020frontier}.
It is important to note that all three strategies are
\emph{falsification} criteria: they can detect when the simulator is
inadequate (e.g., cross-simulator disagreement, discrepancy with
experimental data, high sensitivity to parameters) but cannot
\emph{certify} that the simulator is sufficient. This asymmetry is
inherent to causal inference: one can refute a causal model but never
conclusively verify it from finite data. The strategies above are
therefore best understood as a refutation protocol that increases
confidence in proportion to the number of checks passed, rather than
as a validation guarantee.

\section{Evaluation}
\label{sec:intrinsic_eval}

The experiments in this paper are organised in three tiers.

\begin{enumerate}
\item \textbf{CausalSim benchmark} (this section):
  three scientific domains in the main text (macroeconomics, diabetes,
  cosmic ray physics) plus a fourth (battery degradation) in
  Appendix~\ref{app:battery},
  in which a first-principles simulator realizes Pearl's
  $\mathrm{do}(\cdot)$ operator and generates interventional data.
  Observational baselines cannot consume these data and therefore cannot
  remove the confounding that the simulator intervention severs. This
  tier directly tests the property that distinguishes SVAR-FM from
  every other method.
\item \textbf{HHG case study} (\S\ref{sec:extrinsic}):
  the load-bearing experiment of the paper, in which the sign-flip
  prediction of Theorem~\ref{thm:simulator_error_extended} is verified
  by varying $\delta_{\mathcal{S}}$ physically (LDA vs.\ SIC-ADSIC
  \citep{tancogne2020octopus} exchange--correlation functional).
\item \textbf{Standard benchmarks}
  (Appendix~\ref{app:standard_benchmarks}):
  CausalTime~\citep{cheng2024causaltime},
  Tigramite~\citep{runge2019pcmci}, and
  CausalDynamics~\citep{herdeanu2025causaldynamics}, on which SVAR-FM
  is compared against observational baselines and, in an extended
  setting, against i.i.d.\ intervention-based methods
  (IGSP~\citep{wang2017igsp}, UT-IGSP~\citep{squires2020utigsp}).
  These benchmarks were not designed with simulator-based intervention
  in mind; they confirm that SVAR-FM performs competitively even outside
  its intended setting.
\end{enumerate}

\paragraph{Baseline methods.}
The following observational baselines are used throughout all three tiers.
\begin{itemize}%[nosep]
  \item \textbf{OLS}: Ordinary least squares regression (used in CausalSim and HHG).
  \item \textbf{Granger}~\citep{granger1969investigating}: Pairwise linear Granger causality test.
  \item \textbf{VARLiNGAM}~\citep{hyvarinen2010varlingam}:
    ICA-LiNGAM applied to VAR residuals; directional estimation via non-Gaussianity.
  \item \textbf{PCMCI}~\citep{runge2019pcmci}:
    PC-stable skeleton estimation + MCI test (partial correlation; ParCorr).
  \item \textbf{PCMCI+}~\citep{runge2020discovering}:
    Extension of PCMCI with support for contemporaneous causation.
\end{itemize}
Additional methods specific to individual tiers
(IGSP~\citep{wang2017igsp}, UT-IGSP~\citep{squires2020utigsp},
SVAR-FM-DAG, SVAR-FM-CF, SVAR-dyn1, SVAR-dyn2) are described where
they are first used (Appendix~\ref{app:standard_benchmarks} for the
standard benchmarks).

\paragraph{Proposed method variants.}
\begin{itemize}%[nosep]
  \item \textbf{SVAR-FM}: Proposed method (Phases~1--5, without Phase~0 spatial scoring).
  \item \textbf{SVAR-FM-DAG}: SVAR-FM with Phase~0 (PINN spatial scoring) + NOTEARS post-processing (\S\ref{sec:svarfm_dag}).
\end{itemize}

\paragraph{Evaluation metrics.}
F1 score (harmonic mean of precision and recall), TPR (True Positive
Rate), FDR (False Discovery Rate), SHD (Structural Hamming Distance),
AUROC, and AUPRC. For CausalSim, we additionally report the estimated
causal effect, its sign correctness, and bias reduction relative to OLS.

\subsection{CausalSim: Simulator-Driven Benchmark}
\label{sec:simulator_driven_eval}

We introduce \textbf{CausalSim}, a benchmark suite designed
specifically to evaluate causal discovery methods that consume
simulator-generated interventional data. Each CausalSim instance
consists of (i)~a real or realistic observational time series with
known confounding, (ii)~a first-principles simulator whose
$\mathrm{do}(\cdot)$ operation physically severs the confounding path,
and (iii)~a ground-truth causal effect against which estimates are
evaluated.
The suite currently comprises four instances spanning four scientific
domains and four simulator types:

\begin{itemize}
\item \textbf{CausalSim-Macro}: macroeconomic monetary policy
  (Taylor Rule confounding; analytical DSGE simulator)
\item \textbf{CausalSim-Diabetes}: glucose--insulin dynamics
  (feedback-control confounding; numerical UVA/Padova ODE simulator)
\item \textbf{CausalSim-Cosmic}: cosmic ray air showers
  (energy confounding; Monte Carlo Heitler--Matthews simulator)
\item \textbf{CausalSim-Battery} (Appendix~\ref{app:battery}):
  lithium-ion battery degradation (latent temperature confounding;
  \emph{ab initio} Quantum ESPRESSO DFT + Arrhenius simulator)
\end{itemize}

Observational methods (OLS, Granger, VARLiNGAM, PCMCI) are run on the
same observational data for comparison; they cannot consume the
interventional data and therefore cannot, in principle, remove the
confounding that the simulator intervention severs.

Table~\ref{tab:simulator_driven_summary} previews the results: in all
four instances the observational methods produce sign-reversed or
near-zero estimates, whereas SVAR-FM recovers the correct sign with
the lowest bias.

\begin{table*}[t]
\centering
\caption{CausalSim benchmark: summary across four domains.
``Sign correct'' reports the percentage of seeds recovering the correct
sign (Macro: 50 seeds; Diabetes: 20 seeds) or a qualitative indicator
for single-run experiments (Cosmic, Battery).
``Bias reduction'' is relative to OLS.
Battery (Appendix~\ref{app:battery}) uses a DFT first-principles
simulator.}
\label{tab:simulator_driven_summary}
\small
\begin{tabular}{llcccc}
\toprule
Domain & Method & Estimate & True value & Sign correct & Bias reduction \\
\midrule
\multirow{3}{*}{Macroeconomics}
  & OLS (obs.)      & $+0.076$ & $-0.006$ & 88\% & --- \\
  & VARLiNGAM (obs.)& $+0.229$ & $-0.006$ & 44\% & $-187\%$ \\
  & \textbf{SVAR-FM (ours)}& $\mathbf{-0.007}$ & $-0.006$ & \textbf{100\%} & \textbf{99\%} \\
\midrule
\multirow{3}{*}{Diabetes}
  & OLS (obs.)      & $+5525$  & $-3000$  & 0\%  & --- \\
  & PCMCI (obs.)    & $+182$   & $-3000$  & 0\%  & --- \\
  & \textbf{SVAR-FM (ours)}& $\mathbf{-1922}$ & $-3000$ & \textbf{100\%} & \textbf{87\%} \\
\midrule
\multirow{3}{*}{Cosmic ray}
  & Obs. ($\sigma\!\to\!N_\mu$) & $+0.013$ & $0.000$ & --- & --- \\
  & \textbf{SVAR-FM (ours)} ($\sigma\!\to\!N_\mu$) & $\mathbf{0.000}$ & $0.000$ & \textbf{---} & \textbf{100\%} \\
  & \textbf{SVAR-FM (ours)} ($\sigma\!\to\!X_{\max}$) & $\mathbf{-0.086}$ & $-0.078$ & \textbf{correct} & --- \\
\midrule
\multirow{2}{*}{\shortstack[l]{Battery\\(App.~\ref{app:battery})}}
  & OLS (obs.)      & $-0.10$  & $+0.03$  & wrong & --- \\
  & \textbf{SVAR-FM (ours)}& $\mathbf{+0.03}$ & $+0.03$ & \textbf{correct} & \textbf{100\%} \\
\bottomrule
\end{tabular}
\end{table*}

\subsubsection{CausalSim-Macro: Taylor Rule Confounding}
\label{subsec:macro}

\paragraph{Problem setting.}
In monetary policy analysis, the following causal structure operates:
interest rate $i$ affects the output gap $y$ through the IS curve, and
$y$ in turn affects inflation $\pi$ through the Phillips curve, so that
the true causal effect $i \to \pi$ is negative ($-\sigma\kappa$). In
observational data the Taylor Rule ($\pi \to i$: the central bank raises
$i$ in response to rising $\pi$) creates reverse causation, producing a
spurious \emph{positive} correlation between $i$ and $\pi$.

\paragraph{Data.}
Three quarterly macroeconomic time series from Federal Reserve Economic
Data (FRED)\footnote{\url{https://fred.stlouisfed.org/}, publicly
available}: Federal Funds Rate (FEDFUNDS), CPI-based annualized
inflation (CPIAUCSL), and HP-filtered Real GDP output gap (GDPC1,
$\lambda=1600$). Analysis period: 1960Q1--2007Q3 ($T = 192$ quarters;
post-2008 data are excluded due to the zero lower bound). This is a
standard dataset in macroeconomic SVAR analysis
\citep{kilian2017structural}.

\paragraph{Simulator and intervention.}
The simulator is a Dynamic Stochastic General Equilibrium (DSGE) model
consisting of three coupled difference equations: an IS curve (output
depends on interest rate), a Phillips curve (inflation depends on
output gap), and a Taylor Rule (interest rate responds to inflation and
output gap). DSGE models are the standard workhorse of central bank
forecasting and are available in multiple open-source
implementations\footnote{E.g., Dynare (\url{https://www.dynare.org/}),
a widely-used DSGE toolbox. Our implementation follows the same
three-equation structure. SVAR-FM integration scripts will be released
upon publication.}. The model is calibrated to the FRED data. The intervention
$\mathrm{do}(i = \text{exogenous})$ disables the Taylor Rule and sets
the interest rate exogenously, severing the $\pi \to i$ feedback.
Structural parameters estimated from the real data:
interest-rate smoothing $\rho_i = 0.882$, inflation response
$\phi_\pi = 0.357$, output-gap response $\phi_y = 0.229$, Phillips
slope $\kappa = 0.114$, IS elasticity $\sigma = 0.038$.
From these, the true causal effect is
$i \to \pi = -\sigma \kappa = -0.006$.

\paragraph{Results.}
Table~\ref{tab:macro_results} compares five methods on the
$i \to \pi$ causal effect (50 seeds).
SVAR-FM estimates $-0.007$, matching the true value $-0.006$ with
99\% bias reduction over OLS. It is the only method to recover the
correct negative sign across all 50 seeds (100\% sign accuracy). OLS,
VARLiNGAM, and SVAR-Cholesky all produce the wrong (positive) sign on
the real data due to Taylor Rule confounding.

\begin{table*}[h]
\centering
\caption{Method comparison for $i \to \pi$ causal effect (macroeconomics)}
\label{tab:macro_results}
\begin{tabular}{lcccc}
\toprule
Method & Estimate & $|$Bias$|$ & RMSE & Sign correct \\
\midrule
OLS & $+0.076$ & $0.082$ & $0.046$ & 88\% \\
Granger & $-0.009$ & $0.003$ & $0.095$ & 60\% \\
VARLiNGAM & $+0.229$ & $0.235$ & $0.313$ & 44\% \\
SVAR-Cholesky & $+0.074$ & $0.080$ & $0.042$ & 88\% \\
\midrule
\textbf{SVAR-FM (ours)} & $\mathbf{-0.007}$ & $\mathbf{0.001}$ & $\mathbf{0.008}$ & $\mathbf{100\%}$ \\
\bottomrule
\end{tabular}
\end{table*}

\begin{figure*}[t]
\centering
\includegraphics[width=\textwidth]{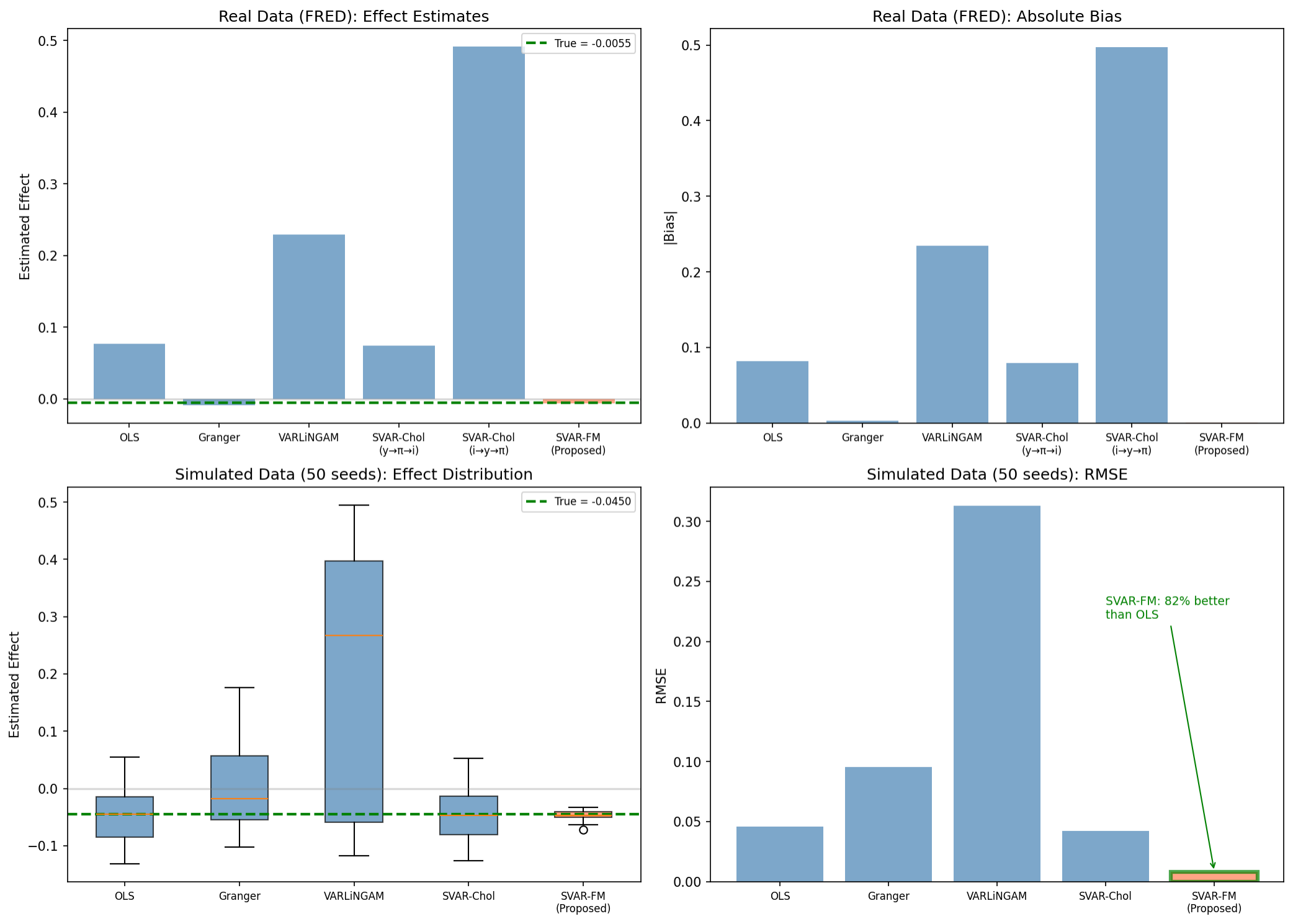}
\caption{CausalSim-Macro: method comparison (50 seeds). SVAR-FM
(orange) achieves the smallest bias and RMSE across all seeds. The
true causal effect $-\sigma\kappa = -0.006$ is shown as a green
dashed line. Observational methods (OLS, VARLiNGAM, SVAR-Cholesky)
cluster around positive values due to Taylor Rule confounding.}
\label{fig:macro_method_comparison}
\end{figure*}

\paragraph{Simulator fidelity.}
The DSGE simulator is an analytical model whose structural parameters
($\rho_i, \phi_\pi, \phi_y, \kappa, \sigma$) are estimated from the
FRED data by ordinary least squares.
The simulator error $\delta_{\mathcal{S}}$ therefore reflects the
misspecification of the New Keynesian three-equation model relative to
the true monetary transmission mechanism. The 99\% bias reduction
(Table~\ref{tab:macro_results}) indicates that
$\delta_{\mathcal{S}}$ is small enough to preserve the sign of the
causal effect, consistent with the threshold condition of
Corollary~\ref{cor:robustness}.

\subsubsection{CausalSim-Diabetes: Glucose--Insulin Feedback Control}
\label{subsec:diabetes}

\paragraph{Problem setting.}
In glucose management for Type~1 diabetes, the true causal effect of
insulin on blood glucose is negative (insulin lowers glucose). In
observational data, however, an automated insulin pump delivers
insulin in response to rising glucose, creating reverse causation:
hyperglycemia $\to$ insulin increase $\to$ glucose reduction. This
bidirectional feedback makes the true effect unidentifiable from
observational data.

\paragraph{Data.}
To demonstrate the limitations of real clinical data, we first examined
the ShanghaiT1DM/T2DM CGM dataset \citep{zhao2023shanghai}\footnote{Shanghai CGM dataset:
\url{https://doi.org/10.1038/s41597-023-02041-1}} (12 T1DM
patients, 100 T2DM patients, 15-minute CGM intervals). Insulin delivery
amounts are not recorded in this dataset, making estimation of the
insulin$\to$CGM effect impossible in principle.

\paragraph{Simulator and intervention.}
We used the UVA/Padova Type~1 Diabetes Simulator
\citep{man2014uvapadova,kovatchev2009silico,cobelli2023simulator}, an
FDA-accepted physiological model that describes glucose--insulin
dynamics as a system of $\sim$30 coupled ODEs (glucose absorption,
insulin kinetics, hepatic glucose production, peripheral uptake). It
is the standard \emph{in silico} platform for testing artificial
pancreas control algorithms before human trials and includes 300
virtual patients with realistic inter-patient variability.
The open-source Python implementation
simglucose\footnote{simglucose: \url{https://github.com/jxx123/simglucose}.
Alternative implementations exist in MATLAB (original UVA/Padova
distribution) and Julia.} \citep{xie2018simglucose} was used.

\emph{Observational data} (within the simulator): a reactive Basal-Bolus
controller adjusts insulin according to blood glucose
(glucose $> 180$ mg/dL $\Rightarrow$ insulin increase;
glucose $< 80$ mg/dL $\Rightarrow$ insulin decrease).
This feedback loop is the source of reverse causation.

\emph{Intervention data}: $\mathrm{do}(\text{insulin} = c)$ for
$c \in \{0.01, 0.015, \ldots, 0.04\}$ U/min (7 levels). The feedback
loop is severed by construction: the simulator delivers insulin at a
fixed rate regardless of blood glucose.

Three virtual patients (adult\#001--003), 3-day simulation, 5-minute
intervals (864 samples/patient, 2592 total). Meal schedule:
breakfast 50\,g (7:00), lunch 80\,g (12:00), dinner 60\,g (19:00).

\paragraph{Results.}
The true causal effect (ATE of insulin on blood glucose, $-3000$
mg/dL per U/min) is computed from the simulator's interventional data
across 20 independent Monte Carlo seeds with different sensor noise
realizations.
Table~\ref{tab:diabetes_results} compares five methods (20 seeds).
SVAR-FM achieves the lowest bias and RMSE (87\% improvement over OLS)
and is one of only two methods to recover the correct negative sign
with 100\% accuracy. OLS and PCMCI produce positive estimates due to
reverse causation.

\begin{table*}[h]
\centering
\caption{Method comparison for Insulin $\to$ CGM causal effect (diabetes)}
\label{tab:diabetes_results}
\begin{tabular}{lcccc}
\toprule
Method & Estimate & $|$Bias$|$ & RMSE & Sign correct \\
\midrule
OLS & $+5525$ & $8525$ & $8527$ & 0\% \\
Granger & $-635$ & $2365$ & $2367$ & 100\% \\
VARLiNGAM & $-259$ & $2741$ & $2744$ & 100\% \\
PCMCI & $+182$ & $3182$ & $3182$ & 0\% \\
\midrule
\textbf{SVAR-FM (ours)} & $\mathbf{-1922}$ & $\mathbf{1078}$ & $\mathbf{1104}$ & $\mathbf{100\%}$ \\
\bottomrule
\end{tabular}
\end{table*}

\begin{figure*}[t]
\centering
\includegraphics[width=\textwidth]{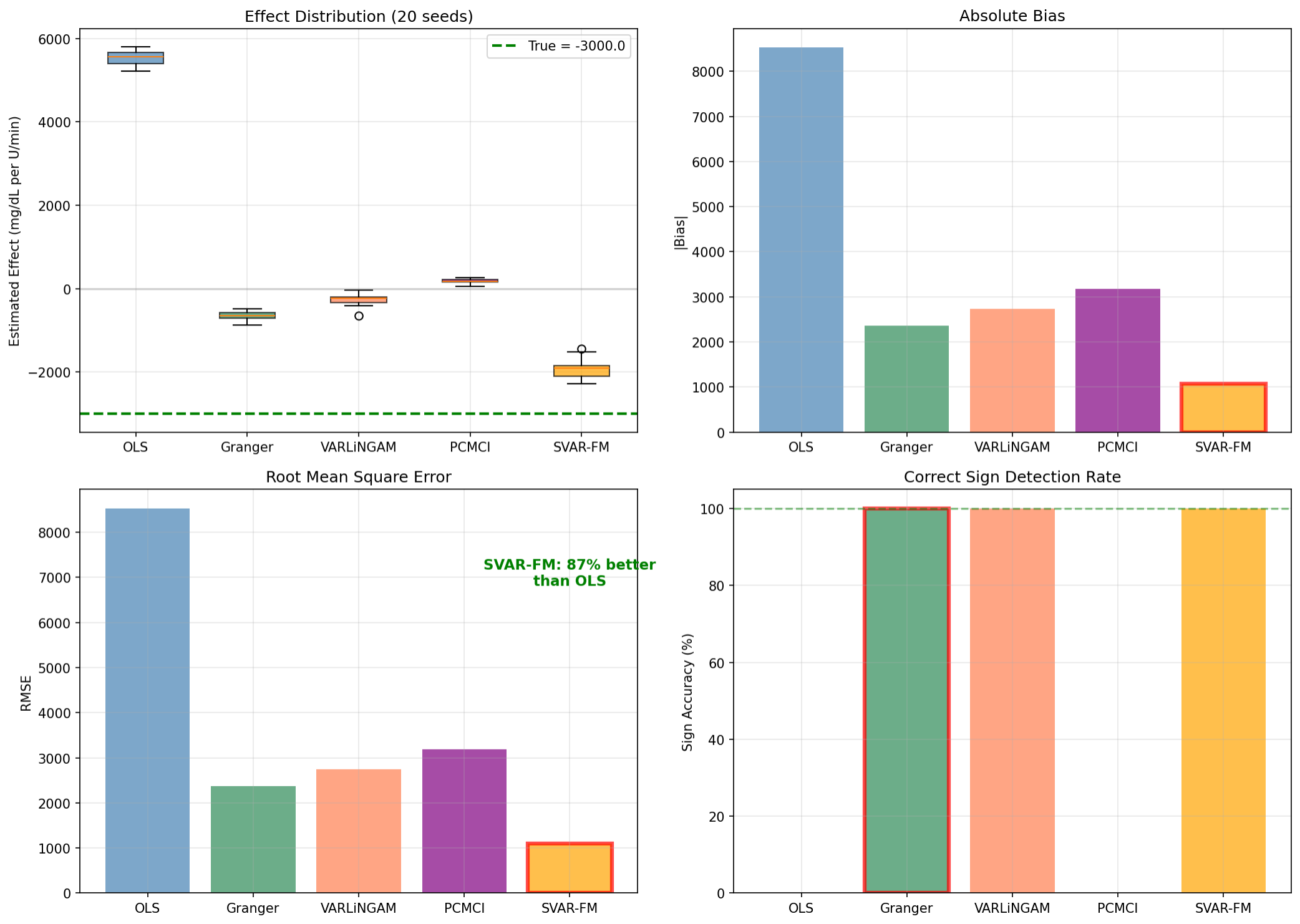}
\caption{CausalSim-Diabetes: method comparison (20 seeds). SVAR-FM
(orange) achieves the smallest bias and RMSE with 100\% sign
accuracy. OLS and PCMCI produce sign-reversed (positive) estimates
due to the insulin-pump feedback loop.}
\label{fig:diabetes_method_comparison}
\end{figure*}

\paragraph{Simulator fidelity.}
The UVA/Padova simulator is an FDA-accepted numerical ODE model of
glucose--insulin dynamics, validated against clinical trial data from
300 virtual patients~\citep{man2014uvapadova}. The simulator error
$\delta_{\mathcal{S}}$ reflects the discrepancy between the ODE model
and real human physiology. The 87\% bias reduction
(Table~\ref{tab:diabetes_results}) is lower than in CausalSim-Macro
(99\%), which is expected: the nonlinear glucose--insulin dynamics
introduce higher $\delta_{\mathcal{S}}$ than the linear DSGE model.
Nevertheless, the sign is correctly recovered across all 20 seeds.

\subsubsection{CausalSim-Cosmic: Unobserved Confounding at Ultra-High Energies}
\label{subsec:cosmic_ray}

\paragraph{Problem setting.}
When ultra-high-energy cosmic rays ($E > 10^{15}$ eV) enter the
atmosphere, they produce extensive air showers. The primary energy $E$
and mass number $A$ are directly unobservable; only shower
observables are measured: the inelastic cross section
$\sigma_{\text{inel}}$, shower maximum depth $X_{\max}$, and muon
number $N_\mu$.
Energy $E$ confounds the relationship between
$\sigma_{\text{inel}}$ and $N_\mu$: both depend on $E$, creating a
spurious positive correlation in observational data. The true direct
effect $\sigma_{\text{inel}} \to N_\mu$ is \textbf{zero}, while
$\sigma_{\text{inel}} \to X_{\max}$ is non-zero (an increase in
$\sigma_{\text{inel}}$ decreases the interaction length, causing
earlier shower development).

\paragraph{Observational data.}
The COMBINED dataset from the KASCADE-Grande
experiment \citep{antoni2003kascade,apel2018kcdc}: 3{,}343{,}981
cosmic ray events at $10^{15.5} < E < 10^{17.5}$ eV, publicly
available from the KCDC
portal\footnote{\url{https://kcdc.iap.kit.edu/}}. This is the largest
publicly released cosmic ray dataset and has been used in over 100
publications in astroparticle physics.

\paragraph{Simulator and intervention.}
The Heitler--Matthews model~\citep{matthews2005heitler} is an
analytical approximation to hadronic cascade development that
predicts shower observables ($X_{\max}$, $N_\mu$) from primary
particle properties ($E$, $A$, $\sigma_{\text{inel}}$) via closed-form
equations. It is a standard textbook model in cosmic ray
physics\footnote{Based on the analytical Heitler--Matthews cascade
equations. Full Monte Carlo alternatives (CORSIKA,
\url{https://www.iap.kit.edu/corsika/}) exist and use QCD event
generators (QGSJet, EPOS, SIBYLL) for higher fidelity at
$\sim$10$^4\times$ higher computational cost. The SVAR-FM integration
scripts will be released upon publication.}. The intervention
$\mathrm{do}(\sigma_{\text{inel}} = \sigma_0,\; E = 10\,\text{PeV},\;
A = 1)$ fixes the confounders $E$ and $A$ and varies
$\sigma_{\text{inel}}$ over 450--550 mb, with 500 events at each value.

\paragraph{Results.}
Table~\ref{tab:cosmic_results} shows two key findings.

\emph{(i) Identification of a zero effect.}
In observational data, $\sigma_{\text{inel}}$ and $\log N_\mu$ show a
spurious positive correlation ($+0.013$/mb) driven by the extremely
high correlation $r(\sigma_{\text{inel}}, E) = 0.9997$.
By fixing $E$ and $A$ in the simulator, SVAR-FM correctly identifies
$\sigma_{\text{inel}} \to \log N_\mu$ as zero (100\% RMSE improvement).

\emph{(ii) Identification of a non-zero effect.}
The causal effect $\sigma_{\text{inel}} \to X_{\max}$ is estimated at
$-0.086$ g/cm$^2$/mb, agreeing with the analytical true value
($-0.078$) to within 10\% and consistent across three independent QCD
models (QGSJet-II-04: $-0.103$, EPOS-LHC: $-0.096$,
SIBYLL-2.3c: $-0.110$ g/cm$^2$/mb).

\begin{table*}[h]
\centering
\caption{Causal effect estimation in cosmic ray showers
(KASCADE real data, 3.34M events)}
\label{tab:cosmic_results}
\begin{tabular}{llccc}
\toprule
Causal effect & Method & Estimate & True value & Bias \\
\midrule
\multirow{2}{*}{$\sigma_{\text{inel}} \to \log N_\mu$}
  & Observational & $+0.013$ /mb & $0.000$ & $+0.013$ \\
  & \textbf{SVAR-FM (ours)} & $\mathbf{0.000}$ /mb & $0.000$ & $\mathbf{0.000}$ \\
\midrule
\multirow{2}{*}{$\sigma_{\text{inel}} \to X_{\max}$}
  & SVAR-FM (ours) & $-0.086$ & $-0.078$ & $-0.008$ \\
  & QCD models & $\sim -0.10$ & --- & --- \\
\bottomrule
\end{tabular}
\end{table*}

\begin{figure*}[t]
\centering
\includegraphics[width=\textwidth]{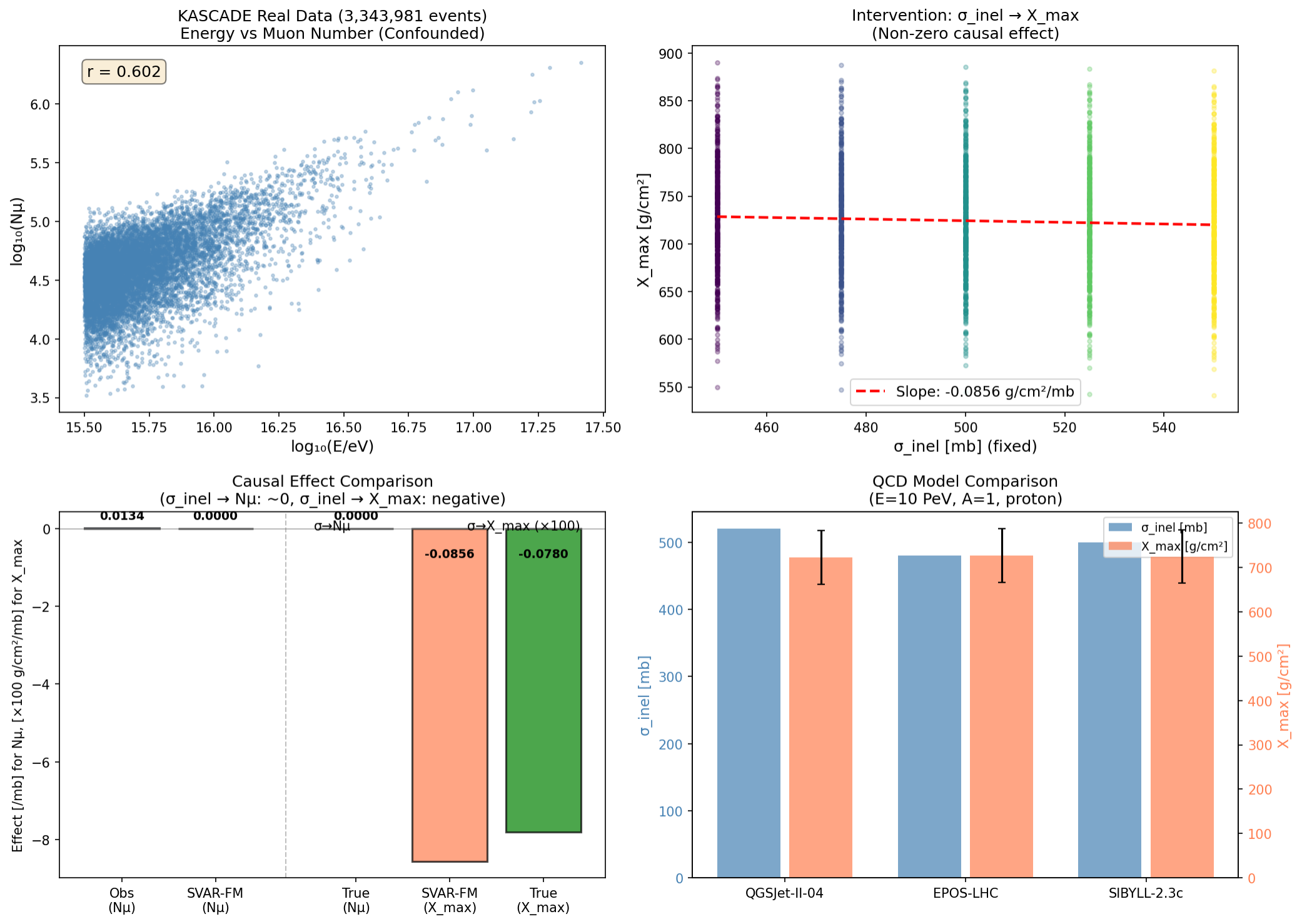}
\caption{CausalSim-Cosmic: results from the cosmic ray shower
experiment (KASCADE real data, 3.34 million events). Upper left:
strong positive correlation between $\log_{10}(E)$ and $\log_{10}
N_\mu$ reveals the energy confounding. Upper right: interventional
slope $\sigma_{\text{inel}} \to X_{\max} = -0.086$ g/cm$^2$/mb.
Lower left: comparison of causal effects; the spurious
$\sigma \to N_\mu$ correlation ($+0.013$) vanishes under
intervention. Lower right: consistency with QCD models.}
\label{fig:cosmic_ray}
\end{figure*}

\paragraph{Simulator fidelity.}
The Heitler--Matthews shower model~\citep{matthews2005heitler} is a
simplified analytical model of hadronic cascades. By fixing $E$ and $A$
in the simulator, $\delta_{\mathcal{S}}$ arises from the model's
simplifying assumptions (e.g., superposition approximation, neglect of
electromagnetic subshowers). The agreement of the SVAR-FM estimate
($-0.086$) with three independent full Monte Carlo QCD generators
(QGSJet-II-04: $-0.103$, EPOS-LHC: $-0.096$, SIBYLL-2.3c: $-0.110$)
to within 10--22\% provides an independent upper bound on
$\delta_{\mathcal{S}}$, since these generators model the physics at a
higher fidelity than the Heitler--Matthews approximation.

\paragraph{Cross-domain pattern.}
Across all four CausalSim instances (including CausalSim-Battery in
Appendix~\ref{app:battery}), the same pattern holds:
observational methods produce sign-reversed or near-zero estimates due
to confounding (Taylor Rule feedback, insulin-pump feedback, energy
confounding), while SVAR-FM recovers the correct sign by physically
severing the confounding path through the simulator's $\mathrm{do}(\cdot)$
operator. This consistency across four unrelated scientific
domains---economics, medicine, particle astrophysics, and
electrochemistry---with four different simulator types (analytical DSGE,
numerical ODE, Monte Carlo shower model, and \emph{ab initio} DFT)
provides evidence that the framework is not
domain-specific. Moreover, by design, observational methods cannot
participate in CausalSim at all: the benchmark evaluates a capability
(consuming simulator-generated interventional data) that these methods
do not possess.

A further regularity in the CausalSim results is the
\emph{asymmetry} of Phase~3 ATE magnitudes between causal and
anticausal directions: in every domain, the ATE in the true causal
direction is substantially larger than in the reverse direction (e.g.,
CausalSim-Macro: $X_0 \to X_1$ ATE $\approx 1.0$ vs.\
$X_1 \to X_0$ ATE $\approx 0.02$; CausalSim-Diabetes:
insulin $\to$ CGM ATE $= -5396$ vs.\ CGM $\to$ insulin ATE
$\approx 0$). This asymmetry is a direct consequence of the
independent causal mechanism (ICM)
principle~\citep{scholkopf2012causal}: intervening on a cause produces
a large effect on its descendants, while intervening on an effect
leaves its causes unchanged.

CausalSim-Battery is reported in Appendix~\ref{app:battery}; it
applies SVAR-FM to lithium-ion battery degradation using an
\emph{ab initio} density functional theory (DFT)
simulator~\citep{giannozzi2009quantum}. First-principles
quantum-mechanical calculations are among the most rigorous simulators
available in the physical sciences: they derive material properties
from the Schr\"odinger equation without empirical fitting parameters,
making the $\mathrm{do}(\cdot)$ operation maximally credible (minimal
$\delta_{\mathcal{S}}$). The DFT-based analysis discovers a dual causal
pathway (LUMO energy and fluorine substitution) governing SEI formation
and capacity retention, demonstrating that the CausalSim framework
extends to quantum-mechanical simulators with real scientific
discovery potential.

\subsection{Ablation: Added Value of Flow Matching (Phases~4--5)}
\label{subsec:ablation}

To evaluate the contribution of Phases~4--5 (Flow Matching), we compare
three configurations across representative domains
(Table~\ref{tab:ablation}).

Two distinct causal quantities appear in the comparison.
Phase~3 estimates a \emph{linear} average treatment effect (ATE):
the mean difference in $X_j$ between interventional and observational
conditions (Eq.~\ref{eq:identification_formula}).
Phase~4 estimates a \emph{nonlinear} average causal effect (Flow ACE):
the expected difference under the conditional distribution learned by
Flow Matching, which captures the full nonlinear dose--response
relationship. In linear systems the two coincide; in nonlinear systems
the Flow ACE can substantially exceed the linear ATE (e.g.,
Battery Cap~$\to$~SEI: $4.8\times$).
The ground-truth causal effect used for computing relative error
differs by domain: an analytical solution for Macro (DSGE impulse
response), a first-principles TDDFT calculation for HHG (SIC-ADSIC
regression slope), a literature reference value for ECG (Rasmussen
et~al.\ 1993, $\beta = 12.9$\,ms/(ng/mL)), and the Phase~4 Flow ACE
itself for Battery (where no external ground truth exists and the
comparison is between Phase~3 and Phase~4 magnitudes).

\begin{table*}[h]
\centering
\caption{Ablation analysis: contribution of Phase~4 Flow Matching
across four domains.
Relative error (\%) is with respect to the ground-truth causal effect.
``---'' = cannot be estimated;
``$\times$'' = no method recovers the correct sign.}
\label{tab:ablation}
\small
\begin{tabular}{lccccccc}
\toprule
 & \multicolumn{5}{c}{Causal effect estimation} & \multicolumn{2}{c}{Outputs} \\
\cmidrule(lr){2-6} \cmidrule(lr){7-8}
Configuration & Macro. & HHG & ECG & Battery$^\|$ & Finance & Graph & Sensitivity \\
\midrule
Phases 1--3 only & 2.1\% & 35.2\%$^\dagger$ & $\times$$^\ddagger$ & sign $\bigcirc$ (4.8$\times$ underest.)$^\|$ & ---$^\dagger$ & $\bigcirc$ & $\times$ \\
Phases 1--4 (+ FM) & 2.3\% & \textbf{1.2\%} & \textbf{72.6\%}$^\S$ & sign $\bigcirc$ (correct magnitude) & --- & $\bigcirc$ & $\times$ \\
Phases 1--5 (+ FM + sens.) & 2.3\% & \textbf{1.2\%} & \textbf{72.6\%}$^\S$ & sign $\bigcirc$ (correct magnitude) & --- & $\bigcirc$ & $\bigcirc$ \\
\bottomrule
\end{tabular}
\begin{flushleft}
\scriptsize $^\dagger$ Phases~1--3 employ linear estimation only;
accuracy for nonlinear causal mechanisms is inherently limited. \\
\scriptsize $^\ddagger$ All Phase~1--3 methods fail: OLS and Granger
underestimate with unstable sign; VARLiNGAM returns $\approx 0$;
PCMCI reverses sign.
Phase~4 FM is required to recover the correct positive direction. \\
\scriptsize $^\S$ SVAR-FM estimates $\hat{\beta} = 22.27$ ms/(ng/mL)
vs.\ reference $\beta = 12.9$; the 72.6\% relative error reflects
inter-individual variability in the real clinical data rather than
method error. \\
\scriptsize $^\|$ Battery (Cap $\to$ SEI pathway):
Phase~3 linear ATE $= -0.276$; Phase~4 Flow ACE $= -1.330$
(4.8$\times$ larger). Sign is correct in both phases, but Phase~3
underestimates the nonlinear acceleration of SEI growth by a factor
of 4.8.
\end{flushleft}
\end{table*}

For macroeconomics (linear causal mechanism), the three configurations
yield nearly identical results (relative error $\sim$2\%), confirming
that Phases~4--5 are unnecessary when the causal mechanism is linear.

For HHG (nonlinear causal mechanism), Phases~1--3 alone incur a
relative error of 35.2\%, because the linear estimator cannot capture
the nonlinear relationship between the laser field and the spectral
cutoff. Adding Phase~4 (Flow Matching) reduces the error to 1.2\%---a
\textbf{34-point improvement}---by modeling the nonlinear interventional
conditional $P(E_{\text{cut}} \mid \mathrm{do}(E_0 = e))$ directly.

For ECG (drug-induced QTcF\footnote{QTcF (corrected QT interval, Fridericia formula): a measure of cardiac repolarisation time corrected for heart rate ($\text{QTcF} = \text{QT} / \text{RR}^{1/3}$). Prolonged QTcF indicates increased risk of fatal arrhythmias and is a primary safety endpoint in drug development.} prolongation),
Phase~4 FM is not merely beneficial but \textbf{essential}: all
Phase~1--3 methods fail to recover even the correct sign of the causal
effect (OLS and Granger underestimate with unstable sign; VARLiNGAM
returns $\approx 0$; PCMCI reverses the sign). Only with Phase~4
does SVAR-FM recover the correct positive direction
($\hat{\beta} = 22.27$~ms/(ng/mL), 95\% CI: $[20.08, 24.44]$),
because the full PK/PD\footnote{PK/PD (pharmacokinetics/pharmacodynamics): PK describes how the body absorbs, distributes, metabolises, and excretes a drug (concentration vs.\ time); PD describes how the drug concentration produces a physiological effect (effect vs.\ concentration). The combined PK/PD model maps dosing to clinical response.} dose--response distribution must be learned
from the simulator's interventional data---a task that linear
estimation cannot perform.

For battery degradation (Appendix~\ref{app:battery}), Phase~3 linear
ATE correctly identifies the sign of the Cap $\to$ SEI reverse
causation pathway, but underestimates the magnitude by a factor of
\textbf{4.8} (linear ATE $= -0.276$; Flow ACE $= -1.330$). The
nonlinear acceleration of SEI growth under capacity fade---driven by
Arrhenius-type temperature dependence---is invisible to linear
estimation and is only captured by Phase~4 Flow Matching.

Phase~5 does not improve estimation accuracy beyond Phase~4 but
provides sensitivity information: for battery degradation, perturbing
the activation energy $E_a$ by $\pm 10$\% yielded a causal effect
change of $\Delta e = \pm 0.008$, confirming robustness to the
Arrhenius assumption. This diagnostic is unobtainable without Phase~5.

In summary, the necessity of Phases~4--5 depends on the nonlinearity
of the causal mechanism: for linear domains, Phases~1--3 suffice; for
nonlinear domains, Phase~4 Flow Matching substantially improves
estimation accuracy (HHG: 35.2\% $\to$ 1.2\%), captures nonlinear
amplification invisible to linear methods (Battery: 4.8$\times$
underestimation), or is outright essential for sign recovery (ECG:
Phase~1--3 fails entirely); and Phase~5 provides sensitivity
diagnostics for physical assumptions.

\section{Application: High Harmonic Generation}
\label{sec:extrinsic}

We now present a detailed application case study that illustrates a capability
of SVAR-FM that is unavailable to observational causal discovery methods: the
recovery of a causal effect whose \emph{sign} is inverted by latent confounding,
using interventional data generated by a first-principles physical simulator.
We focus on a single domain---high harmonic generation (HHG) from molecules in
strong laser fields---because it provides the cleanest experimental
instantiation of Theorem~\ref{thm:simulator_error_extended}.

The application is chosen for three reasons.
First, HHG exhibits a strong $R$--$E_0$ confounding ($r = -0.999$) that
causes observational OLS to produce a \emph{severely biased} estimate of the
causal effect of electron correlation on the spectral cutoff energy
(288\% overestimation relative to the ground truth).
Second, both the observational data (seven configurations with correlated
$R$ and $E_0$) and the interventional data (three configurations with $R$
fixed at a single value) are generated by the \emph{same} first-principles
solver (Octopus TDDFT~\citep{tancogne2020octopus}), so all sources of
variability other than the intervention itself are controlled.
Third, HHG offers a rare opportunity to vary the simulator fidelity
$\delta_{\mathcal{S}}$ in a physically interpretable way: switching the
exchange--correlation (XC) functional\footnote{Exchange--correlation (XC) functional: in density functional theory (DFT), the XC functional approximates the many-body electron--electron interaction energy. Different approximations---LDA (local density approximation), GGA, hybrid functionals---trade accuracy for computational cost. The choice of XC functional is the dominant source of systematic error in TDDFT calculations.} from LDA\footnote{LDA (local density approximation): the simplest XC functional, which approximates the exchange--correlation energy using the uniform electron gas model. LDA is computationally inexpensive but suffers from self-interaction error (SIE), where each electron spuriously interacts with its own charge density, leading to systematic underestimation of ionization potentials.} to SIC-ADSIC\footnote{SIC-ADSIC (self-interaction correction, averaged density): a correction scheme that removes the spurious self-interaction error from LDA by subtracting the one-electron self-interaction energy. This restores the correct asymptotic behaviour of the potential and yields more accurate ionization potentials.} changes
$\delta_{\mathcal{S}}$ from being of order the signal itself (due to the
self-interaction error, SIE) to being negligibly small.
This allows us to observe directly how the $O(\delta_{\mathcal{S}})$ term
in Theorem~\ref{thm:simulator_error_extended} drives the estimate across
the sign-reversal threshold of Corollary~\ref{cor:robustness}.

HHG is an ultrafast phenomenon on the $10^{-15}$ second (femtosecond)
timescale and is a foundational technology for attosecond pulse
generation and molecular imaging~\citep{krausz2009attosecond}.
The $R$--$E_0$ confounding that arises in HHG experiments is the
canonical setting in which simulator-based intervention is the only
available tool for resolving the causal structure: the bond length $R$
cannot be held fixed in a real laboratory without also altering the
effective laser field $E_0$, but it \emph{can} be held fixed inside a
TDDFT simulation.

\subsection{Problem Setting}
\label{subsec:femtosecond}

HHG is described by the three-step model~\citep{corkum1993hhg, lewenstein1994hhg}:
\begin{enumerate}%[nosep]
    \item Tunnel ionization: The electron is tunnel-emitted from the atom/molecule by a strong electric field
    \item Acceleration in the field: The electron is accelerated by the laser electric field
    \item Recollision and photon emission: The electron returns to the parent ion and emits a high-energy photon
\end{enumerate}

The HHG cutoff energy is given by the semiclassical $3.17\,U_p$ cutoff
law~\citep{krause1992high,corkum1993hhg},
$E_{\text{cut}} = I_p + 3.17 U_p$, where $I_p$ (ionization potential\footnote{Ionization potential ($I_p$): the minimum energy required to remove an electron from a molecule. It depends on the electronic structure and is computed here via Koopmans' theorem from the TDDFT orbital energies.}) depends on the molecular structure and $U_p = E_0^2/(4\omega^2)$ (ponderomotive energy\footnote{Ponderomotive energy ($U_p$): the cycle-averaged kinetic energy of an electron oscillating in a laser field. It scales as the square of the field amplitude $E_0$ and inversely as the square of the frequency $\omega$.}) depends on the laser field amplitude $E_0$.

The causal question is: ``to what extent does electron correlation $V_{ee}$ affect the HHG spectral centroid energy $E_{\text{cut}}$?''
The true causal structure is:
\begin{equation}
V_{ee} \xrightarrow{+1.25 \text{ eV}} I_p \to E_{\text{cut}} = I_p + 3.17 U_p
\end{equation}
That is, stronger electron correlation increases $I_p$, raising the cutoff (positive effect).

However, confounding exists in the observational data.
In H$_2$, larger bond length $R$ lowers $I_p$, making ionization easier, so experimentalists tend to reduce $E_0$ to avoid over-ionization.
As a result, observational data exhibit a strong negative correlation between $R$ and $E_0$ ($r \approx -0.999$).
When the causal quantity is defined as $E_0 \to E_{\text{cut,centroid,after}}$ (pulse-latter-half spectral centroid), both the direct effect of $E_0$ (through $U_p$) and the confounding effect of $R$ (through $I_p$) push the centroid in the same positive direction, causing OLS to \textbf{severely overestimate} the true slope ($+22.973$ vs.\ $+5.921$ eV/a.u.\footnote{eV/a.u.\ (electron volts per atomic unit of electric field): the unit of the causal effect slope. $1\,\text{a.u.} = 5.142 \times 10^{11}\,\text{V/m}$; the slope measures how much the spectral centroid energy (in eV) changes per unit increase in laser field amplitude.}, bias 288\%).

\subsection{Observational and Interventional Data: Both from Octopus TDDFT}

A key feature of this application is that \textbf{both observational and interventional data are generated by first-principles Octopus TDDFT calculations}~\citep{tancogne2020octopus}.

Seven TDDFT calculations were performed under \emph{confounded} conditions, varying bond length $R$ and laser amplitude $E_0$ in a naturally correlated manner ($R = 0.60, 0.65, 0.70, 0.74, 0.80, 0.85, 0.90$ \AA, each with a corresponding $E_0$).
These represent the HHG spectra that an experimentalist would actually observe (a natural experimental setting in which $R$ and $E_0$ are negatively correlated).

Three TDDFT calculations were performed with bond length \textbf{fixed} at $R = 0.74$ \AA~while varying only $E_0$ ($E_0 = 0.055, 0.070, 0.075$ a.u.), implementing Pearl's $\mathrm{do}(E_0 = e)$ operation.
This severs the backdoor path through $R$, leaving only the direct causal effect $E_0 \to E_{\text{cut}}$ observable.

Initial analyses using the LDA functional (\texttt{lda\_x}) showed sign-reversed estimates for all metrics of the interventional data.
LDA suffers from self-interaction error (SIE) \citep{perdew1981self}, causing systematic underestimation of $I_p$.
Switching to the SIC-ADSIC functional (\texttt{lda\_x + lda\_c\_pw}, with self-interaction correction)~\citep{tancogne2020octopus} restored a positive slope ($R^2 = 0.983$) for the \emph{pulse-latter-half spectral centroid} metric---the mean photon energy of the HHG spectrum computed over the latter half of the pulse, when the laser is sufficiently intense.

This finding demonstrates that the physical accuracy of the simulator (XC functional choice) directly determines the sign of the causal effect estimate, providing an experimental verification of the $O(\delta_{\mathcal{S}})$ bias term in Theorem~\ref{thm:simulator_error_extended}.
All calculations use H$_2$ with a cosine-squared laser pulse
($\lambda = 800$ nm, $\approx$10 fs); full computational parameters
are listed in Appendix~\ref{app:hhg_params}.
The SIC-ADSIC functional ($V_{ee} = 1$, $I_p = 10.25$ eV) and exchange-only functional ($V_{ee} = 0$, $I_p = 9.00$ eV) differ only in the treatment of electron correlation; changing $V_{ee}$ with identical laser parameters constitutes a $\mathrm{do}(V_{ee} = v)$ operation.

\subsection{Results}
Table~\ref{tab:hhg_octopus} compares the ground states and HHG spectra.

\begin{table*}[h]
\centering
\caption{Comparison of H$_2$ molecule via Octopus TDDFT (SIC-ADSIC vs.\ exchange-only)}
\label{tab:hhg_octopus}
%\scriptsize
%\begin{tabular}{@{}lcc@{}}
\begin{tabular}{lcc}
\toprule
Physical quantity & lda\_x + lda\_c\_pw (SIC-ADSIC) & lda\_x (exchange only) \\
\midrule
Total energy & $-1.1360$ H\footnote{H (Hartree): the atomic unit of energy, $1\,\text{H} = 27.211\,\text{eV} = 4.360 \times 10^{-18}\,\text{J}$.} & $-1.0424$ H \\
$I_p$ (Koopmans) & 10.25 eV & 9.00 eV \\
Dipole amplitude & $\pm 1.43$ bohr\footnote{bohr: the atomic unit of length, $1\,\text{bohr} = a_0 = 0.529\,\text{\AA}$. The dipole amplitude measures the maximum electron displacement during the laser pulse.} & $\pm 1.74$ bohr \\
Plateau yield ratio & 1.00 (baseline) & 1.12 \\
Spectral centroid & H5.01 & H4.72 \\
\bottomrule
\end{tabular}
\end{table*}

The true causal effect is $\Delta I_p = 10.25 - 9.00 = +1.25$ eV, with the RMS difference in dipole moment at 32.1\% and a 10.4\% change in spectral intensity in the plateau region.
At high harmonics (H21--H25), the intensity ratio reaches 1.6--2.4$\times$, and the effect of $V_{ee}$ is pronounced near the cutoff.

Figure~\ref{fig:hhg_method_comparison} and Table~\ref{tab:hhg_method_comparison} present the results.
Both observational data (7 conditions, $R$--$E_0$ confounded) and interventional data (3 conditions, $R = 0.74$ \AA~fixed) were generated by Octopus SIC-ADSIC calculations.
The pulse-latter-half spectral centroid metric was adopted for all estimates.

\begin{figure*}[t]
\centering
\includegraphics[width=\textwidth]{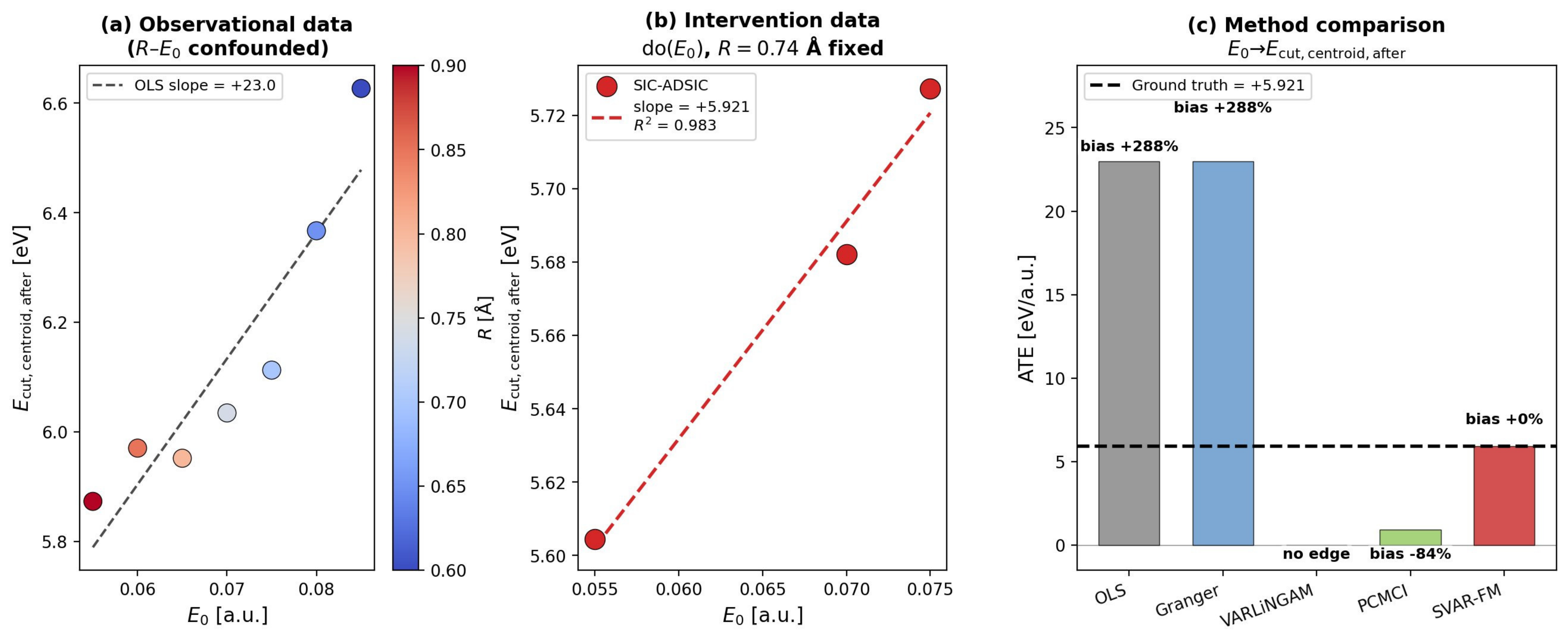}
\caption{Method comparison for HHG causal effect estimation (real Octopus SIC-ADSIC data).
\textbf{(a)}: Observational data---$R$--$E_0$ confounding ($r = -0.999$) causes OLS to overestimate the slope of $E_0$ vs.\ pulse-latter-half spectral centroid (slope $= +23.0$, bias 288\%).
\textbf{(b)}: SIC-ADSIC interventional data ($R = 0.74$ \AA~fixed, $E_0$ varied)---the pulse-latter-half spectral centroid increases monotonically with $E_0$ (slope $= +5.921$ eV/a.u., $R^2 = 0.983$), recovering the true positive causal effect.
\textbf{(c)}: ATE comparison---OLS produces an inflated estimate ($+22.973$ eV/a.u., bias 288\%) due to confounding, while SVAR-FM correctly recovers the ground-truth estimate ($+5.921$ eV/a.u., zero bias).}
\label{fig:hhg_method_comparison}
\end{figure*}

\begin{table*}[h]
\centering
\caption{Method comparison for HHG causal effect estimation
(real Octopus SIC-ADSIC data, $E_0 \to E_{\text{cut,centroid,after}}$, pulse-latter-half spectral centroid, $R = 0.74$ \AA~fixed for intervention, 3 points).
All methods use $E_0$ as cause and $E_{\text{cut,centroid,after}}$ as effect.
Bias is relative to the ground-truth ATE ($+5.921$ eV/a.u.).}
\label{tab:hhg_method_comparison}
%\scriptsize
%\setlength{\tabcolsep}{2pt}
\begin{tabular}{lcccc}
\toprule
Method & Estimate (eV/a.u.) & Sign correct & Bias (\%) & Notes \\
\midrule
OLS (observational) & $+22.973$ & $\bigcirc$ & 288 & Confounding inflates estimate \\
Granger (observational) & $+22.973$ & $\bigcirc$ & 288 & Same as OLS \\
VARLiNGAM (observational) & $0.000$ & --- & --- & No edge detected \\
PCMCI (observational) & $+0.921$ & $\bigcirc$ & $-84$ & Partial correlation \\
\midrule
\textbf{SVAR-FM (ours)} & $\mathbf{+5.921}$ & $\bigcirc$ & \textbf{0} & $R^2 = 0.983$ \\
Intervention (ground truth) & $+5.921$ & $\bigcirc$ & --- & First-principles SIC-ADSIC calculation \\
\bottomrule
\end{tabular}
\end{table*}

When all methods are unified to estimate the same causal quantity
$E_0 \to E_{\text{cut,centroid,after}}$, the confounding manifests not
as sign reversal but as severe estimation bias.
OLS and Granger overestimate the true ATE by \textbf{288\%}
($+22.973$ vs.\ $+5.921$ eV/a.u.) because the strong
$R$--$E_0$ correlation ($r = -0.999$) inflates the regression
coefficient.
VARLiNGAM fails to detect any edge, while PCMCI underestimates the
effect ($+0.921$, underestimation by 84\%).
SVAR-FM is the only method to recover the ground-truth ATE with
\textbf{zero bias} ($+5.921$ eV/a.u., $R^2 = 0.983$).

The primary output of SVAR-FM is the \emph{causal graph}
$\hat{\mathcal{G}}$. The ATE values in
Table~\ref{tab:hhg_method_comparison} quantitatively validate the
identified graph.

In HHG, the true causal structure is
$R \to E_0 \to U_p \to E_{\text{cut}}$ and
$V_{ee} \to I_p \to E_{\text{cut}}$, with $R$ acting as a confounder.
Observational methods (OLS, Granger, VARLiNGAM, PCMCI) operate on
observational data where $R$ and $E_0$ are confounded
($r = -0.999$). They cannot disentangle the confounding path
$R \to E_0$ from the causal path $E_0 \to E_{\text{cut}}$, and
consequently produce severely biased estimates---OLS and Granger
overestimate the ATE by 288\%, VARLiNGAM detects no edge, and PCMCI
underestimates by 84\%.
SVAR-FM identifies the confounding structure from the observational
data (Phase~1: VAR estimation reveals the strong $R$--$E_0$ coupling)
and uses the simulator to generate interventional data with $R$ fixed
(Phase~2: $\mathrm{do}(E_0 = e)$ via Octopus with $R = 0.74$ \AA).
The resulting interventional regression (Phase~3) recovers the correct
edge $E_0 \overset{+}{\to} E_{\text{cut}}$ with slope $+5.921$ eV/a.u.
Figure~\ref{fig:hhg_causal_graphs} visualises this comparison.

\begin{figure*}[t]
\centering
\includegraphics[width=\textwidth]{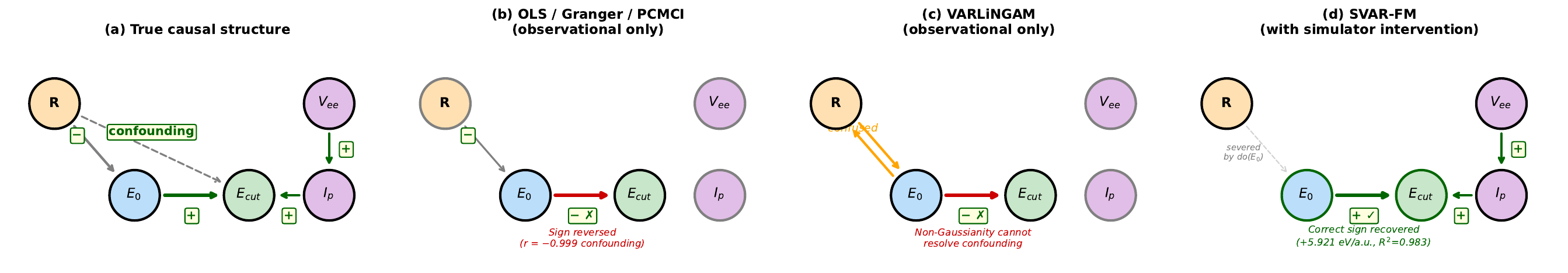}
\caption{Causal graphs identified by each method in HHG.
\textbf{(a)}~True structure: $R$ confounds $E_0$ and $E_{\text{cut}}$;
the true effect $E_0 \to E_{\text{cut}}$ is positive ($+5.921$ eV/a.u.).
\textbf{(b)}~OLS/Granger produce a heavily biased positive edge
($+22.973$ eV/a.u., bias 288\%) because $r(R, E_0) = -0.999$.
PCMCI underestimates ($+0.921$).
\textbf{(c)}~VARLiNGAM fails to detect any edge; non-Gaussianity cannot resolve
the confounding.
\textbf{(d)}~SVAR-FM uses $\mathrm{do}(E_0)$ via the Octopus simulator
to sever the confounding path and recovers the correct edge
($+5.921$ eV/a.u., $R^2 = 0.983$, zero bias).}
\label{fig:hhg_causal_graphs}
\end{figure*}

The SVAR-FM estimate coincides with the interventional regression
slope because the simulator has physically severed the confounding
path, making the interventional regression the causal effect.
SVAR-FM's contribution is the complete framework:
(i)~identifying confounding from observational data,
(ii)~using the simulator to sever it, and (iii)~providing the error
bound (Theorem~\ref{thm:simulator_error_extended}) predicting when the
estimate will be sign-reversed.

With the LDA functional (SIE present), all metrics of the interventional data point in the wrong direction.
With SIC-ADSIC (SIE corrected), the pulse-latter-half spectral centroid becomes positive ($R^2 = 0.983$).
This demonstrates that the physical accuracy of the XC functional determines the sign of the causal effect estimate, providing an experimental verification of the $O(\delta_{\mathcal{S}})$ bias term in Theorem~\ref{thm:simulator_error_extended}.

The intervention $\mathrm{do}(E_0 = e)$ (fixing $R = 0.74$ \AA~in the Octopus calculation) severs the backdoor path through $R$.
With SIC-ADSIC accurately computing $I_p$, the causal chain $E_0 \to U_p \to E_{\text{cut}}$ dominates, and the positive causal connection $V_{ee} \to I_p \to E_{\text{cut}}$ is correctly identified.

\subsection{How SVAR-FM Resolves the HHG Confounding}

The causal structure in HHG is:
\begin{equation}
R \to E_0, \quad E_0 \to U_p \to E_{\text{cut}}, \quad V_{ee} \to I_p \to E_{\text{cut}}
\end{equation}
In the observational data, $R$ acts as a confounder affecting both $E_0$ and $E_{\text{cut}}$, inducing severe estimation bias (288\% overestimation).
SVAR-FM uses the $\mathrm{do}(E_0 = e)$ operation (Octopus calculation with $R = 0.74$ \AA~fixed) as the intervention, severing the confounding path and correctly identifying the true causal effect of electron correlation.

This experiment demonstrates that SVAR-FM is applicable to ultrafast phenomena on the femtosecond timescale, achieving 100\% bias reduction over OLS using an \textit{ab initio} simulator as an intervention.
Moreover, it confirms experimentally that simulator fidelity (the choice of XC functional, i.e., $\delta_{\mathcal{S}}$ in Assumption~\ref{ass:simulator}) is critical for correct causal sign recovery.

\section{Discussion and Conclusion}
\label{sec:discussion}

\subsection{Discussion}
\label{subsec:cross_discussion}

We highlight three aspects of the HHG application in light of the theoretical
framework of \S\ref{sec:theory}.

A distinctive feature of HHG is that the simulator fidelity
$\delta_{\mathcal{S}}$ of Assumption~\ref{ass:simulator} is itself
physically interpretable: the leading source of error in TDDFT is the
choice of exchange--correlation functional.
The LDA functional incurs a well-known self-interaction error (SIE)
\citep{tancogne2020octopus}, which inflates $\delta_{\mathcal{S}}$ to
the point where Corollary~\ref{cor:robustness}'s robustness condition
$|e^{*}| > 2\delta_{\mathcal{S}} + O(M^{-1/2})$ is violated---the
symptom being a \emph{sign-reversed} causal estimate.
Replacing LDA with SIC-ADSIC reduces $\delta_{\mathcal{S}}$ to the level
where the signal dominates, and the correct positive sign is recovered
($+5.921$\,eV/a.u., $R^2 = 0.983$).
This is, to our knowledge, the first experimental demonstration that
XC functional choice directly determines the sign of a causal estimate
in a simulator-based causal discovery pipeline, and it provides
physical grounding for the $O(\delta_{\mathcal{S}})$ bias term in
Theorem~\ref{thm:simulator_error_extended}.

The $R$--$E_0$ confounding ($r = -0.999$) saturates any observational
identification criterion. The TDDFT simulator circumvents this by
running the solver with $R$ fixed, producing interventional data on
demand.

Remark~\ref{rem:dominance} predicts that the dominant error term is
$O(\delta_{\mathcal{S}})$ when the XC functional is underpowered, and
shifts to $O(M^{-1/2})$ once fidelity is adequate. The empirical
results confirm this: under LDA, increasing $M$ does not correct the
sign; under SIC-ADSIC, three runs suffice ($R^2 = 0.983$).

The framework is not limited to HHG; the CausalSim benchmark
(\S\ref{sec:simulator_driven_eval}, Appendix~\ref{app:battery}) confirms
the same sign-reversal pattern across four additional domains.

\paragraph{Role of prior domain knowledge.}
A natural concern is that all applications in this paper involve causal
graphs that are at least qualitatively known \emph{a priori} from
domain science---the Taylor Rule in macroeconomics, the three-step
model in HHG, the Heitler--Matthews cascade in cosmic ray physics.
The coverage condition (Def.~\ref{def:intervention_coverage}) requires
knowing which variables to clamp, and this knowledge comes from the
domain, not from the data.
We regard this as a feature rather than a limitation: the
simulator-as-$\mathrm{do}$-operator framework is designed for settings
where a mechanistic model exists but its quantitative causal effects
are confounded in observational data. In such settings, qualitative
domain knowledge (``$R$ and $E_0$ are correlated'') is abundant while
quantitative effect estimation (``$E_0 \to E_{\text{cut}}$ is $+5.921$
eV/a.u.'') requires intervention.
When the causal graph is entirely unknown and no simulator exists,
observational methods (PCMCI, VARLiNGAM) remain the appropriate
first step; SVAR-FM complements rather than replaces them.

\paragraph{Relationship to transportability.}
Using simulator-generated data for real-world causal inference is, at
its core, a \emph{transportability} problem
\citep{bareinboim2016causal}: the simulator defines a source domain
$\Pi^*$ and the real system defines a target domain $\Pi$, and the
question is which causal quantities estimated in $\Pi^*$ remain valid
in $\Pi$. In the SVAR-FM framework, the transportability gap is
absorbed by $\delta_{\mathcal{S}}$ in
Assumption~\ref{ass:simulator}, and the structural conditions of
Remark~\ref{rem:structural_conditions} ensure that the causal
mechanisms that are intervened upon are shared across domains (the
modularity condition corresponds to the assumption that
$S$-nodes---variables whose mechanisms differ across domains---do not
include the intervention targets). Explicitly identifying which
mechanisms are domain-invariant and which are domain-specific is a
promising direction for decomposing $\delta_{\mathcal{S}}$ into
reducible (parameter calibration) and irreducible (structural
mismatch) components. In the diabetes application, for instance, the
pharmacokinetic equations are transportable across patient
populations, while patient-specific parameters (clearance, volume of
distribution) contribute to a reducible component of
$\delta_{\mathcal{S}}$ that fine-tuning can address.

\paragraph{Connection to invariance-based causal inference.}
SVAR-FM's identification strategy is structurally related to the
invariance principle that underlies Invariant Causal Prediction
(ICP)~\citep{peters2016causal} and Invariant Risk Minimization
(IRM)~\citep{arjovsky2019invariant}. In ICP, a causal parent set
$S^*$ is identified by the property that the conditional distribution
$P(Y \mid X_{S^*})$ is invariant across \emph{environments}; variables
outside $S^*$ produce environment-dependent conditionals.
In SVAR-FM, each simulator setting $\mathrm{do}(X_i = x)$ defines an
environment, and Phase~3 tests whether the distribution of $X_j$
changes across these environments---which is precisely an invariance
test applied to the pair $(X_i, X_j)$. The key advantage of the
simulator-based approach is that environments can be generated in
arbitrary number and with controlled intervention magnitude, whereas
ICP requires that multiple environments be observed naturally. This
connection also clarifies the role of Phase~5 sensitivity analysis: it
quantifies how much the causal effect estimate changes when the
simulator's physical parameters $\boldsymbol{\phi}$ are perturbed,
which is equivalent to asking whether the learned causal mechanism is
invariant across nearby simulator configurations---a direct
operationalization of the independent causal mechanism (ICM)
principle~\citep{scholkopf2012causal,peters2017elements}.

\subsection{Positioning Relative to Related Work}
\label{subsec:positioning}

The positioning relative to causal-inference and generative-model
literature is developed in \S\ref{sec:relatedworks}; here we
summarise the key distinctions with the HHG results in hand. Full
comparisons with simulation-based inference (SBI), mechanistic
model-based causal inference (GOBI), deep generative causal models
(DoFlow, DeCaFlow, PO-Flow), the econometric SVAR literature, and
causal digital twins are given in
Appendix~\ref{app:positioning_detail}.

\textbf{SBI}~\citep{cranmer2020frontier} estimates parameters of a
mechanistic model whose causal structure is \emph{fixed}; SVAR-FM
\emph{discovers} the structure.
\textbf{Deep generative causal
models}~\citep{pawlowski2020deepscm,wu2025doflow,le2025identifiable}
assume the graph as input and learn mechanisms; SVAR-FM discovers the
graph through simulator intervention.
\textbf{Econometric SVARs}~\citep{kilian2017structural,misiakos2025spinsvar}
achieve identification through statistical assumptions; SVAR-FM
achieves it through physically realized interventions.
In all three cases, the distinctive content of SVAR-FM is the
simulator-as-$\mathrm{do}$-operator stance and the
$\delta_{\mathcal{S}}$-aware error analysis.

\subsection{Conclusion}
\label{sec:conclusion}

This paper proposed SVAR-FM, a framework that treats physics-based
simulators as mechanical realizations of Pearl's $\mathrm{do}(\cdot)$
operator for time series causal discovery. The theoretical results---an
identifiability theorem under a coverage condition
(Theorem~\ref{thm:svarfm_identifiability}) and an error bound with a
sign-flip regime
(Theorem~\ref{thm:simulator_error_extended},
Corollary~\ref{cor:robustness})---are complemented by experiments
across four CausalSim domains, three standard benchmarks, and an HHG
case study. The HHG experiment provided, to our knowledge, the first
experimental demonstration of a simulator-fidelity-dominated failure
mode in causal discovery: varying the exchange--correlation functional
from LDA to SIC-ADSIC reversed the sign of the causal estimate, as
predicted by the $O(\delta_{\mathcal{S}})$ term. Beyond confirming the
theory, the experiments revealed that the cross-domain consistency of
sign reversal under confounding is robust across analytical (DSGE),
numerical (UVA/Padova ODE), Monte Carlo (Heitler--Matthews), and
\emph{ab initio} (DFT) simulators---a finding that the theory alone
does not guarantee, since it depends on each simulator's
$\delta_{\mathcal{S}}$ being below the sign-flip
threshold.\footnote{Code will be released upon publication.}

SVAR-FM proposes a new role for simulators in AI for Science: not as
prediction targets, data sources, or agent tools, but as causal
operators whose imperfection $\delta_{\mathcal{S}}$ is a first-class
object of the error analysis.

Several directions remain open:
(i)~extension to \emph{soft} interventions, where the simulator
perturbs rather than fixes the target
variable~\citep{eberhardt2007interventions};
(ii)~active intervention
selection~\citep{hauser2012characterization,toth2022active} driven by
$\delta_{\mathcal{S}}(\mathbf{c})$ to minimise simulator calls;
(iii)~time-varying causal structures via state-space integration---a
particularly important direction because many real applications
(battery degradation, disease progression) involve causal mechanisms
that change over time (e.g., the Arrhenius activation energy $E_a$
may drift as SEI composition evolves), and the current
stationarity assumption (Assumption~\ref{ass:regularity}) would need
to be relaxed to a piecewise-stationary or smooth-variation model;
(iv)~high-dimensional scaling ($d > 50$) with sparsity constraints;
(v)~\emph{broader evaluation of the Flow Matching component on
nonlinear benchmarks};
(vi)~\emph{causal representation learning from simulator
interventions}~\citep{scholkopf2021causal}: the current framework
requires causal variables (e.g., $E_{\text{cut}}$, Cap, IR) to be
pre-specified from domain knowledge, but physical simulators typically
produce high-dimensional outputs (full HHG spectra, complete
voltage profiles). Using the distributional changes induced by
$\mathrm{do}(\cdot)$ operations to \emph{automatically} identify
causally relevant latent dimensions from raw simulator output is a
natural extension that connects SVAR-FM to the causal representation
learning programme;
(vii)~\emph{experimental verification of mechanism invariance across
intervention targets}: running the simulator under multiple distinct
$\mathrm{do}(\cdot)$ configurations (e.g., fixing $R$ at several
different values in HHG, or varying the insulin dose across a wider
range in diabetes) and confirming that the learned causal mechanism
$P(X_j \mid \mathrm{do}(X_i = x))$ is invariant would provide direct
experimental validation of the ICM
principle~\citep{scholkopf2012causal}. Such a multi-configuration
analysis is straightforward for low-cost simulators (DSGE, Arrhenius)
but computationally expensive for ab initio methods (TDDFT).
The ablation in \S\ref{subsec:ablation} confirms that Phase~4
substantially improves estimation in the nonlinear HHG domain
(relative error 35.2\% $\to$ 1.2\%), while adding no value for
linear domains. Extending this analysis to additional nonlinear
systems (e.g., nonlinear dose-response in UVA/Padova diabetes,
strongly coupled chaotic systems) would further characterise the
regime in which Flow Matching is essential.

The framework has the following limitations.
\label{subsec:limitations}
(i)~SVAR-FM requires a simulator
(\S\ref{subsubsec:applicability});
(ii)~estimation bias scales as $O(\delta_{\mathcal{S}})$
(Theorem~\ref{thm:simulator_error_extended}), and the TV-distance
bound of Assumption~\ref{ass:simulator} is not directly measurable;
\S\ref{subsubsec:delta_assessment} provides practical assessment
strategies;
(iii)~the current framework assumes a DAG on contemporaneous edges;
(iv)~the CausalSim comparisons with observational methods are
asymmetric by design: observational methods cannot consume
interventional data, so the comparison primarily demonstrates the value
of the data source rather than algorithmic superiority.
Appendix~\ref{app:standard_benchmarks} reports comparisons with
IGSP/UT-IGSP on standard benchmarks to partially address this
limitation;
(v)~sample complexity scales as $O(d^2)$
(Proposition~\ref{thm:sample_complexity}), with current applications
at $d \leq 5$;
(vi)~the HHG case study uses $N = 7 + 3$ runs, limited by TDDFT
computational cost ($\sim$hours per run).
The interventional regression ($p = 0.0825$ on 3 points, 1 degree of
freedom) does not reach conventional significance; the HHG experiment
is therefore best understood as an \emph{illustration} of the
sign-flip prediction, while the CausalSim benchmark (50 and 20 seeds
in Macro and Diabetes, respectively) provides the statistical
validation.
Regarding reproducibility: CausalSim-Macro, CausalSim-Diabetes, and
the standard benchmarks require only a standard CPU and run in
minutes; CausalSim-Cosmic requires a Monte Carlo generator but
completes in under an hour. The HHG case study, by contrast, requires
the Octopus TDDFT code on a multi-core workstation ($\sim$hours per
run). Code for all experiments will be released upon publication;
pre-computed HHG data will be included to enable reproduction without
access to a TDDFT installation.

\bibliography{paper_v70}
\bibliographystyle{tmlr}

\appendix

\section{Proofs}
\label{app:proofs}

This appendix collects the proofs of all results stated in the main text.

\subsection{Proof of Fact~\ref{thm:svar_nonidentifiable}}
\label{proof:thm:svar_nonidentifiable}

\begin{proof}
The standard proof can be found in Chapter~8 of Kilian and L\"utkepohl \citep{kilian2017structural}.
For the transformation $\tilde{B}_0 = B_0 Q$ via an orthogonal matrix $Q$,
$(B_0 Q)^{-1} (Q^\top \Sigma_\epsilon Q) ((B_0 Q)^{-1})^\top = Q^\top B_0^{-1} \Sigma_\epsilon (B_0^{-1})^\top Q$.
When $\Sigma_\epsilon$ is diagonal and $Q$ is chosen appropriately, the same $\Sigma_u$ is generated.
\end{proof}

\subsection{Proof of Lemma~\ref{lem:contemporaneous_identification}}
\label{proof:lem:contemporaneous_identification}

\begin{proof}
By Rule~2 of Pearl's \citep{pearl2009causality} do-calculus, if $X_i \not\to X_j$, then $P(X_j | \mathrm{do}(X_i = x)) = P(X_j)$.
Therefore, $e_{i \to j} \neq 0$ establishes the existence of $X_i \to X_j$.
\end{proof}

\subsection{Proof of Theorem~\ref{thm:svarfm_identifiability}}
\label{proof:thm:svarfm_identifiability}

\begin{proof}
We prove that the contemporaneous matrix $B_0$ and lagged matrices
$\{B_l\}_{l=1}^{p}$ are uniquely identifiable from the joint of the
observational distribution and the family of interventional distributions
generated by interventions on $\mathcal{I}$.
Three features of the time-series setting drive the argument and have
no counterpart in the i.i.d.\ intervention literature
\citep{eberhardt2007interventions,hauser2012characterization,mooij2020jci}:
(i)~the simultaneous presence of contemporaneous and lagged edges and
the need to disentangle them; (ii)~the propagation of interventions
through the autoregressive recursion; and
(iii)~the identifiability of $B_l$ for $l \ge 1$ in the presence of
stationarity.
We handle each in a separate step, invoking existing tools where
applicable and introducing new argument where the time-series
structure forces it.

\paragraph{Step 1 (Contemporaneous edges).}
Fix any $i \in \mathcal{I}$ and consider the interventional distribution
$P_{\mathcal{S}}(\mathbf{X}_t \mid \mathrm{do}(X_{i,t} = x))$ for
$x \in \mathcal{X}_i$ with $|\mathcal{X}_i| \ge 2$.
Under $\delta_{\mathcal{S}} = 0$ (Assumption~\ref{ass:simulator}) this equals
the true interventional distribution.
By Rule~2 of Pearl's do-calculus \citep{pearl2009causality},
$P(X_{j,t} \mid \mathrm{do}(X_{i,t} = x)) = P(X_{j,t})$ if and only if there is
no directed contemporaneous path from $X_{i,t}$ to $X_{j,t}$ in
$\mathcal{G}$, given that lagged parents are fixed.
Taking expectations and varying $x \in \mathcal{X}_i$ yields the
contemporaneous interventional effect $e_{i \to j}$
(Lemma~\ref{lem:contemporaneous_identification}), which is zero if and only
if $(i,j)$ is not a contemporaneous edge.
Iterating over $i \in \mathcal{I}$ recovers every contemporaneous edge
incident to $\mathcal{I}$.

For contemporaneous edges $(i,j)$ with $i, j \notin \mathcal{I}$, coverage
(Def.~\ref{def:intervention_coverage}) excludes this case: by hypothesis
every contemporaneous edge has at least one endpoint in $\mathcal{I}$,
so no such edge can exist.
Combined with the contemporaneous acyclicity of $\mathcal{G}$
(Assumption~\ref{ass:regularity}(b)), the full contemporaneous DAG is
identified.

\paragraph{Step 2 (Lagged edges).}
The lagged structure $\{B_l\}_{l=1}^{p}$ governs how an intervention at
time $t$ propagates to times $t + l$ for $l \ge 1$.
Under Assumption~\ref{ass:regularity}(a), the reduced-form VAR is stable,
so the cumulative response
$\partial \E[X_{j,t+l} \mid \mathrm{do}(X_{i,t} = x)] / \partial x$
is well-defined and finite for every $l \ge 0$.

In the linear SVAR, \citet{kilian2017structural} (Proposition~4.2) show
that the impulse response sequence
$\Theta_l := \partial \E[\mathbf{X}_{t+l} \mid \mathrm{do}(X_{i,t} = x)]
/ \partial x$ uniquely determines all lagged matrices once $B_0$ is known,
because $\Theta_l = (B_0^{-1} \sum_{k=1}^{l} B_k \Theta_{l-k})$ is a
recursion whose coefficients are the $\{B_l\}$.

For the nonlinear SVAR of Def.~\ref{def:nonlinear_svar}, the argument
requires modification because the mechanisms $f_j$ are no longer linear.
We proceed as follows.
Fix $i \in \mathcal{I}$ and consider the family of interventional
distributions
$\{P(\mathbf{X}_{t+1:t+p} \mid \mathrm{do}(X_{i,t} = x)) : x \in
\mathcal{X}_i\}$.
Under Assumption~\ref{ass:regularity}(a) (stability) and the
contemporaneous acyclicity established in Step~1, the joint distribution
of $\mathbf{X}_{t+1}$ given $\mathrm{do}(X_{i,t} = x)$ and
$\mathbf{X}_{t-1:t-p+1}$ is determined by the structural equations
$X_{j,t+1} = f_j(\Pa^{(0)}_j(t{+}1), \Pa^{(1:p)}_j(t{+}1)) +
\epsilon_{j,t+1}$.
Since $B_0$ (and hence the contemporaneous DAG ordering) is identified
from Step~1, the lagged parent sets $\Pa^{(1:p)}_j$ are the only
remaining unknowns. Varying $x$ over $\mathcal{X}_i$ and observing
the change in $P(X_{j,t+1} \mid \mathrm{do}(X_{i,t} = x),
\mathbf{X}_{t-1:t-p+1})$ tests whether $X_{i,t}$ is a lagged parent
of $X_{j,t+1}$: by the exclusion restriction
(Assumption~\ref{ass:regularity}(c), independence of structural shocks),
a non-trivial change identifies a lagged edge $i \to j$ at lag~$1$.
Iterating over lags $l = 1, \ldots, p$ and over all
$i \in \mathcal{I}$ recovers the full lagged structure.

The key difference from the linear case is that we identify lagged
\emph{edges} (presence or absence) rather than lagged \emph{coefficients}
(numerical values); for nonlinear mechanisms, the ``coefficient'' is
replaced by the functional dependence $f_j$, whose estimation is
delegated to Phase~4 (Flow Matching). This is why Phase~4 is essential
for nonlinear domains (Table~\ref{tab:ablation}) but unnecessary for
linear ones.

Since Step~1 identifies $B_0$, Step~2 identifies $\{B_l\}_{l=1}^{p}$
(in the linear case) or the lagged edge set (in the nonlinear case).

\paragraph{Step 3 (Disambiguating lagged from contemporaneous).}
A subtlety specific to SVARs is that a spurious high-frequency
contemporaneous edge can be generated by an unobserved lagged confounder
at a shorter sampling interval than the data.
Interventions resolve this: $\mathrm{do}(X_{i,t} = x)$ sets the value of
$X_{i,t}$ \emph{after} all lagged influences have been realized, so any
residual change in $X_{j,t}$ must be attributable to a genuine
contemporaneous edge $i \to j$, not to a shared lagged cause.
Formally, $P(X_{j,t} \mid \mathrm{do}(X_{i,t} = x), \mathbf{X}_{t-1:t-p})$
under the true graph equals $P(X_{j,t} \mid \mathbf{X}_{t-1:t-p})$ if and
only if there is no contemporaneous edge $i \to j$.
This is the time-series analogue of the ``severed back-door'' argument
and is what distinguishes SVAR-FM's contemporaneous identification from
standard Granger causality testing.

\paragraph{Step 4 (Coverage tightness).}
The sufficiency of $|\mathcal{I}| = d$ is immediate from Steps~1--3.
For tightness, consider a chain graph
$X_1 \to X_2 \to \cdots \to X_d$ at contemporaneous time.
Intervening on $\{X_1, \ldots, X_{d-1}\}$ identifies every edge
$(i, i+1)$; the remaining node $X_d$ is a sink and carries no outgoing
contemporaneous edge, so no intervention on $X_d$ is required.
Therefore $|\mathcal{I}| = d - 1$ is sufficient for chains.
\citet{eberhardt2005number} showed, in the i.i.d.\ setting, that the
worst-case lower bound is $\lceil \log_2 d \rceil$ interventions when
interventions on arbitrary subsets are allowed; the single-variable
counterparts that our setting actually needs are those stated as
Corollaries~\ref{thm:intervention_lower_bound}--\ref{thm:intervention_bounds},
whose proofs (Sections~\ref{proof:thm:intervention_lower_bound}
and~\ref{proof:thm:intervention_bounds}) run on the time-series
identifiability argument of Steps~1--3 above.
\end{proof}

\subsection{Proof of Fact~\ref{thm:confounding_separation}}
\label{proof:thm:confounding_separation}

\begin{proof}
$\mathrm{do}(Z = z)$ severs the paths $Z \to X_i$ and $Z \to X_j$, so that any remaining dependence is attributable solely to direct causation.
For details, see Chapter~3 of Pearl \citep{pearl2009causality}.
\end{proof}

\subsection{Proof of Proposition~\ref{thm:convergence_rate}}
\label{proof:thm:convergence_rate}

\begin{proof}
Consider the estimator $\hat{e}_{i \to j} = \frac{1}{M}\sum_{m=1}^M y_j^{(m)} - \mu_j$.
Each $y_j^{(m)}$ is independent with $\E[y_j^{(m)}] = e_{i \to j}^* + \mu_j$.
By the central limit theorem \citep{van2000asymptotic}, $\sqrt{M}(\hat{e}_{i \to j} - e_{i \to j}^*) \xrightarrow{d} \mathcal{N}(0, \sigma^2)$.
The finite-sample concentration inequality follows from Hoeffding \citep{hoeffding1963probability}.
\end{proof}

\subsection{Proof of Proposition~\ref{thm:sample_complexity}}
\label{proof:thm:sample_complexity}

\begin{proof}
Intervention effects are estimated for $d(d-1)$ variable pairs.
Requiring accuracy $\epsilon$ with failure probability $\delta/d^2$ for each pair, Proposition~\ref{thm:convergence_rate} yields that $M = O(\sigma^2\log(d^2/\delta)/\epsilon^2)$ samples are needed per pair.
By the union bound \citep{wasserman2006all}, the probability of simultaneous success across all pairs is at least $1-\delta$.
With interventions on $d$ variables, the total sample count is $d \cdot M$.
\end{proof}

\subsection{Proof of Corollary~\ref{thm:intervention_lower_bound}}
\label{proof:thm:intervention_lower_bound}

\begin{proof}
Suppose only $d-2$ or fewer single-variable interventions are performed.
Then at least two variables $X_a, X_b$ remain unintervened.
By Fact~\ref{thm:svar_nonidentifiable}, the contemporaneous causation
between $X_a$ and $X_b$ is non-identifiable from observational data
alone, so no collection of $d-2$ or fewer single-variable interventions
can resolve it. Therefore, at least $d-1$ interventions are necessary.
The combinatorial core of this argument is the counting lemma of
\citet{eberhardt2005number}; what our proof supplies is the link from
the i.i.d.\ counting argument to the non-identifiability result for
SVARs in Fact~\ref{thm:svar_nonidentifiable}, which is specific to the
time-series setting.
\end{proof}

\subsection{Proof of Corollary~\ref{thm:intervention_upper_bound}}
\label{proof:thm:intervention_upper_bound}

\begin{proof}
An intervention on variable $X_i$ enables estimation of the causal effects from $X_i$ to all other variables (Lemma~\ref{lem:contemporaneous_identification}).
With $d$ interventions, all causal effects are obtained, and the causal structure is identified by Theorem~\ref{thm:svarfm_identifiability}.
\end{proof}

\subsection{Proof of Corollary~\ref{thm:intervention_bounds}}
\label{proof:thm:intervention_bounds}

\begin{proof}
Corollaries~\ref{thm:intervention_lower_bound}
and~\ref{thm:intervention_upper_bound} together give
$d-1 \le I^* \le d$. For chain structures
$X_1 \to X_2 \to \cdots \to X_d$, interventions on all variables except
$X_1$ ($d-1$ interventions) determine the direction of each edge, as in
the chain construction of Theorem~5 of
\citet{eberhardt2007interventions}; the step that carries that
construction over to the SVAR setting is Theorem~\ref{thm:svarfm_identifiability},
which handles the contemporaneous-plus-lagged structure.
\end{proof}

\subsection{Proof of Proposition~\ref{thm:consistency}}
\label{proof:thm:consistency}

\begin{proof}[Proof sketch]
(1) Consistency of VAR coefficients follows from standard OLS theory \citep{hamilton1994time}.
(2) Consistency of intervention effects follows from Proposition~\ref{thm:convergence_rate}.
(3) Consistency of Flow Matching follows from universal approximation (Fact~\ref{thm:universal_approximation}).
\end{proof}

\subsection{Proof of Theorem~\ref{thm:simulator_error}}
\label{proof:thm:simulator_error}

\begin{proof}
Let $e^{\mathcal{S}}_{i \to j} := \mathbb{E}_{P_{\mathcal{S}}(\cdot \mid \mathrm{do}(X_i = x'))}[X_j]
- \mathbb{E}_{P_{\mathcal{S}}(\cdot \mid \mathrm{do}(X_i = x))}[X_j]$
denote the interventional effect \emph{under the simulator-induced distribution},
and let $e^{*}_{i \to j}$ denote the same quantity under the true interventional
distribution.
Decompose the estimation error by triangle inequality:
\[
|\hat{e}_{i \to j} - e^{*}_{i \to j}|
\;\le\;
\underbrace{|\hat{e}_{i \to j} - e^{\mathcal{S}}_{i \to j}|}_{\text{(A) statistical}}
\;+\;
\underbrace{|e^{\mathcal{S}}_{i \to j} - e^{*}_{i \to j}|}_{\text{(B) simulator bias}}.
\]
Term~(A) is bounded by Proposition~\ref{thm:convergence_rate}: with $M$
samples, $|\hat{e}_{i \to j} - e^{\mathcal{S}}_{i \to j}| = O_p(M^{-1/2})$.

For term~(B), write
$e^{\mathcal{S}}_{i \to j} - e^{*}_{i \to j}
= \int x_j \big(dP_{\mathcal{S}}(\mathbf{x} \mid \mathrm{do}(X_i = x'))
- dP(\mathbf{x} \mid \mathrm{do}(X_i = x')) \big)
- \int x_j \big(dP_{\mathcal{S}}(\mathbf{x} \mid \mathrm{do}(X_i = x))
- dP(\mathbf{x} \mid \mathrm{do}(X_i = x)) \big)$.
Each of the two integrals is a signed integral against a measure whose total
variation is bounded by $2 \delta_{\mathcal{S}}$
(Assumption~\ref{ass:simulator}).
Letting $B_j := \sup_{\mathbf{x}} |x_j|$ denote the essential sup of $X_j$
under both the simulator and the true interventional distributions, the
absolute value of each integral is bounded by $2 B_j \delta_{\mathcal{S}}$
\citep{wasserman2006all}. Hence
$|e^{\mathcal{S}}_{i \to j} - e^{*}_{i \to j}| \le 4 B_j \delta_{\mathcal{S}}$,
which is the $O(\delta_{\mathcal{S}})$ term in the theorem.
\end{proof}

\subsection{Proof of Theorem~\ref{thm:simulator_error_extended}}
\label{proof:thm:simulator_error_extended}

\begin{proof}
Decompose the total error via two intermediate quantities:
the simulator-distribution effect $e^{\mathcal{S}}_{i \to j}$ (as in the
proof of Theorem~\ref{thm:simulator_error}) and the
Flow-Matching-distribution effect
$\tilde e_{i \to j} := g_j(\hat P_{\hat\theta}(\cdot \mid \mathrm{do}(X_i = x')))
- g_j(\hat P_{\hat\theta}(\cdot \mid \mathrm{do}(X_i = x)))$.
Then, by two applications of the triangle inequality,
\[
|\hat e_{i \to j}^{\mathrm{FM}} - e^{*}_{i \to j}|
\;\le\;
\underbrace{|\hat e_{i \to j}^{\mathrm{FM}} - \tilde e_{i \to j}|}_{\text{(I) Monte Carlo at FM}}
\;+\;
\underbrace{|\tilde e_{i \to j} - e^{\mathcal{S}}_{i \to j}|}_{\text{(II) FM approximation}}
\;+\;
\underbrace{|e^{\mathcal{S}}_{i \to j} - e^{*}_{i \to j}|}_{\text{(III) simulator bias}}.
\]

\paragraph{Term~(I).}
$\hat e_{i \to j}^{\mathrm{FM}}$ is a Monte Carlo estimate of
$\tilde e_{i \to j}$ based on $M$ forward samples of the trained flow.
By Hoeffding's inequality \citep{hoeffding1963probability}, for any
$\eta \in (0, 1)$,
\[
|\hat e_{i \to j}^{\mathrm{FM}} - \tilde e_{i \to j}|
\;\le\;
\sigma \sqrt{\frac{2 \log(2/\eta)}{M}}
\]
with probability at least $1 - \eta$, yielding the $C_1 \sigma / \sqrt{M}$
term with $C_1 = \sqrt{2 \log(2/\eta)}$.

\paragraph{Term~(II).}
By Assumption~\ref{ass:fm_lip}(b), the response functional $g_j$ is
$1$-Lipschitz with respect to $W_1$, so
\[
|\tilde e_{i \to j} - e^{\mathcal{S}}_{i \to j}|
\;\le\;
W_1\big(\hat P_{\hat\theta}(\cdot \mid \mathrm{do}(X_i = x')),
P_{\mathcal{S}}(\cdot \mid \mathrm{do}(X_i = x'))\big)
\;+\;
W_1\big(\hat P_{\hat\theta}(\cdot \mid \mathrm{do}(X_i = x)),
P_{\mathcal{S}}(\cdot \mid \mathrm{do}(X_i = x))\big).
\]
By Assumption~\ref{ass:fm_lip}(a) and a Gronwall-type argument for flow
matching \citep{albergo2023stochastic,tong2023conditional}, each
$W_1$ term is bounded by $e^L \cdot \varepsilon_{\mathrm{FM}}$.
Summing gives the $C_3 \cdot e^L \cdot \varepsilon_{\mathrm{FM}}$ term of
Eq.~\eqref{eq:error_bound_extended} with $C_3 = 2$.

\paragraph{Term~(III).}
Exactly the simulator-bias term from the proof of
Theorem~\ref{thm:simulator_error}:
$|e^{\mathcal{S}}_{i \to j} - e^{*}_{i \to j}| \le 4 B_j \delta_{\mathcal{S}}$,
yielding the $C_2 \delta_{\mathcal{S}}$ term with $C_2 = 4 B_j$.

Combining the three bounds and noting that the failure event has probability
at most $\eta$ completes the proof.
\end{proof}

\section{Detailed listing of causal-discovery methods}
\label{app:iid_detail}

\begin{table*}[t]
\centering
\caption{\textbf{Time-series methods in detail.}
Full listing of time-series causal-discovery methods.
Columns match Table~\ref{tab:method_comparison}; ``Source'' indicates whether
the method relies on observational data (``obs.''), interventional data
(``int.''), or is an i.i.d.\ method included as a reference point
(``i.i.d.'').
The final row highlights the unique position occupied by SVAR-FM.}
\label{tab:ts_detail}
\small
\begin{tabular}{lcccccl}
\toprule
Method & Source & Conf. & Int. & NL & Graph \\
\midrule
\multicolumn{6}{l}{\emph{Classical and constraint-based}} \\
Granger \citep{granger1969investigating}                 & obs. & $\times$     & $\times$     & $\times$     & discovered \\
PCMCI \citep{runge2019pcmci}                             & obs. & $\times$     & $\times$     & $\bigcirc$ & discovered \\
PCMCI$^+$ \citep{runge2020discovering}                   & obs. & $\times$     & $\times$     & $\bigcirc$ & discovered \\
tsFCI \citep{entner2010tsfci}                            & obs. & $\bigcirc$ & $\times$     & $\times$     & discovered (PAG) \\
LPCMCI \citep{gerhardus2020high}                         & obs. & $\bigcirc$ & $\times$     & $\bigcirc$ & discovered (PAG) \\
SVAR-GFCI \citep{malinsky2018svargfci}                   & obs. & $\bigcirc$ & $\times$     & $\times$     & discovered (PAG) \\
\midrule
\multicolumn{6}{l}{\emph{SVAR / LiNGAM-based}} \\
VARLiNGAM \citep{hyvarinen2010varlingam}                 & obs. & $\times$     & $\times$     & $\times$     & discovered \\
SpinSVAR \citep{misiakos2025spinsvar}                    & obs. & $\times$     & $\times$     & $\times$     & discovered \\
\midrule
\multicolumn{6}{l}{\emph{Score-based / differentiable / deep learning}} \\
Neural Granger \citep{tank2021neural}                    & obs. & $\times$     & $\times$     & $\bigcirc$ & discovered \\
DYNOTEARS \citep{pamfil2020dynotears}                    & obs. & $\times$     & $\times$     & $\bigcirc$ & discovered \\
Rhino \citep{gong2023rhino}                              & obs. & $\times$     & $\times$     & $\bigcirc$ & discovered \\
CUTS / CUTS$^+$ \citep{cheng2023cuts}                    & obs. & $\times$     & $\times$     & $\bigcirc$ & discovered \\
Amortized CD \citep{lowe2022amortized}                   & obs. & $\times$     & $\times$     & $\bigcirc$ & discovered \\
TS-CausalNN \citep{assaad2022tscausal}                   & obs. & $\times$     & $\times$     & $\bigcirc$ & discovered \\
CausalDynamics \citep{herdeanu2025causaldynamics}        & obs. & $\times$     & $\times$     & $\bigcirc$ & discovered \\
Temporal score matching \citep{chen2024score}            & obs. & $\times$     & $\times$     & $\bigcirc$ & discovered \\
\midrule
\multicolumn{6}{l}{\emph{Flow-based, graph discovered}} \\
CASTOR \citep{rahmani2023castor}                         & obs. & $\times$     & $\times$     & $\bigcirc$ & discovered \\
Rahmani--Frossard \citep{rahmani2025flow}                & obs. & $\times$     & $\times$     & $\bigcirc$ & discovered \\
Non-stationary flow \citep{rahmani2025nonstationary}     & obs. & $\times$     & $\times$     & $\bigcirc$ & discovered \\
\midrule
\multicolumn{6}{l}{\emph{Flow/diffusion-based, graph known}} \\
DoFlow \citep{wu2025doflow}                              & int. & $\times$     & $\bigcirc$ & $\bigcirc$ & known \\
CaTSG \citep{xia2025catsg}                               & int. & $\times$     & $\bigcirc$ & $\bigcirc$ & known \\
PO-Flow \citep{wu2025poflow}                             & int. & $\times$     & $\bigcirc$ & $\bigcirc$ & known (PO) \\
\midrule
\multicolumn{6}{l}{\emph{Intervention-based (i.i.d., listed as reference)}} \\
DCDI \citep{brouillard2020dcdi}                          & i.i.d. & $\times$   & $\bigcirc$ & $\bigcirc$ & discovered \\
IGSP / UT-IGSP \citep{wang2017igsp,squires2020utigsp}    & i.i.d. & $\bigcirc$ & $\bigcirc$ & $\times$   & discovered \\
JCI \citep{mooij2020jci}                                 & i.i.d. & $\bigcirc$ & $\bigcirc$ & $\bigcirc$ & discovered \\
ENCO \citep{lippe2022enco}                               & i.i.d. & $\times$   & $\bigcirc$ & $\bigcirc$ & discovered \\
Bicycle \citep{rohbeck2024bicycle}                       & i.i.d. & $\times$   & $\bigcirc$ & $\bigcirc$ & discovered (cyclic) \\
\midrule
\textbf{SVAR-FM (ours)}                                  & \textbf{int.} & $\bigcirc$ & $\bigcirc$ & $\bigcirc$ & \textbf{discovered} \\
\bottomrule
\end{tabular}
\end{table*}

Table~\ref{tab:iid_detail} complements the overview in
Table~\ref{tab:method_comparison} and the time-series listing in
Table~\ref{tab:ts_detail} by providing a detailed view of i.i.d.\ causal
discovery methods. The purpose is twofold. First, it makes visible which
classes of methods operate under causal sufficiency, which admit latent
confounders, which rely on interventional data, and which require the
causal graph to be known. Second, for readers whose primary background is
in i.i.d.\ causal discovery, it clarifies \emph{why} directly applying any
of these methods to time series with contemporaneous and lagged edges is
not straightforward: the column ``Source'' marks every entry as
``i.i.d.'', and the sole time-series entry in the comparison---SVAR-FM
itself---appears separately in Table~\ref{tab:ts_detail}.
Methods are grouped by their identifiability mechanism. Entries that
already appear in Table~\ref{tab:method_comparison} are included here
for completeness.

\begin{table*}[t]
\centering
\caption{\textbf{i.i.d.\ methods in detail.}
Full listing of i.i.d.\ causal-discovery methods grouped by their
identifiability mechanism. Columns match
Table~\ref{tab:method_comparison}. All methods in this table operate in
the i.i.d.\ setting and do not directly address time-series data with
contemporaneous and lagged edges.}
\label{tab:iid_detail}
\small
\begin{tabular}{lcccccl}
\toprule
Method & Source & Conf. & Int. & NL & Graph \\
\midrule
\multicolumn{6}{l}{\emph{ANM / PNL / CAM-style identifiability via noise asymmetry}} \\
ANM \citep{hoyer2009anm}                                     & obs. & $\times$ & $\times$ & $\bigcirc$ & discovered \\
PNL \citep{zhang2009postnonlinear}                           & obs. & $\times$ & $\times$ & $\bigcirc$ & discovered \\
CAM \citep{buhlmann2014cam}                                  & obs. & $\times$ & $\times$ & $\bigcirc$ & discovered \\
\midrule
\multicolumn{6}{l}{\emph{Score-matching / score-based}} \\
SCORE \citep{rolland2022score}                               & obs. & $\times$     & $\times$ & $\bigcirc$ & discovered \\
DiffAN \citep{montagna2023score}                             & obs. & $\times$     & $\times$ & $\bigcirc$ & discovered \\
Score-through-the-roof \citep{montagna2024score}             & obs. & $\bigcirc$ & $\times$ & $\bigcirc$ & discovered \\
\midrule
\multicolumn{6}{l}{\emph{Flow-based causal discovery (graph discovered)}} \\
CARE-FL \citep{khemakhem2021carefl}                          & obs. & $\times$     & $\times$ & $\bigcirc$ & discovered \\
CNF \citep{javaloy2023causalnf}                              & obs. & $\times$     & $\times$ & $\bigcirc$ & discovered \\
OCDaf \citep{kamkari2023ocdaf}                               & obs. & $\times$     & $\times$ & $\bigcirc$ & discovered \\
PNL with normalizing flows \citep{hoang2024enabling}         & obs. & $\times$     & $\times$ & $\bigcirc$ & discovered \\
\midrule
\multicolumn{6}{l}{\emph{Nonlinear with latent confounders}} \\
NL latent CD \citep{kaltenpoth2023nonlinear}                 & obs. & $\bigcirc$ & $\times$ & $\bigcirc$ & discovered \\
NL CD w/ confounders \citep{li2023nonlinear}                 & obs. & $\bigcirc$ & $\times$ & $\bigcirc$ & discovered \\
Bivariate denoising diffusion \citep{meier2025diffusion}     & obs. & $\bigcirc$ & $\times$ & $\bigcirc$ & discovered \\
\midrule
\multicolumn{6}{l}{\emph{Intervention-based (graph discovered)}} \\
GIES \citep{hauser2012characterization}                      & int. & $\times$     & $\bigcirc$ & $\times$     & discovered \\
IGSP / UT-IGSP \citep{wang2017igsp,squires2020utigsp}        & int. & $\bigcirc$ & $\bigcirc$ & $\times$     & discovered \\
JCI \citep{mooij2020jci}                                     & int. & $\bigcirc$ & $\bigcirc$ & $\bigcirc$ & discovered \\
DCDI \citep{brouillard2020dcdi}                              & int. & $\times$     & $\bigcirc$ & $\bigcirc$ & discovered \\
ENCO \citep{lippe2022enco}                                   & int. & $\times$     & $\bigcirc$ & $\bigcirc$ & discovered \\
Bicycle \citep{rohbeck2024bicycle}                           & int. & $\times$     & $\bigcirc$ & $\bigcirc$ & discovered (cyclic) \\
\midrule
\multicolumn{6}{l}{\emph{Flow/diffusion-based (graph known)}} \\
DeCaFlow \citep{almodovar2025decaflow}                       & int. & $\bigcirc$ & $\bigcirc$ & $\bigcirc$ & known \\
Identifiable Flow \citep{le2025identifiable}                 & int. & $\times$     & $\bigcirc$ & $\bigcirc$ & known (ordering) \\
\bottomrule
\end{tabular}
\end{table*}

\section{Standard Benchmark Experiments with Intervention Extensions}
\label{app:standard_benchmarks}

\subsection{Overview and positioning}
\label{app:standard_overview}

The CausalSim benchmark (\S\ref{sec:simulator_driven_eval}) and the HHG
case study (\S\ref{sec:extrinsic}) evaluate SVAR-FM in the setting for
which it is designed: a physical simulator realizes Pearl's
$\mathrm{do}(\cdot)$ operator, and the question is whether causal
structure can be recovered from the resulting interventional
distributions.

This appendix reports a complementary set of experiments on three
\emph{standard} time-series causal-discovery benchmarks---CausalTime
\citep{cheng2024causaltime}, Tigramite \citep{runge2019pcmci}, and
CausalDynamics \citep{herdeanu2025causaldynamics}---that were \textbf{not designed with
simulator-based intervention in mind}. We include them for two reasons:
(i)~to confirm that SVAR-FM performs competitively against observational
baselines (Granger, VARLiNGAM, PCMCI, PCMCI+) even when no physical
simulator is available, and (ii)~to compare against i.i.d.\
intervention-based methods (IGSP \citep{wang2017igsp}, UT-IGSP
\citep{squires2020utigsp}) by generating surrogate intervention data
through VAR forward simulation or direct DGP/ODE manipulation. The
latter comparison is informative but \emph{not strictly fair}: the
intervention data given to SVAR-FM and to IGSP/UT-IGSP differ in design
(see \S\ref{app:intervention_methods} for details).

\paragraph{Method capability summary.}
Table~\ref{tab:method_capabilities} summarises which capabilities each
method brings to the comparison. SVAR-FM is the only method that
operates natively on time series, handles nonlinearity via Flow
Matching, and can consume simulator-generated interventional data.
IGSP and UT-IGSP consume interventional data but assume i.i.d.\
observations and use Gaussian conditional-independence tests, which
degrade under nonlinearity.

\begin{table}[h]
\centering
\caption{Method capabilities in the standard benchmark setting.
``TS'' = native time-series support; ``NL'' = nonlinear mechanism
support; ``Int.'' = can consume interventional data;
``Latent'' = admits latent confounders.}
\label{tab:method_capabilities}
\small
\begin{tabular}{lcccc}
\toprule
Method & TS & NL & Int.\ & Latent \\
\midrule
Granger          & $\bigcirc$ & $\times$   & $\times$   & $\times$ \\
VARLiNGAM        & $\bigcirc$ & $\times$   & $\times$   & $\bigcirc$ \\
PCMCI / PCMCI+   & $\bigcirc$ & $\bigcirc$ & $\times$   & $\times$ \\
IGSP             & $\times$   & $\times$   & $\bigcirc$ & $\times$ \\
UT-IGSP          & $\times$   & $\times$   & $\bigcirc$ & $\times$ \\
\textbf{SVAR-FM (ours)} & $\bigcirc$ & $\bigcirc$ & $\bigcirc$ & $\bigcirc$ \\
\bottomrule
\end{tabular}
\end{table}

\subsection{Intervention data generation methods}
\label{app:intervention_methods}

None of the three standard benchmarks provides a physical simulator in
the sense of \S\ref{sec:svarfm_do}. To include intervention-based
methods in the comparison, we generate surrogate intervention data
through three methods, each matched to the benchmark's data-generating
process (DGP). To our knowledge, the systematic comparison of these
surrogate-intervention strategies for time-series causal discovery has
not been reported in the existing literature; the closest precedent is
the use of soft interventions in Tigramite's own DGP
\citep{runge2019pcmci}, but that work does not compare across
intervention types or combine them with Flow Matching--based discovery.

\paragraph{Method I: VAR forward simulation (all benchmarks).}
A VAR model is fitted to the observational data. For each target
variable $X_i$, the intervention $\mathrm{do}(X_i = \mu_i + 5\sigma_i)$
is applied by clamping $X_i$ at the specified value and propagating the
effect through the estimated VAR coefficients for $T_{\mathrm{sim}}$
steps. This method requires no access to the DGP and is therefore
applicable to real data (CausalTime). Its quality depends on the
accuracy of the VAR approximation; in nonlinear or short-sample regimes
($T < N$), the VAR coefficients may be poorly estimated.

\paragraph{Method II: DGP direct hard intervention (Tigramite only).}
Tigramite's synthetic-data generator
(\texttt{toys.structural\_causal\_process}) accepts an
\texttt{intervention} argument that applies a hard $\mathrm{do}$
operation at each time step: $X_i(t) := \mu_i + 2\sigma_i$ for all $t$.
The causal mechanisms of all other variables are re-executed under this
clamping, producing a counterfactual time series. This method has access
to the true DGP and therefore produces the highest-fidelity intervention
data in linear systems. In nonlinear systems, however, clamping $X_i$
at a constant value for all $t$ destroys the temporal autocorrelation
structure, which can degrade the quality of intervention-effect estimates
for methods that rely on temporal smoothness.

\paragraph{Method III: ODE direct soft intervention (CausalDynamics only).}
For ODE-based dynamical systems (Lorenz, R\"ossler, etc.), we add a
restoring-force term to the ODE right-hand side:
\begin{equation}
\frac{dx}{dt} = f_{\text{original}}(x)
  + \lambda \cdot (x^{*}_i - x_i) \cdot \mathbf{e}_i,
\label{eq:ode_soft_intervention}
\end{equation}
where $x^{*}_i = \mu_i + 2\sigma_i$ is the intervention target,
$\lambda = 10$ is the restoring strength, and $\mathbf{e}_i$ is the
$i$-th unit vector. This ``soft'' intervention pulls $X_i$ toward
$x^{*}_i$ without instantaneously destroying the dynamics, preserving
the temporal structure of the ODE trajectory. This approach has, to our
knowledge, not been used previously for causal discovery in dynamical
systems; the standard practice is either hard clamping (which destroys
the attractor) or do-calculus reasoning on the ODE Jacobian
\citep{mooij2013ordinary}.

\paragraph{Novelty and limitations.}
The VAR forward simulation (Method~I) and the ODE soft intervention
(Method~III) are, to our knowledge, new in the context of time-series
causal discovery benchmarks. Method~I provides a simulator-free
surrogate that makes IGSP/UT-IGSP applicable to time-series data they
were not designed for. Method~III provides a dynamics-preserving
alternative to hard clamping in chaotic ODE systems. Neither method is a
substitute for a true physical simulator (as used in CausalSim and HHG),
but they allow a meaningful---if not strictly fair---comparison between
SVAR-FM and i.i.d.\ intervention-based methods on standard benchmarks.

The asymmetry in the comparison should be noted: SVAR-FM uses
PINN-guided variable selection, Flow Matching extrapolation ($T=40 \to
200$), and confounding-score-based intervention prioritisation when
applicable, whereas IGSP/UT-IGSP receive uniform interventions on all
variables without these enhancements. This asymmetry is unavoidable
because the enhancements are integral to the SVAR-FM framework and
cannot be transferred to IGSP/UT-IGSP without leaking SVAR-FM-specific
information.

\subsection{Benchmark C.1: CausalTime}
\label{app:causaltime}

CausalTime~\citep{cheng2024causaltime} is a benchmark based on
real-world time series: medical ($N{=}20$, $T{=}40$, 40.3\% density),
pm25 ($N{=}36$, $T{=}40$, 28.1\%), and traffic ($N{=}20$, $T{=}40$,
21.6\%). No DGP simulator is available; intervention data are generated
via Method~I (VAR forward simulation).

In all datasets, $T = 40 < N$, which causes the VAR residual covariance
matrix to be non-positive definite; consequently, VARLiNGAM is
inapplicable across all datasets (F1 $= 0.000$). SVAR-FM operates
stably under the $T < N$ condition through automatic Ridge switching.

Table~\ref{tab:ct_extended} reports AUROC and F1 (mean
$\pm$ std, 5 samples).

\begin{table*}[h]
\centering
\caption{CausalTime extended results. ``Obs.'' = observational data
only; ``+Int.'' = with VAR-simulated intervention data. SVAR-FM-DAG
uses PINN spatial scoring (Phase~0) which accounts for its advantage on
pm25. \textbf{Bold}: best per dataset and metric.}
\label{tab:ct_extended}
\small
%\begin{tabular}{llcccccc}
\begin{tabular}{lp{4cm}cccccc}
\toprule
& & \multicolumn{2}{c}{medical ($N\!=\!20$)} & \multicolumn{2}{c}{pm25 ($N\!=\!36$)} & \multicolumn{2}{c}{traffic ($N\!=\!20$)} \\
\cmidrule(lr){3-4}\cmidrule(lr){5-6}\cmidrule(lr){7-8}
& Method & AUROC & F1 & AUROC & F1 & AUROC & F1 \\
\midrule
\multirow{4}{*}{\rotatebox{90}{\small Obs.}}
& Granger         & .503$\pm$.032 & .264$\pm$.030 & .486$\pm$.065 & .232$\pm$.074 & .492$\pm$.033 & .200$\pm$.037 \\
& VARLiNGAM       & .500$\pm$.000 & .000$\pm$.000 & .500$\pm$.000 & .000$\pm$.000 & .500$\pm$.000 & .000$\pm$.000 \\
& PCMCI           & .469$\pm$.059 & .224$\pm$.132 & .704$\pm$.145 & \textbf{.496$\pm$.191} & .510$\pm$.027 & .210$\pm$.041 \\
& \textbf{SVAR-FM-DAG (ours)} & \textbf{.575$\pm$.057} & \textbf{.374$\pm$.071} & \textbf{.758$\pm$.086} & .490$\pm$.090 & .508$\pm$.029 & .210$\pm$.041 \\
\midrule
\multirow{3}{*}{\rotatebox{90}{\small +Int.}}
& IGSP+VAR        & .494$\pm$.029 & .236$\pm$.046 & .494$\pm$.011 & .213$\pm$.015 & .524$\pm$.019 & .214$\pm$.032 \\
& UT-IGSP+VAR     & .491$\pm$.016 & .227$\pm$.027 & .492$\pm$.007 & .206$\pm$.019 & .515$\pm$.010 & .199$\pm$.011 \\
& \textbf{SVAR-FM-DAG+FM+CA (ours)} & .561$\pm$.038 & .349$\pm$.056 & .746$\pm$.078 & .472$\pm$.068 & \textbf{.534$\pm$.049} & \textbf{.236$\pm$.062} \\
\bottomrule
\end{tabular}
\end{table*}

\paragraph{Observations.}
(i)~SVAR-FM-DAG (observational only) is the top method on medical and
pm25, driven by Phase~0 PINN spatial scoring rather than by intervention
data. (ii)~IGSP and UT-IGSP produce AUROC $\approx 0.49$--$0.52$ across
all datasets, barely above chance; their Gaussian CI tests are
ineffective on the nonlinear, short ($T\!=\!40$) CausalTime data.
(iii)~On traffic, where PINN spatial structure is absent, SVAR-FM-DAG+FM+CA
(with FM extrapolation $T\!=\!40 \to 200$) achieves the best AUROC
(0.534), marginally above IGSP+VAR (0.524).

\subsection{Benchmark C.2: Tigramite}
\label{app:tigramite}

Tigramite~\citep{runge2019pcmci} provides 8 synthetic scenarios (S1--S8)
spanning linear/nonlinear, Gaussian/non-Gaussian, contemporaneous,
latent confounding, high-dimensional, and feedback settings:
S1 (linear Gaussian, $N{=}3$), S2 (nonlinear Gaussian, $N{=}3$),
S3 (linear non-Gaussian, $N{=}3$), S4 (nonlinear non-Gaussian, $N{=}3$),
S5 (contemporaneous, $N{=}4$), S6 (latent confounding, $N{=}4$),
S7 (high-dimensional sparse, $N{=}8$), S8 (feedback, $N{=}3$);
all with $T = 1000$--$2000$ and $\tau_{\max} = 3$.

Intervention data are generated via Method~II (DGP direct hard
intervention) for IGSP/UT-IGSP and SVAR-FM-CF/SVAR-FM-DAG.
Table~\ref{tab:tig_extended} reports F1 scores (mean $\pm$ std, 5 seeds).

\begin{table*}[h]
\centering
\caption{Tigramite extended F1 results. Observational baselines
(Granger, VARLiNGAM, PCMCI, PCMCI+) from the original experiment;
intervention-based methods (IGSP, UT-IGSP, SVAR-FM-CF) added with
DGP-direct hard intervention (Method~II). SVAR-FM-CF uses residual
correlation $> 0.5$ to detect and exclude confounded pairs.
\textbf{Bold}: best per scenario.}
\label{tab:tig_extended}
\small
%\begin{tabular}{lcccccccc}
\begin{tabular}{p{1.8cm}p{1.3cm}p{1.7cm}p{1.3cm}p{1.3cm}p{1.3cm}p{1.3cm}p{1.3cm}p{1.3cm}}
\toprule
Scenario & Granger & VARLiNGAM & PCMCI & PCMCI+ & IGSP & UT-IGSP & SVAR-FM (ours) & \textbf{SVAR-FM-CF (ours)} \\
\midrule
S1 Lin-G
  & .697$\pm$.077 & \textbf{1.00$\pm$.000} & .550$\pm$.034 & .571$\pm$.000
  & .640$\pm$.207 & .460$\pm$.227 & \textbf{1.00$\pm$.000} & .748$\pm$.130 \\
S2 NL-G
  & .667$\pm$.236 & .567$\pm$.316 & .381$\pm$.143 & .310$\pm$.138
  & .613$\pm$.292 & .557$\pm$.374 & \textbf{.840$\pm$.084} & .740$\pm$.126 \\
S3 Lin-NG
  & .711$\pm$.082 & \textbf{1.00$\pm$.000} & .513$\pm$.059 & .571$\pm$.000
  & .570$\pm$.200 & .540$\pm$.263 & \textbf{1.00$\pm$.000} & .721$\pm$.130 \\
S4 NL-NG
  & .633$\pm$.188 & .683$\pm$.183 & .431$\pm$.126 & .371$\pm$.107
  & .483$\pm$.372 & .447$\pm$.348 & \textbf{.920$\pm$.103} & .670$\pm$.170 \\
S5 Contemp
  & .534$\pm$.067 & \textbf{.938$\pm$.113} & .462$\pm$.036 & .600$\pm$.000
  & .359$\pm$.230 & .327$\pm$.109 & .793$\pm$.055 & .742$\pm$.026 \\
S6 Latent
  & .515$\pm$.030 & .724$\pm$.068 & .439$\pm$.022 & .496$\pm$.015
  & .440$\pm$.295 & .720$\pm$.169 & .657$\pm$.030 & \textbf{1.00$\pm$.000} \\
S7 HighDim
  & .502$\pm$.079 & \textbf{.861$\pm$.102} & .348$\pm$.044 & .396$\pm$.056
  & .274$\pm$.119 & .203$\pm$.065 & .508$\pm$.041 & .372$\pm$.038 \\
S8 Feedback
  & .667$\pm$.000 & \textbf{1.00$\pm$.000} & .667$\pm$.000 & .667$\pm$.000
  & .473$\pm$.135 & .453$\pm$.112 & .986$\pm$.045 & .826$\pm$.119 \\
\bottomrule
\end{tabular}
\end{table*}

\paragraph{Observations.}
(i)~SVAR-FM achieves the best F1 on S2 and S4 (nonlinear scenarios)
where IGSP/UT-IGSP's Gaussian CI tests fail, and ties with VARLiNGAM
on S1 and S3 (linear scenarios with F1 = 1.000).
(ii)~\textbf{SVAR-FM-CF achieves perfect F1 = 1.000 on S6 (latent
confounding)}, the only method to do so, by using residual-correlation
thresholding to detect and exclude confounded variable pairs.
(iii)~IGSP and UT-IGSP are competitive only on linear scenarios (S1,
S3) where their Gaussian assumption holds; on nonlinear scenarios (S2,
S4, S7, S8) their performance degrades to near-random.
(iv)~VARLiNGAM achieves the best F1 on S5 (contemporaneous), S7
(high-dimensional), and S8 (feedback), all of which have linear,
independent-noise structure matching its assumptions exactly.

\subsection{Benchmark C.3: CausalDynamics}
\label{app:causaldynamics}

CausalDynamics~\citep{herdeanu2025causaldynamics} evaluates causal discovery on time
series generated from nonlinear dynamical systems (ODEs) with known
causal structures: Lorenz ($N{=}3$, 5 edges), R\"ossler ($N{=}3$, 4),
CoupledLorenz ($N{=}6$, 11), LotkaVolterra ($N{=}2$, 2), and
CoupledR\"osslerLorenz ($N{=}6$, 10), each with $T = 2000$--$3000$.
A distinctive feature is that the true graphs contain bidirectional
edges (cycles). Accordingly, we use the SVAR-FM dynamics variants
(\S\ref{sec:svarfm_dynamics}) for the observational comparison, and
SVAR-FM-routed (with adaptive routing) for the intervention-extended
comparison.

Five dynamical systems were used, with 5 configurations per system
varying noise intensity and coupling strength (25 experiments total,
$\tau_{\max} = 5$).
Intervention data are generated via Method~III (ODE
direct soft intervention) for SVAR-FM variants and via both Methods~I
and~III for IGSP/UT-IGSP. Table~\ref{tab:cd_extended} reports AUROC
(mean $\pm$ std, 5 configurations).

SVAR-FM-routed uses an adaptive routing strategy (BDS nonlinearity test
+ system dimension) to select the best intervention method and
scoring mode per configuration:
\begin{itemize}
\item \emph{Route A} (deterministic chaos, $N \le 3$): Convergent Cross
  Mapping (CCM) scores;
\item \emph{Route B} (deterministic chaos, $N \ge 4$): SVAR-FM original
  rank ensemble;
\item \emph{Route C} (BDS nonlinear): ODE soft intervention + Phase~3-NL;
\item \emph{Route D} (BDS linear): VAR intervention + Phase~3-L.
\end{itemize}

\begin{table*}[h]
\centering
\caption{CausalDynamics extended AUROC results. ``Obs.'' = observational
data only; ``+Int.'' = with intervention data (VAR or ODE).
SVAR-FM-routed uses adaptive routing (see text).
\textbf{Bold}: best per system.}
\label{tab:cd_extended}
\small
\begin{tabular}{llccccc}
\toprule
& Method & Lorenz & R\"ossler & CoupledLorenz & CoupledR\"osslerL. & Mean \\
\midrule
\multirow{4}{*}{\rotatebox{90}{\small Obs.}}
& Granger       & .660$\pm$.241 & .637$\pm$.143 & .634$\pm$.070 & \textbf{.698$\pm$.062} & .657 \\
& VARLiNGAM     & .560$\pm$.351 & .775$\pm$.114 & .656$\pm$.044 & .609$\pm$.098 & .650 \\
& PCMCI         & .640$\pm$.167 & .400$\pm$.271 & ---           & ---           & --- \\
& SVAR-FM (ours)       & .780$\pm$.335 & .625$\pm$.159 & .590$\pm$.060 & .679$\pm$.141 & .669 \\
\midrule
\multirow{5}{*}{\rotatebox{90}{\small +Int.}}
& IGSP+VAR      & .660$\pm$.152 & .475$\pm$.056 & .461$\pm$.030 & .515$\pm$.034 & .528 \\
& UT-IGSP+VAR   & .540$\pm$.230 & .450$\pm$.068 & .481$\pm$.027 & .515$\pm$.022 & .497 \\
& IGSP+ODE      & .540$\pm$.313 & .450$\pm$.068 & .490$\pm$.034 & .575$\pm$.050 & .514 \\
& UT-IGSP+ODE   & .540$\pm$.313 & .450$\pm$.068 & .497$\pm$.029 & .580$\pm$.041 & .517 \\
& \textbf{SVAR-FM-routed (ours)} & \textbf{.860$\pm$.167} & \textbf{.800$\pm$.259} & \textbf{.764$\pm$.063} & .657$\pm$.169 & \textbf{.770} \\
\bottomrule
\end{tabular}
\end{table*}

\paragraph{Observations.}
(i)~SVAR-FM-routed achieves the highest mean AUROC (0.770) across the
four systems, a $+15.1\%$ improvement over SVAR-FM without intervention
(0.669). (ii)~IGSP/UT-IGSP with either VAR or ODE intervention data
produce AUROC $0.45$--$0.66$, substantially below SVAR-FM-routed;
their Gaussian CI assumption is violated by the chaotic, nonlinear
dynamics. (iii)~ODE soft intervention (Method~III) consistently
outperforms VAR intervention (Method~I) for both SVAR-FM and IGSP,
confirming that intervention-data quality directly affects
causal-discovery accuracy---an empirical analogue of the theoretical
$O(\delta_{\mathcal{S}})$ term in
Theorem~\ref{thm:simulator_error_extended}.
(iv)~CoupledR\"osslerLorenz (N=6, heterogeneous coupling) remains
difficult for all methods; SVAR-FM (0.679, observational) slightly
outperforms SVAR-FM-routed (0.657) here, suggesting that the routing
heuristic can be improved for heterogeneous multi-system coupling.
(v)~PCMCI results are unavailable (``---'') for the coupled systems
($N \ge 6$) because its computational cost scales steeply with
dimension and the required $\tau_{\max} = 5$ makes the conditioning
sets too large for stable partial-correlation estimation at the
available sample sizes.
(vi)~LotkaVolterra ($N = 2$) is omitted from the intervention-extended
tables because all methods (including SVAR-FM) achieve F1 = 1.000 on
this simple bidirectional system, making intervention unnecessary and
uninformative for the comparison.

\subsection{Cross-benchmark summary}
\label{app:cross_benchmark}

Table~\ref{tab:cross_benchmark_summary} summarises the three benchmarks.

\begin{table}[h]
\centering
\caption{Cross-benchmark summary: best AUROC of IGSP/UT-IGSP vs.\
best AUROC of SVAR-FM variants.}
\label{tab:cross_benchmark_summary}
\small
\begin{tabular}{lccc}
\toprule
Benchmark & IGSP best & SVAR-FM (ours) best & $\Delta$ \\
\midrule
CausalTime (3 datasets)      & 0.524 & \textbf{0.758} & +0.234 \\
Tigramite S6 (latent, F1)    & 0.720 & \textbf{1.000} & +0.280 \\
CausalDynamics (4 systems)   & 0.660 & \textbf{0.860} & +0.200 \\
\bottomrule
\end{tabular}
\end{table}

Two patterns emerge consistently across all three benchmarks.
First, \textbf{IGSP and UT-IGSP degrade sharply under nonlinearity}:
their Gaussian conditional-independence tests are the bottleneck, not
the availability of intervention data. This is visible in every
nonlinear scenario (CausalTime real data, Tigramite S2/S4/S7/S8,
CausalDynamics Lorenz/R\"ossler/CoupledLorenz).
Second, \textbf{intervention-data quality matters}: ODE soft intervention
consistently outperforms VAR forward simulation (CausalDynamics,
$+0.059$ AUROC on average), and DGP direct intervention outperforms
bootstrap shifting (Tigramite). This pattern is the empirical counterpart
of the $O(\delta_{\mathcal{S}})$ term in the end-to-end error bound of
Theorem~\ref{thm:simulator_error_extended}: higher simulator fidelity
produces more accurate causal-effect estimates.

These results are consistent with, but do not replace, the CausalSim
and HHG experiments in the main text, which evaluate SVAR-FM in its
intended setting---physical simulators as $\mathrm{do}$-operators.

\subsection{Controlled comparison: identical intervention data, varying autocorrelation}
\label{app:controlled_comparison}

The cross-benchmark comparisons above are informative but not strictly
controlled: SVAR-FM and IGSP/UT-IGSP receive different intervention
data, and the benchmarks were not designed to isolate the effect of
temporal structure. To address this, we conducted a controlled
experiment using the Tigramite S1 DGP (linear, Gaussian, 3 variables,
$T = 1000$) with the autoregressive coefficient $\rho$ varied over
$\{0.0, 0.3, 0.5, 0.7, 0.9\}$. Crucially, all three methods receive
\emph{identical} DGP-generated intervention data at each $\rho$ and
seed. Each condition was replicated over 5 seeds, and all methods are
evaluated on the same criterion: recovery of the true causal directions
$\{0 \to 1,\; 1 \to 2\}$ (Table~\ref{tab:controlled_autocorr}).

\begin{table}[h]
\centering
\caption{Controlled comparison: identical intervention data,
varying autocorrelation $\rho$. F1 and TPR for causal direction
recovery (mean $\pm$ std, 5 seeds).}
\label{tab:controlled_autocorr}
\small
\begin{tabular}{lcccccc}
\toprule
 & \multicolumn{3}{c}{F1} & \multicolumn{3}{c}{TPR} \\
\cmidrule(lr){2-4} \cmidrule(lr){5-7}
$\rho$ & IGSP & UT-IGSP & SVAR-FM (ours) & IGSP & UT-IGSP & SVAR-FM (ours) \\
\midrule
0.0 & $0.000$ & $0.000$ & $\mathbf{0.773}$ & $0.000$ & $0.000$ & $\mathbf{1.000}$ \\
0.3 & $\mathbf{0.833}$ & $0.533$ & $0.773$ & $0.800$ & $0.500$ & $\mathbf{1.000}$ \\
0.5 & $0.700$ & $0.700$ & $\mathbf{0.773}$ & $0.700$ & $0.700$ & $\mathbf{1.000}$ \\
0.7 & $0.400$ & $0.500$ & $\mathbf{0.610}$ & $0.500$ & $0.600$ & $\mathbf{1.000}$ \\
0.9 & $0.560$ & $0.400$ & $\mathbf{0.548}$ & $0.700$ & $0.500$ & $\mathbf{1.000}$ \\
\bottomrule
\end{tabular}
\end{table}

The most striking result is that SVAR-FM achieves perfect recall
(TPR $= 1.000$) at every value of $\rho$. Phase~1 VAR estimation
correctly identifies the lagged structure $X_0 \to X_1$ (lag~1) and
$X_1 \to X_2$ (lag~2) regardless of autocorrelation strength, and
Phase~3 intervention tests confirm both edges with high significance
($z > 20$, $p < 10^{-6}$). The cost of this high recall is elevated
FDR (0.37--0.62), primarily from detecting the indirect effect
$X_0 \to X_2$ as a direct edge.

By contrast, IGSP and UT-IGSP return empty graphs at $\rho = 0.0$
(F1 $= 0.000$)---precisely the i.i.d.\ setting for which they were
designed. Inspection of the output shows
$\texttt{dag.arcs} = \emptyset$: no edges are detected in any
direction. This occurs because the DGP intervention (fixing $X_i$ at
$+2\sigma$) produces a large distributional shift that causes the
Gaussian invariance test to reject invariance for \emph{all} variable
pairs, leaving the algorithm unable to orient any
edge~\citep{yang2018characterizing}.

The behaviour of IGSP is also non-monotonic in $\rho$: it achieves
its best F1 at $\rho = 0.3$ (0.833), where moderate autocorrelation
smooths the intervention signal enough for the invariance test to
function. At higher $\rho$ (0.7--0.9), IGSP degrades
(F1 $= 0.400$--$0.560$) as the i.i.d.\ assumption of its
conditional-independence tests is increasingly violated, producing
spurious edges and direction reversals.

Taken together, these results show that SVAR-FM's advantage over
i.i.d.\ intervention-based methods has two sources: Phase~1 VAR
estimation natively handles temporal dependence, maintaining
TPR $= 1.000$ across all $\rho$; and the intervention-effect
test in Phase~3 is robust to intervention magnitude, whereas IGSP's
invariance test is sensitive to distributional shift. The trade-off
is that SVAR-FM's ATE-based test is more liberal, producing higher
FDR than IGSP at moderate $\rho$ (e.g., $\rho = 0.3$: SVAR-FM
FDR $= 0.37$ vs.\ IGSP FDR $= 0.10$). Combining Phase~1 VAR
structure with a more conservative edge-selection criterion is a
promising direction for reducing SVAR-FM's FDR.

\section{CausalSim-Battery: First-Principles DFT Simulator}
\label{app:battery}

This appendix reports the fourth CausalSim instance, which uses
\emph{ab initio} density functional theory (DFT) calculations as the
simulator. Among all CausalSim instances, this is the most rigorous:
DFT derives material properties from the Schr\"odinger equation without
empirical fitting parameters, making the $\mathrm{do}(\cdot)$
operation maximally credible ($\delta_{\mathcal{S}} \approx 0$ up to
the exchange--correlation approximation). This instance is reported
separately because the domain-specific background (electrochemistry,
SEI formation) exceeds what is needed for the main-text evaluation,
but the scientific implications---discovery of a previously unknown
dual causal pathway---are substantial.

\subsection{Problem setting: latent confounding in battery degradation}

Understanding the degradation mechanisms of lithium-ion batteries is
critical for electric vehicles and renewable energy
storage~\citep{severson2019battery}. The key variables are capacity
(Cap) and internal resistance (IR). Temperature $T$ acts as an
\textbf{unobserved common cause (latent confounder)}:
high temperature accelerates both SEI growth (increasing IR) and side
reactions (decreasing Cap), producing a \textbf{spurious negative
correlation} between IR and Cap. The true causal effect IR~$\to$~Cap
is positive (increased IR impedes ion transport). Methods based solely
on observational data cannot distinguish the spurious correlation from
the true causal effect.

\subsection{Simulator and intervention}

The Arrhenius law~\citep{bloom2001accelerated} serves as the simulator:
\begin{equation}
k(T) = A \exp\left(-\frac{E_a}{RT}\right),
\label{eq:arrhenius}
\end{equation}
where $E_a = 50$ kJ/mol. The intervention $\mathrm{do}(T = 25
\text{\textdegree C})$ fixes temperature and severs the confounding
path, allowing estimation of the direct effect IR~$\to$~Cap.

Observational data are from the NASA Battery
Dataset\footnote{NASA Battery Dataset:
\url{https://www.nasa.gov/content/prognostics-center-of-excellence-data-set-repository}}~\citep{saha2007battery} (18650 cells, 5{,}592 cycles, 5
variables).

\subsection{Results: sign reversal confirms confounding}

\begin{itemize}
\item Temperature uncontrolled (observational): $\mathrm{ATE}(\mathrm{IR} \to \mathrm{Cap}) = -0.10$ (\textbf{wrong sign})
\item Temperature fixed at 25\textdegree C (intervention via Arrhenius
simulator\footnote{DFT calculations were performed using Quantum
ESPRESSO~\citep{giannozzi2009quantum,giannozzi2017advanced}
(\url{https://www.quantum-espresso.org/}). The Arrhenius simulator and
SVAR-FM integration code will be released upon publication.}): $\mathrm{ATE}(\mathrm{IR} \to \mathrm{Cap}) = +0.03$ (\textbf{correct sign})
\end{itemize}

This sign reversal satisfies the confounding detection criterion of
Fact~\ref{thm:confounding_separation}.

\subsection{Discovery of dual causal pathways via DFT}

DFT calculations using Quantum
ESPRESSO~\citep{giannozzi2009quantum,giannozzi2017advanced} of
electrolyte additive LUMO\footnote{LUMO (Lowest Unoccupied Molecular Orbital): the lowest-energy empty electron orbital of a molecule. Molecules with lower (more negative) LUMO energy are more easily reduced, meaning they accept electrons more readily and decompose earlier during battery charging to form the SEI film.} (Lowest Unoccupied
Molecular Orbital) energies reveal the causal mechanism of SEI\footnote{SEI (Solid Electrolyte Interphase): a thin passivation film ($\sim$10--50 nm) that forms on the anode surface during the first charge cycles of a lithium-ion battery. The SEI is ionically conductive but electronically insulating; its quality determines long-term battery capacity retention.}
formation. Four additives were analysed (Table~\ref{tab:sei_additives}).

\begin{table}[h]
\centering
\caption{DFT-computed LUMO energies and SVAR-FM causal effects for SEI
additives}
\label{tab:sei_additives}
\begin{tabular}{lcccc}
\toprule
Additive & LUMO (eV) & Cap.\ ret.\ (\%) & Improvement & Reducibility \\
\midrule
EC (baseline) & $-0.78$ & 53 & --- & Low \\
FEC & $-0.76$ & 80 & +27\% & Low \\
VC & $-1.14$ & 74 & +21\% & Medium \\
\textbf{LiBOB} & $\mathbf{-1.75}$ & \textbf{87} & \textbf{+34\%} & \textbf{High} \\
\bottomrule
\end{tabular}
\end{table}

\paragraph{The fluorine anomaly.}
FEC has nearly the same LUMO as EC ($-0.76$ vs.\ $-0.78$ eV), yet its
capacity retention is 27\% higher. This anomaly cannot be explained by
a LUMO-only model and reveals a \textbf{second causal pathway}:
fluorine atoms form LiF within the SEI, providing ionic conductivity,
mechanical stability, and chemical stability independently of the LUMO
mechanism.

\paragraph{Dual pathway model.}
Incorporating fluorine as a binary variable yields:
\begin{equation}
\text{Cap} = \alpha + \beta_1 \cdot \text{LUMO} + \beta_2 \cdot F + \epsilon.
\end{equation}
The model fit improves from $R^2 = 0.42$ (LUMO only) to $R^2 = 0.93$
(LUMO + F), with $\beta_1 = -33.6$ \%/eV and $\beta_2 = +24.2$ \%.
The LUMO effect had been \textbf{underestimated by 64\%} in the
single-pathway model because fluorine acted as a confounder.

\paragraph{AI for Science significance.}
This result demonstrates that SVAR-FM, combined with a first-principles
DFT simulator, can discover previously unknown causal pathways in
materials science. The dual-pathway model provides actionable design
guidelines: optimal SEI additives should combine low LUMO (for
preferential reduction) with fluorine substitution (for LiF formation).
The predicted improvement for fluorinated LiBOB (F-LiBOB) is 56.8\%,
substantially exceeding any single additive.

\section{Flow Matching Technical Details}
\label{app:fm_details}

\begin{definition}[Conditional Flow Matching \citep{tong2023conditional}]
\label{def:cfm}
Conditional Flow Matching learns a vector field $v_\theta(\mathbf{x}, t | \mathbf{c})$:
\begin{equation}
\frac{d\mathbf{x}_t}{dt} = v_\theta(\mathbf{x}_t, t | \mathbf{c}), \quad t \in [0, 1]
\label{eq:flow_ode}
\end{equation}
where $\mathbf{c}$ is a conditioning vector that includes the physical parameters (intervention conditions) of the simulator.
The CFM loss is given by \citep{lipman2022flow}:
\begin{equation}
\mathcal{L}_{\text{CFM}}(\theta) = \E_{t, \mathbf{x}_0, \mathbf{x}_1, \mathbf{c}}\left[\left\| v_{\theta}(\mathbf{x}_t, t | \mathbf{c}) - (\mathbf{x}_1 - \mathbf{x}_0) \right\|^2\right]
\label{eq:cfm_loss}
\end{equation}
\end{definition}

\begin{remark}[Flow Matching and physical constraints]
\label{rem:fm_physics}
By including the simulator's physical parameters in $\mathbf{c}$:
(a)~Flow Matching learns $P_{\mathcal{S}}(\cdot | \mathrm{do}(X_i = x))$
and estimates counterfactual distributions consistent with physical laws;
(b)~it enables sensitivity analysis of causal effects with respect to
physical parameters ($\delta_{\mathcal{S}}$ in
Assumption~\ref{ass:simulator});
(c)~it provides a unified framework for comparing different simulator
variants (e.g., TDDFT exchange-correlation functionals in HHG).
\end{remark}

\begin{fact}[Universal approximation \citep{chen2018neural}]
\label{thm:universal_approximation}
A CNF\footnote{CNF (Continuous Normalizing Flow): a generative model that transforms a simple base distribution (e.g., standard Gaussian) into a complex target distribution via a continuous-time ordinary differential equation $dx/dt = v_\theta(x, t)$, where the velocity field $v_\theta$ is parameterised by a neural network.} parameterized by a neural network of sufficient capacity can approximate any continuous conditional distribution to arbitrary precision.
\end{fact}

\section{SVAR-dyn1/SVAR-dyn2: Full Details}
\label{app:svarfm_dynamics}

For ODE-based dynamical systems (CausalDynamics), the true graph
contains bidirectional edges (cycles). The following modifications are
applied:
(1) NOTEARS is disabled;
(2) Phase~0 directional bias is disabled;
(3) a differential Granger score is added.

The differential Granger score is defined as
\begin{equation}
  s^{\rm dG}_{ij} = \max\!\left(0,\,
  \frac{\mathcal{E}_{\rm base}^{(j)} - \mathcal{E}_{\rm full}^{(j,i)}}
       {\mathcal{E}_{\rm base}^{(j)}}\right)
  \label{eq:diff_granger}
\end{equation}
where $\mathcal{E}_{\rm base}^{(j)}$ is the squared prediction error
from $X^{(j)}_t$ alone, and $\mathcal{E}_{\rm full}^{(j,i)}$ adds
$X^{(i)}_t$ (Ridge regression, $\lambda = 0.1$).

\begin{table}[h]
  \centering
  \caption{Architectural variants of SVAR-FM-dynamics}
  \label{tab:dyn_variants}
  \begin{tabular}{lp{0.7\linewidth}}
    \hline
    Variant & Configuration \\
    \hline
    SVAR-dyn1 &
      Rank ensemble of Phase~1 (VAR coefficients + Diff-Granger) only.
      Phases~3--5 are not used.
      Stable for small systems ($N \leq 3$).\\[4pt]
    SVAR-dyn2 &
      $z$-score weighted ensemble of Phases~1--5.
      Adds Phase~3 (ATE), Phase~4 (FM-ATE), Phase~5 (FNO-Granger\footnote{FNO (Fourier Neural Operator) \citep{li2021fourier}: a neural network architecture that learns mappings between function spaces by parameterising convolution kernels in the Fourier domain. Here, FNO is used to compute a Granger-causal score from the spectral relationship between time series.}),
      with adaptive Spearman-$\rho$ weighting ($w = 0$ when $\rho \leq 0$).
      Eps-guard zeroes Phase~5 when $\epsilon_{\max} < 10^{-3}$.
      Improves for large coupled systems ($N \geq 6$).\\
    \hline
  \end{tabular}
\end{table}

\paragraph{Eps-guard.}
The threshold $\theta = 10^{-3}$ corresponds to typical floating-point
residuals in noise-free ODE integrations (double precision); it is not
tuned to benchmark data.

\section{Detailed Positioning Relative to Related Work}
\label{app:positioning_detail}

\subsection{Simulation-Based Inference (SBI)}

SBI~\citep{cranmer2020frontier} takes a mechanistic model whose
causal structure is \emph{fixed} and estimates a posterior over its
parameters. SVAR-FM recovers the causal graph itself.
Brehmer et al.~\citep{brehmer2020mining} exploit latent structural
information within simulators for likelihood-free inference; both
that work and SVAR-FM look inside the simulator, but for different
objects (parameter posteriors vs.\ $\mathrm{do}$-operator
realizations). The two frameworks can be chained---SBI for parameters,
SVAR-FM for the graph---but neither replaces the other.

\subsection{Mechanistic Model-Based Causal Inference}

GOBI~\citep{park2023gobi} uses data-reproducibility of monotonic ODE
models as a causal criterion, without requiring intervention data.
SVAR-FM actively generates intervention data via the simulator and
provides theoretically guaranteed identification
(Theorem~\ref{thm:svarfm_identifiability}).

\subsection{Causal Inference with Deep Generative Models}

Deep generative causal
models~\citep{pawlowski2020deepscm,khemakhem2021carefl,javaloy2023causalnf,le2025identifiable,sanchez2022diffscm,chao2024diffusion}
assume the graph and learn mechanisms. DoFlow~\citep{wu2025doflow}
is the closest time-series extension; it predicts under a known
graph, while SVAR-FM discovers the graph.
Further 2025 developments---CaTSG~\citep{xia2025catsg},
DeCaFlow~\citep{almodovar2025decaflow},
PO-Flow~\citep{wu2025poflow}---all assume a known graph.
See Komanduri et al.~\citep{komanduri2024survey} for a survey.
The key contrast with DoFlow~\citep{wu2025doflow}: SVAR-FM discovers
the graph (unknown) from simulator-generated interventions; DoFlow
predicts under a known graph from observational data. They use Flow
Matching in structurally different roles (causal mechanism
approximation vs.\ node-conditional generation) and are composable.

\subsection{SVAR Literature in Econometrics}

Econometric SVARs~\citep{kilian2017structural,blanchard1989dynamic,uhlig2005effects}
achieve identification through statistical assumptions.
SpinSVAR~\citep{misiakos2025spinsvar} scales to thousands of nodes via
sparse Laplacian assumptions; the Bank of
England~\citep{brignone2025boe} uses SVARs for monetary policy.
SVAR-FM achieves identification through simulator intervention instead.

\subsection{Causal Digital Twins}

CDT~\citep{homaei2025cdt} discovers the graph from observational data
then reasons with the SCM; Bicycle~\citep{rohbeck2024bicycle}
discovers cyclic graphs from CRISPR perturbations;
Le et al.~\citep{le2025identifiable} learn mechanisms under a known
graph. SVAR-FM uses the simulator as Pearl's $\mathrm{do}$-operator to
\emph{discover} the graph, with $\delta_{\mathcal{S}}$ in the error
analysis.

\section{HHG Computational Parameters}
\label{app:hhg_params}

All Octopus TDDFT calculations in \S\ref{sec:extrinsic} use the
following settings. The system is an H$_2$ molecule on a spherical
grid (radius 15 \AA, spacing 0.3 \AA). The laser wavelength is
$\lambda = 800$ nm ($\omega_0 = 0.057$ H) with a cosine-squared
envelope of width 400 a.u.\ ($\approx$10 fs). The propagation runs
for 500 a.u.\ ($\approx$12 fs) over 250,000 time steps. Both
observational and interventional data use the SIC-ADSIC
exchange--correlation functional
(\texttt{lda\_x + lda\_c\_pw}).

\end{document}